%% file: main.tex
\pdfoutput=1
\documentclass{article}
\usepackage{amsmath,amsthm,amssymb,sansmath,float,color,authblk,stmaryrd,calc}

\PassOptionsToPackage{numbers, compress}{natbib}

\usepackage[preprint]{neurips_2022}

\usepackage[utf8]{inputenc} 
\usepackage[T1]{fontenc}    
\usepackage{url}            
\usepackage{booktabs}       
\usepackage{amsfonts}       
\usepackage{nicefrac}       
\usepackage{microtype}      
\usepackage{xcolor}         
\title{Local - (learned) Global Mixed Kernels}

\usepackage[utf8]{inputenc}
\usepackage[inline]{enumitem}
\usepackage[active]{srcltx}
\usepackage{upgreek}
\usepackage{mathrsfs}
\usepackage[utf8]{inputenc}
\usepackage{makeidx}
\usepackage{oldgerm}
\usepackage{mathrsfs}
\usepackage[active]{srcltx}
\usepackage{verbatim}
\usepackage{aliascnt}
\usepackage{array}
\usepackage[textwidth=4cm,textsize=footnotesize]{todonotes}
\usepackage{xargs}
\usepackage{cellspace}
\usepackage[linesnumbered,ruled,vlined, algo2e]{algorithm2e}
\usepackage{svg}
\usepackage{bbm}
\usepackage{amsmath}

\usepackage{aliascnt}

\usepackage{xr-hyper}
\usepackage[bookmarks=false]{hyperref}

\usepackage{caption}
\usepackage{subcaption}
\usepackage{multirow}

\usepackage{cleveref}

\usepackage{autonum}

\makeatletter
\newtheorem{theorem}{Theorem}
\crefname{theorem}{theorem}{Theorems}
\Crefname{Theorem}{Theorem}{Theorems}

\newaliascnt{lemma}{theorem}
\newtheorem{lemma}[lemma]{Lemma}
\aliascntresetthe{lemma}
\crefname{lemma}{lemma}{lemmas}
\Crefname{Lemma}{Lemma}{Lemmas}

\newaliascnt{corollary}{theorem}

\aliascntresetthe{corollary}
\crefname{corollary}{corollary}{corollaries}
\Crefname{Corollary}{Corollary}{Corollaries}

\newaliascnt{proposition}{theorem}
\newtheorem{proposition}[proposition]{Proposition}
\aliascntresetthe{proposition}
\crefname{proposition}{proposition}{propositions}
\Crefname{Proposition}{Proposition}{Propositions}

\newaliascnt{definition}{theorem}
\newtheorem{definition}[definition]{Definition}
\aliascntresetthe{definition}
\crefname{definition}{definition}{definitions}
\Crefname{Definition}{Definition}{Definitions}

\newaliascnt{remark}{theorem}

\aliascntresetthe{remark}
\crefname{remark}{remark}{remarks}
\Crefname{Remark}{Remark}{Remarks}

\crefname{example}{example}{examples}
\Crefname{Example}{Example}{Examples}

\crefname{algorithm}{algorithm}{algorithms}
\Crefname{Algorithm}{Algorithm}{Algorithms}

\crefname{figure}{figure}{figures}
\Crefname{Figure}{Figure}{Figures}

\newtheorem{assumption}{\textbf{A}\hspace{-3pt}}
\Crefname{assumption}{\textbf{A}\hspace{-3pt}}{\textbf{A}\hspace{-3pt}}
\crefname{assumption}{\textbf{A}}{\textbf{A}}

\newtheorem{assumptionMC}{\textbf{MC}\hspace{-3pt}}
\Crefname{assumptionMC}{\textbf{MC}\hspace{-3pt}}{\textbf{MC}\hspace{-3pt}}
\crefname{assumptionMC}{\textbf{MC}}{\textbf{MC}}

\newtheorem{assumptionSA}{\textbf{SA}\hspace{-3pt}}
\Crefname{assumptionSA}{\textbf{SA}\hspace{-3pt}}{\textbf{SA}\hspace{-3pt}}
\crefname{assumptionSA}{\textbf{SA}}{\textbf{SA}}

\newtheorem{assumptionH}{\textbf{H}\hspace{-3pt}}
\Crefname{assumptionH}{{\textbf{H}\hspace{-1pt}}}{{\textbf{H}\hspace{-1pt}}}
\crefname{assumptionH}{{\textbf{H}}}{{\textbf{H}}}

\title{Local-Global MCMC kernels: the best of both worlds}

\author{%
\textbf{Sergey Samsonov}$^{1}$ \quad \textbf{Evgeny Lagutin}$^{1}$ \quad \textbf{Marylou Gabrié}$^2$ \quad \textbf{Alain Durmus}$^3$ \\
 \quad \textbf{Alexey Naumov}$^1$ \quad \textbf{Eric Moulines}$^{2}$ \\
$^1$HSE University \quad $^2$Ecole Polytechnique \quad $^3$ENS Paris-Saclay\\
\texttt{\{svsamsonov,elagutin,anaumov\}@hse.ru}\\
\texttt{\{eric.moulines,marylou.gabrie\}@polytechnique.edu}\\
\texttt{alain.durmus@ens-paris-saclay.fr}
}

\input{def}

\begin{document}

\maketitle

\begin{abstract}
Recent works leveraging learning to enhance sampling have shown promising results, in particular by designing effective non-local moves and global proposals. However, learning accuracy is inevitably limited in regions where little data is available such as in the tails of distributions as well as in high-dimensional problems. In the present paper we study an Explore-Exploit Markov chain Monte Carlo strategy (\XTryM) that combines local and global samplers showing that it enjoys the advantages of both approaches. We prove $V$-uniform geometric ergodicity of \XTryM\ without requiring a uniform adaptation of the global sampler to the target distribution. We also compute explicit bounds on the mixing rate of the Explore-Exploit strategy under realistic conditions.
Moreover, we {also analyze} an adaptive version of the strategy (\FlXTryM) where a normalizing flow is trained while sampling to serve as a proposal for global moves. We illustrate the efficiency of \XTryM\ and its adaptive version on classical sampling benchmarks as well as in sampling high-dimensional distributions defined by Generative Adversarial Networks seen as Energy Based Models. We provide the code to reproduce the experiments at the link: \url{https://github.com/svsamsonov/ex2mcmc_new}.
\end{abstract}

\section{Introduction}
\label{sec:intro}

We consider the setting where a target distribution $\target$ on a measurable space $(\Xset,\Xsigma)$ is known up to a normalizing constant and one tries to estimate the expectations of some function $f : \Xset \to \rset$ with respect to $\target$. Examples include the extraction of Bayesian statistics from posterior distributions derived from observations as well as the computation of observables of a physical system $x \in \Xset$ under the Boltzmann distribution with non-normalized density $\pdftarget(x)= \rme^{- \beta U(x)}$ for the energy function $U$ at the inverse temperature $\beta$.

A common strategy to tackle this estimation is to resort to Markov chain Monte Carlo algorithms (MCMCs). The MCMC approach aims to simulate a realization of a time-homogeneous Markov chain $\sequence{Y}[n][\nset]$, such that the distribution of the $n$-th iterate $Y_n$ with $n \to \infty$ is arbitrarily close to $\target$, regardless of the initial distribution of $Y_0$. In particular, the Metropolis-Hastings kernel (MH) is the cornerstone of MCMC simulations, with a number of successful variants following the process of a \emph{proposal} step followed by an \emph{accept/reject} step (see e.g. \cite{robert:2015:mh}). In large dimensions, proposal distributions are typically chosen to generate local moves that depend on the last state of the chain in order to guarantee an admissible acceptance rate. However, local samplers suffer from long mixing times as
exploration is inherently slow, and mode switching, when there is more than one, can be extremely infrequent.

On the other hand, independent proposals are able to generate more global updates, but they are difficult to design. Developments in deep generative modelling, in particular versatile autoregressive and normalising flows \cite{kingma:2016:aif,huang:2018:neural_aif,dinh:2016,papamakarios2019normalizing}, spurred efforts to use learned probabilistic models to improve the exploration ability of MCMC kernels. Among a rapidly growing body of work, references include \cite{Hoffman2019, albergo:2019, Nicoli2020, Gabrie2021, Hackett2021}. While these works show that global moves in a number of practical problems can be successfully informed by machine learning models, it remains the case that the acceptance rate of independent proposals decreases dramatically with dimensions -- except in the unrealistic case that they perfectly reproduce the target. This is a well-known problem in the MCMC literature \cite{Bengtsson:2008:Curse,Bickel:2008:obstacle,agapiou2017importance}, and it was recently noted that deep learning-based suggestions are no exception in works focusing on physical systems \cite{DelDebbio2021,Mahmoud2022}.

In this paper we focus on the benefits of combining local and global samplers. Intuitively, local steps interleaved between global updates from an independent proposal (learned or not) increase accuracy by allowing accurate sampling in tails that are not usually well handled by the independent proposal. Also, mixing time is usually improved by the local-global combination, which prevents long chains of consecutive rejections. Here we focus on a global kernel of type iterative-sampling importance resampling (\isir) \citep{tjelmeland2004using,andrieu2010particle,andrieu2018uniform}. This kernel uses multiple proposals in each iteration to take full advantage of modern parallel computing architectures. For local samplers, we consider common techniques such as Metropolis Adjusted Langevin (MALA) and Hamiltonian Monte Carlo (HMC). We call this combination strategy Explore-Exploit MCMC (\XTryM) in the following.

\vspace{-2mm}
\paragraph{Contributions} The main contributions of the paper are as follows:
\begin{itemize}[leftmargin=*,noitemsep]
\vspace{-2mm}
\item We provide theoretical bounds on the accuracy and convergence speed of \XTryM\ strategies. In particular, we prove $V$-uniform geometric convergence of \XTryM\ under assumptions much milder than those required to prove uniform geometric ergodicity of the global sampler \isir~alone. 
\item We {provide convergence guarantees for} an adaptive version of the strategy, called \FlXTryM, which involves learning an efficient proposal while  sampling, as in adaptive MCMC.
\item We perform a numerical evaluation of \XTryM\ and \FlXTryM\ for various sampling problems, including sampling GANs as energy-based models. The results clearly show the advantages of the {combined} approaches compared to purely local or purely global MCMC methods.
\end{itemize}
\vspace{-3mm}
\paragraph{Notations} Denote $\nsets = \nset \setminus \{0\}$. For a measurable function $f: \Xset \mapsto \rset$, we define $\supnorm{f} = \sup_{x \in \Xset}|f(x)|$ and $\target(f) := \int_{\Xset} f(x) \target(\rmd x)$. For a function $V: \Xset \mapsto [1, \infty)$ we introduce the $V$-\emph{norm} of two probability measures $\xi$ and $\xi'$ on $(\Xset,\Xsigma)$, $\Vnorm{\xi - \xi'} \eqdef \sup_{|f(x)| \leq V(x)} |\xi(f) - \xi'(f)|$. If $V\equiv 1$, $\Vnorm [1]{\cdot}$ is equal to the total variation distance (denoted $\tvnorm{\cdot}$).

\vspace{-5pt}
\section{Explore-Exploit Samplers}
\label{sec:mcmc_w_theorems}
\vspace{-2mm}
Suppose we are given a target distribution $\target$ on a measurable space $(\Xset,\Xsigma)$ that is known only up to a normalizing constant. We will often assume that $\Xset = \rset^{d}$ or a subset thereof. Two related problems are sampling from $\target$ and estimating integrals of a function $f: \Xset \mapsto \rset$ \wrt\ $\target$, i.e., $\target(f)$. Among the many methods devoted to solving these problems, there is a popular family of techniques based on \emph{Importance Sampling} (IS) and relying on independent proposals, see e.g. \cite{agapiou2017importance,tokdar:is:overview}. We first give a brief overview of IS, to describe the global sampler \isir . We recall ergodicity results for the latter before investigating the Explore-Exploit sampling strategy which couples the global sampler with a local kernel. Then we present the main theoretical result of the paper on the ergodicity of the coupled strategy.
\vspace{-3mm}
\subsection{From Importance Sampling to \isir\ }
\vspace{-2mm}
The primary purpose of IS is to approximate integrals of the form $\target(f)$. Its main instrument is a (known) \emph{proposal distribution}, which we denote by $\proposal(\rmd x)$. To describe the algorithm, we assume that $\target(\rmd x) = \weightfunc(x) \proposal(\rmd x) / \proposal(\weightfunc)$.
In this formula, $\weightfunc(x)$ is the \emph{importance weight} function assumed to be known and positive, i.e., $\weightfunc(x) > 0$ for all $x \in \Xset$, and $\proposal(\weightfunc)$ is the \emph{normalizing constant} of the distribution $\target$. Typically $\proposal(\weightfunc)$ is unknown. If we assume that $\target$ and $\proposal$ have positive densities \wrt\ a common dominant measure, denoted also by $\pdftarget$ and $\pdfproposal$ respectively, then the \emph{self-normalized importance sampling} (SNIS, see~\cite{robert:casella:2013}) estimator of $\target(f)$ is given by
\begin{equation}\label{eq:is_estimate}
\textstyle{\widehat{\target}_N(f) = \sum_{i=1}^N \omega_{N}^i f(X^i)}\eqsp,
\end{equation}
where $\chunku{X}{1}{N} \simiid \proposal$, and $\omega_{N}^i = \weightfunc(X^{i})/ \sum_{j=1}^N \weightfunc(X^j)$
are the self-normalized importance weights. Note that computing $\omega_{N}^i$ does not require the knowledge of $\proposal(\weightfunc)$. The main problem in the practical applications of IS is the choice of the proposal distribution $\proposal$. The representation $\target(\rmd x) = \weightfunc(x) \proposal(\rmd x) / \proposal(\weightfunc)$ implies that the support of $\proposal$ covers the support of $\target$. At the same time, too large variance of $\proposal$ is obviously detrimental to the quality of \eqref{eq:is_estimate}. This suggests \emph{adaptive importance sampling} techniques (discussed in ~\cite{bugallo:2017:ais:review}), which involve learning the proposal $\proposal$ to improve the quality of \eqref{eq:is_estimate}. We return to this idea in section \ref{sec:adaptive_flow}.
\par 
IS -based techniques can also be used to draw an (approximate) sample from $\target$. For instance, Sampling Importance Resampling (SIR, \cite{rubin1987comment}) follows the steps:
\vspace{-3mm}
\begin{enumerate}[noitemsep]
\setlength{\itemindent}{-2.5em}
\item Draw $\chunku{X}{1}{N} \simiid \proposal$;
\item Compute the self-normalized importance weights $\omega^i_{N} = \weightfunc(X^i)/\sum_{\ell=1}^N \weightfunc(X^\ell)$, $i \in \{1, \dots, N\}$;
\item Select $M$ samples $\chunku{Y}{1}{M}$ from the set $\chunku{X}{1}{N}$ choosing $X^i$ with probability $\omega^i_{N}$ with replacement.
\end{enumerate}
\vspace{-1mm}
The drawback of the procedure is that it is only asymptotically valid with $N \rightarrow \infty$. Alternatively, SIR can be repeated to define a Markov Chain as in \emph{iterated SIR } (\isir), proposed in \cite{tjelmeland2004using} and also studied in \citep{andrieu2010particle,lee2010utility,lee2011auxiliary,andrieu2018uniform}. At each iteration of \isir\, described in \Cref{alg:i-sir-MCMC}, a candidate pool $\chunku{X_{k+1}}{2}{N}$ is sampled from the proposal and the next state $Y_{k + 1}$ is choosen among the candidates and the previous state $X_{k+1}^1 = Y_k$ according to the importance weights.
\isir\ shares similarities with the Multiple-try Metropolis (MTM) algorithm \cite{liu2000multiple}, but is computationally simpler and exhibits more favorable mixing properties; see \Cref{subsec:isir-MTM}. The Markov chain $\sequence{Y}[k][\nset]$ generated by \isir\ has the following Markov kernel
\begin{equation}
\label{eq:isir-kernel}
\MKisir(x, \msa) = \int \updelta_x(\rmd x^{1}) \sum_{i=1}^N \frac{\weightfunc(x^i)}{\sum_{j=1}^N \weightfunc(x^j)}\indi{\msa}(x^i) \prod_{j=2}^N\proposal(\rmd x^j).
\end{equation}
Interpreting \isir\ as a systematic-scan two-stage Gibbs sampler (see~\Cref{subsec:two-stages-gibbs} for more details), it follows easily
that the Markov kernel $\MKisir$ is reversible \wrt\ the target $\target$, Harris recurrent and ergodic (see \Cref{theo:main-properties-deterministic-scan}). Provided also that $\supnorm{\weightfunc} < \infty$, it was shown in \cite{andrieu2018uniform} that the Markov kernel $\MKisir$ is uniformly geometrically ergodic. Namely, for any initial distribution $\xi$ on $(\Xset,\Xsigma)$ and $k \in \nset$,
\begin{equation}
\label{eq:tv_dist_isir}
\tvnorm{\xi \MKisir^k - \target} \leq \driftconstisir^k \quad \text{with $\;\epssmallisir =  \frac{N -1}{2\bound + N - 2}, \bound = \supnorm{\weightfunc} / \proposal(\weightfunc)\eqsp, \;$ and $\;  \driftconstisir = 1 - \epssmallisir$.}
\end{equation}
We provide a simple direct proof of \eqref{eq:tv_dist_isir} in
\Cref{supp:subsec:proof:theo:isir_unfirom_ergodicity}.
Yet, note that the bound \eqref{eq:tv_dist_isir} relies significantly on the restrictive condition that weights are uniformly bounded $\supnorm{\weightfunc} < \infty$. Moreover, even when this condition is satisfied, the rate $\driftconstisir$ can be close to $1$ when the dimension $d$ is large.\footnote{Indeed, consider a simple scenario $\target(x) = \prod_{i=1}^d p(x_i)$ and $\proposal(x) = \prod_{i=1}^d q(x_i)$ for some densities $p(\cdot)$ and $q(\cdot)$ on $\rset$. Then it is easy to see that $\bound = (\sup_{y \in \rset} p(y) / q(y))^{d}$ grows exponentially with $d$.} { We illustrate this phenomenon on a Gaussian target in \Cref{app:subsec:single-gaussian} \Cref{fig:gaussian_sampling} with an experiment that also contrasts the degradation as dimension grows of the purely global sampler with the robustness of the local-global kernels analyzed in the next section.} 

\begin{algorithm2e}[t!]
    \SetKwInOut{In}{Input}
    \SetKwInOut{Out}{Output}
    \caption{Single stage of \isir\ algorithm with independent proposals}
    \label{alg:i-sir-MCMC}
    \SetKwFunction{KwPrint}{print}
    \SetKwProg{Fn}{Procedure \ }{:}{end}
    \SetKwFunction{X-Try}{iterated SIR}
    \SetKwProg{myalg}{Algorithm}{}{}
    \Fn{\isir\ $(Y_k,\proposal)$}{
    \In{Previous state $Y_{k}$; proposal distribution $\proposal$;}
    \Out{New state $Y_{k+1}$; pool of proposals $\chunku{X}{2}{N}_{k+1} \sim \proposal$;}
    Set $X^1_{k+1}= Y_{k}$, draw $\chunku{X}{2}{N}_{k+1} \sim \proposal$; 
    \For{$i \in \range{N}$}{
        compute the normalized weights $\normweight_{i,k+1}= \weightfunc(X^{i}_{k+1})/\sum_{\ell=1}^N\weightfunc(X^{\ell}_{k+1})$; \\
        }
    Draw the proposal index $I_{k+1} \sim \mathrm{Cat}(\normweight_{1,k+1},\dots,\normweight_{N,k+1})$; \\
    Set $Y_{k+1} := X^{I_{k+1}}_{k+1}$. \\
    }
\vspace{-3pt}
\end{algorithm2e} 

\vspace{-3mm}
\subsection{Coupling with local kernels: \XTryM}
\label{sec:coupling-with-local}
\begin{algorithm2e}[t!]
    \SetKwInOut{In}{Input}
    \SetKwInOut{Out}{Output}
    \caption{Single stage of \XTryM\ algorithm with independent proposals}
    \label{alg:X-Try-MCMC}
    \SetKwFunction{KwPrint}{print}
    \SetKwProg{Fn}{Procedure \ }{:}{end}
    \SetKwFunction{X-Try}{X-Try SIR}
    \SetKwProg{myalg}{Algorithm}{}{}
    \Fn{\XTryM\ $(Y_k,\proposal,\rejukernel)$}{
    \In{Previous state $Y_{k}$; proposal distribution $\proposal$; rejuvenation kernel $\rejukernel$;}
    \Out{New sample $Y_{k+1}$; pool of proposals $\chunku{X}{2}{N}_{k+1} \sim \proposal$;}
    $Z_{k+1}\eqsp, \chunku{X}{2}{N}_{k+1} = \isir(Y_{k},\proposal)$; \\
    Draw $Y_{k+1} \sim \rejukernel(Z_{k+1},\cdot)$. \\
    }
    \vspace{-3pt}
\end{algorithm2e} 
\vspace{-3mm}

After each \isir\ step, we apply a local MCMC kernel $\rejukernel$ (rejuvenation kernel), with an invariant distribution $\target$. We call this startegy~\XTryM\ because it combines steps of exploration by \isir\ and steps of exploitation by the local MCMC moves. The resulting algorithm, formulated in \Cref{alg:X-Try-MCMC}, defines a Markov chain $\sequence{Y}[j][\nset]$ with Markov kernel $\XtryK(x,\cdot)= \MKisir \rejukernel(x,\cdot)= \int \MKisir(x, \rmd y) \rejukernel(y,\cdot)$. 



We now present the main theoretical result of this paper on the properties of \XTryM. Under rather weak conditions, provided that $\rejukernel$ is geometrically regular (see \cite[Chapter~14]{douc:moulines:priouret:2018}), it is possible to establish that \XTryM\ remains $V$-uniformly geometrically ergodic even if the weight function $\weightfunc(x)$ is unbounded.
\vspace{-1mm}
\begin{definition}[$V$-Geometric Ergodicity]
\label{def:geometric-ergodicity} \index{Markov chain!geometrically ergodic}
A Markov kernel $\MKQ$ with invariant probability measure $\target$ is $V$-geometrically ergodic if there exist constants $\rho \in (0,1)$ and $M < \infty$ such that, for all $x \in \Xset$ and $k \in \nset$, $\Vnorm{\MKQ^k(x,\cdot) - \target} \leq M \, \{ V(x) + \target(V) \} \rho^k$.
\end{definition}
\vspace{-2mm}
In particular, $V$-geometric ergodicity ensures that the distribution of the $k$-th iterate of a Markov chain converges geometrically fast to the invariant probability in $V$-norm, for all starting points $x \in \Xset$. Here the dependence on the initial state $x$ appears on the right-hand side only in $V(x)$. Denote by $\PVar_{\proposal}[\weightfunc] = \int \{\weightfunc(x) - \proposal(\weightfunc) \}^2 \proposal(\rmd x)$ the variance of the  importance weight functions under the proposal distribution
and consider the following assumptions:
\begin{assumption}
\label{assum:rejuvenation-kernel}
(i) $\rejukernel$ has $\target$ as its unique invariant distribution;
(ii) There exists a function $V\colon \Xset \to \coint{1,\infty}$, such that for all $r \geq r_{\rejukernel} > 1$ there exist $\lambda_{\rejukernel,r} \in \coint{0,1}$, $\bconst[\rejukernel,r] < \infty$, such that $\rejukernel V(x) \leq \lambda_{\rejukernel,r} V(x) + \bconst[\rejukernel,r] \indi{\msv_r}$, where $\msv_r = \{x\colon V(x) \leq r \}$;
\end{assumption}
\begin{assumption}
\label{assum:independent-proposal}
(i) For all $r \geq r_{\rejukernel}$, $\weightfunc_{\infty, r} := \sup_{x\in\msv_r} \{\weightfunc(x) / \proposal(\weightfunc)\} <\infty$ and (ii) $\PVar_{\proposal}[\weightfunc]/ \{ \proposal(\weightfunc) \}^2 < \infty$.
\end{assumption}
\Cref{assum:rejuvenation-kernel}-(ii) states that $\rejukernel$ satisfies a Foster-Lyapunov drift condition for $V$. This condition is fulfilled by most classical MCMC kernels - like Metropolis-Adjusted Langevin (MALA) algorithm or Hamiltonian Monte Carlo (HMC), typically under tail conditions for the target distribution; see~\cite{roberts2004general,moulines:durmus:2022}, and~\cite[Chapter~2]{douc:moulines:priouret:2018} with the references therein. \Cref{assum:independent-proposal}-(i) states that the (normalized) importance weights $\weightfunc(\cdot)/\proposal(\weightfunc)$ are upper bounded on level sets of $\msv_r$.
This is a mild condition: if $\Xset=\rset^d$, and $V$ is norm-like, then the level sets $\msv_r$ are compact and $\weightfunc(\cdot)$ is bounded on $\msv_r$ as soon as $\pdftarget$ and $\pdfproposal$ are positive and continuous.
\Cref{assum:independent-proposal}-(ii) states that the variance of the importance weights is bounded; note that this variance is also equal to the $\chi^2$-distance between the proposal and the target distributions which plays a key role in the non-asymptotic analysis of the performance of IS methods \citep{agapiou2017importance,sanz2018importance}.
\begin{theorem}
\label{theo:main-geometric-ergodicity}
    Assume \Cref{assum:rejuvenation-kernel} and \Cref{assum:independent-proposal}. Then, for all $x \in \Xset$ and $k \in \nset$,
    \begin{equation}
    \label{eq:geometric-ergodicity-XtryK}
      \Vnorm{\XtryK^{k}(x, \cdot) - \target} \leq \cconst[\XtryK] \{ \target(V) + V(x) \} \rate[\XtryK]^k\eqsp,
    \end{equation}
    where the constant $\cconst[\XtryK] $, $\rate[\XtryK] \in \coint{0,1}$ are given in the proof. In addition,
    $\cconst[\XtryK] = \cconst[{\XtryK[\infty]}] + O(N^{-1})$ and $\rate[{\XtryK[\infty]}]= \rate[\XtryK]  + O(N^{-1})$ with explicit expressions provided in \eqref{eq:supplement_constants_ex2}.
\end{theorem}
\vspace{-3mm}
The proof of \Cref{theo:main-geometric-ergodicity} is provided in \Cref{supp:subsec:proof:main-geometric-ergodicity}. We stress that in many situations, the mixing rate $\rate[\XtryK]$ of the \XTryM~Markov Kernel $\XtryK$ is significantly better than the corresponding mixing rate of the local kernel $\rejukernel$, provided $N$ is large enough. {
This is due to the fact that assumptions \Cref{assum:rejuvenation-kernel} and \Cref{assum:independent-proposal} do not require to identify the small sets of the rejuvenation kernel $\rejukernel$ (see \cite[Definition~9.3.5]{douc:moulines:priouret:2018}). At the same time, the quantitative bounds on the mixing rates relies on the constants appearing in the small set condition, see \cite[Theorem~19.4.1]{douc:moulines:priouret:2018}. 
Focusing on MALA (see, e.g.~\citep{roberts:tweedie:1996}) as the rejuvenation kernel $\rejukernel$ 
we detail bounds in \Cref{supp:sec:MALA} and prove in \Cref{th:mix_time_improvement} that the ratio of mixing times of $\XtryK$ is typically very favorable compared to MALA provided that $N$ is large enough. }



\vspace{-3mm}
\section{Adaptive version: \FlXTryM}
\label{sec:adaptive_flow}
%
\vspace{-2mm}
The performance of proposal-based samplers depends on the distribution of importance weights which is related to the similarity of the proposal and target distributions\footnote{more specifically, it depends on the the quantities appearing in \Cref{assum:independent-proposal}, namely, the maximum of the importance weight on a level set of the drift function for the local kernel $\rejukernel$ and the variance of the importance weights under the proposal}.
Therefore, yet another strategy to improve sampling performance is to select the proposal distribution $\proposal$ from a family of parameterized distributions $\left\{\proposal_{\theta}\right\}$ and fit the parameter $\theta \in \Theta = \rset^q$ to the target $\target$, for example, by minimizing a Kullback-Leibler divergence (KL) \cite{Parno2018, albergo:2019, Naesseth2020} or matching moments \cite{Pompe2020}. In \emph{adaptive MCMCs}, parameter adaptation is performed along the MCMC run  \cite{andrieu2006ergodicity,andrieu2008tutorial,roberts2009examples}. In this section we propose an adaptive version of \XTryM, which we call \FlXTryM.

\textbf{Normalizing flow proposal.}
A flexible way to parameterize proposal distributions is to combine a tractable distribution $\proposalbase$ with an invertible parameterized transformation. Let $T: \Xset \mapsto \Xset$ be a $\rmc^1$ diffeomorphism. We denote by $\pushforward[T][\proposalbase]$ the push-forward of $\proposalbase$ under $T$, that is, the distribution of $X=T(Z)$ with $Z\sim\proposalbase$. Assuming that $\proposalbase$ has a p.d.f. (also denoted $\pdfproposalbase$), the corresponding push-forward density (w.r.t. the Lebesgue measure) is given by $\pdfproposal_T(y) = \pdfproposalbase\bigl(T^{-1}(y)\bigr) \jacop{T^{-1}}(y)$, where $\jacop{T}$ denotes the Jacobian determinant of $T$. The parameterized family of diffeomorphisms $\{T_\theta\}_{\theta \in \Theta}$ defines a family of distributions $\{\pdfproposal_{T_\theta}\}_{\theta \in \Theta}$, denoted for simplicity as $\{\pdfproposal_\theta\}_{\theta \in \Theta}$.
This construction is called a \emph{normalizing flow} (NF) and a great deal of work has been devoted to ways of parameterizing invertible flows $T_{\theta}$ with neural networks; see~\citep{kobyzev2020normalizing,papamakarios2019normalizing} for reviews.

\textbf{Simultaneous learning and sampling.} {As with adaptive MCMC methods, the parameters of a NF proposal are learned for the global proposal during sampling, see also \cite{Gabrie2021}.} We work with $M$ copies of the Markov chains $\{(Y_k[j], \chunku{X_k}{1}{N}[j])\}_{k \in \nsets}$ indexed by $j \in \{1,\dots, M \}$. At each step $k \in \nsets$, each copy is sampled as in \XTryM~using the NF proposal, independently from the other copies, but conditionally to the the current value of the parameters $\theta_{k-1}$.
We then adapt the parameters by taking steps of gradient descent on a convex combination of the \emph{forward} KL, $\KL{\pdftarget}{\pdfproposal_\theta} = \int_{} \pdftarget(x) \log(\pdftarget(x)/\pdfproposal_\theta(x)) \rmd x$ and the backward KL  $  \KL{\pdfproposal_\theta}{\pdftarget} = \int \pdfproposal_\theta(x) \log (\pdftarget(x)/\pdfproposal_\theta(x)) \rmd x = \int \pdfproposalbase(z) \log \weightfunc_\theta \circ T_\theta(z)  \rmd z$.
Let $\sequence{\gamma}[k][\nset]$ be a sequence of nonnegative stepsizes and $\sequence{\alpha}[k][\nset]$ be a nondecreasing sequence in $\ccint{0,1}$ with $\alpha_{\infty}= \lim_{k \to \infty} \alpha_k$.
The update rule is $\theta_k= \theta_{k-1} + \gamma_k M^{-1} \sum_{j=1}^M  H(\theta_{k-1},\chunku{X_k}{1}{N}[j],\chunku{Z_k}{2}{N}[j])$ where  $H(\theta,\chunku{x}{1}{N},\chunku{z}{2}{N})= \alpha_k H^{f}(\theta,\chunku{x}{1}{N})+
(1-\alpha_k) H^b(\theta,\chunku{z}{2}{N})$ with
\begin{align}
\label{eq:fwdKL}
H^{f}(\theta,\chunku{x}{1}{N}) &=  \sum\nolimits_{\ell=1}^N \frac{\weightfunc_{\theta}(x^{\ell})}{\sum\nolimits_{i=1}^N \weightfunc_{\theta}(x^i)} \nabla_{\theta} \log \pdfproposal_{\theta}(x^{\ell})  \eqsp, \quad \weightfunc_\theta(x)= \pdftarget(x)/ \pdfproposal_\theta(x)\eqsp, \\
\label{eq:bwdKL}
H^b(\theta,\chunku{z}{2}{N})&= 
 - \frac{1}{N-1} \sum\nolimits_{\ell=2}^N \{ \nabla_\theta \log \pdftarget \circ T_\theta(z^\ell) + \nabla_\theta \log \jacop{T_\theta} (z^{\ell}) \}\eqsp.
\end{align}
Note that we use a Rao-Blackwellized estimator of the gradient of the forward KL \eqref{eq:fwdKL} where we fully recycle all the $N$ candidates sampled at each iteration of \isir. The quality of this estimator is expected to improve along the iterations $k$ of the algorithm as the variance of importance weights decreases as the proposal improves. Note also that using only gradients from the backward KL \eqref{eq:bwdKL} is prone to mode-collapse \cite{Parno2018,Noe2019,Naesseth2020,Gabrie2021}, hence the need for also using gradients from the forward KL $H^{f}(\theta,\chunku{x}{1}{N})$, which requires the simultaneous sampling from $\target$. {See also \Cref{app:subsec:mixture-highd} for further discussions.}
{The \FlXTryM\ algorithm is summarized in \Cref{alg:NF-X-Try-MCMC}.} 

Since the parameters of the Markov kernel $\theta_{k}$ are updated using samples $\chunku{X_k}{1}{N}$ from the chain, $\dsequence{Y}{\chunku{X}{1}{N}}[k][\nset]$ is no longer Markovian. This type of problems has been considered in \cite{metivier1987theoremes,benveniste1990adaptive,gu1998stochastic,andrieu2005stability} and to prove convergence of the strategy we need to strengthen assumptions compared to the previous section.
\begin{assumption}
\label{assum:condition-gradient}
There exists a function $W: \Xset \to \rset_+$ such that $\proposalbase(W^2)= \int W^2(z) \proposalbase(\rmd z) < \infty$, and a constant $L < \infty$ such that, for all $\theta, \theta' \in \Theta$ and $z \in \Xset$, $\Vert \nabla_\theta \log \target \circ T_\theta(z) - \nabla_\theta \log \target \circ T_{\theta'}(z)\Vert \leq L \| \theta - \theta' \| W(z)$ and $ \Vert \nabla_\theta \log \jacop{T_\theta} (z) - \nabla_\theta \log \jacop{T_{\theta'}} (z) \Vert \leq L \norm{\theta - \theta'} W(z)$.
\end{assumption}
\begin{assumption}
\label{assum:independent-proposal-strengthen}
(i) For all $d \geq d_{\rejukernel}$, $\weightfunc_{\infty, d}= \sup_{\theta \in \Theta} \sup_{x\in\msv_d} \weightfunc_\theta(x) / \proposal_\theta(\weightfunc_\theta) < \infty$ and (ii) $\sup_{\theta \in \Theta} \PVar_{\proposalbase}(\weightfunc_\theta \circ T_\theta)/ \{ \proposal_\theta(\weightfunc_\theta) \}^2 < \infty$.
\end{assumption}
\Cref{assum:condition-gradient} is a continuity condition on the NF push-forward density \wrt~its parameters $\theta$. \Cref{assum:independent-proposal-strengthen} implies that the Markov kernel $\XtryK[N,\theta] = \MKisir[N,\theta] \rejukernel$ satisfies a drift and minorization condition uniform in $\theta$.
\begin{theorem}[simplified]
\label{thm:KL-simplified}
Assume \Cref{assum:rejuvenation-kernel}-\Cref{assum:condition-gradient}-\Cref{assum:independent-proposal-strengthen} and that $\sum_{k=0}^\infty \gamma_k = \infty$, $\sum_{k=0}^{\infty} \gamma_k^2 < \infty$ and $\lim_{k \to \infty} \alpha_k= \alpha_\infty$. Then, w.p. 1, the sequence $\sequence{\theta}[k][\nset]$ converges to the set
$\{\theta \in \Theta,  0 = \alpha_\infty \nabla \KL{\target}{\proposal_\theta} + (1-\alpha_\infty) \nabla \KL{\proposal_\theta}{\target} \}$.
\end{theorem}
\Cref{thm:KL-simplified} proves the convergence of the learning of parameters $\theta$ to a stationary point of the loss. The proof is postponed to \Cref{sec:thm:KL-simplified}.
Note that once the proposal learning has converged, $\FlXTryM$ boils back to $\XTryM$ with a fixed learned proposal. 
Our experiments show that
adaptivity can significantly speed up mixing for \isir, especially for distributions with complex geometries and that the addition of a rejuvenation kernel
further improves samples quality.

{
\input{algo_flex}
}

\vspace{-4mm}
\section{Related Work}
\label{sec:related}



\vspace{-3mm}
The possibility to parametrize very flexible probabilistic models with neural networks thanks to deep learning has rekindled interest in adapting MCMC kernels; see \eg\  \cite{song2017nice,Hoffman2019,albergo:2019,Nicoli2020,Hackett2021}. While significant performance gain were found in problems of moderate dimensions, these learning-based methods were found to suffer from increasing dimensions as fitting models accurately becomes more difficult \cite{DelDebbio2021,Mahmoud2022}. Similarly to \FlXTryM, a few work proposed adaptive algorithms that alternates between global and local MCMC moves to ensure ergodicity without requiring a perfect learning of the proposal\cite{Pompe2020,Gabrie2021}. {More precisely}, \cite{Pompe2020} focused on multimodal distributions and analysed a mode jumping algorithm using proposals parametrized as mixture of simple distributions. While \cite{Gabrie2021}, closer to this work, introduced a combination of a local and a global sampler {leveraging normalizing flows} with a more classical choice for the global sampler: independent Metropolis-Hasting (IMH) instead of \isir. {The present work builds on these previous propositions of combinations of local and global sampler by clarifying the reasons of their effectiveness through entirely novel detailed mathematical and empirical analyses. We chose to focus on $\isir$ with an adaptive proposal as the global sampler since (i) the learning component allows to tackle high-dimensional targets, (ii) theoretical guarantees can be obtained for \isir~whereas IMH is more difficult to analyze, (iii) IMH and \isir~(as a multiple-try MCMC) are expected to have similar performances for comparable computational budget \cite{bedard2012scaling} but IMH is sequential where $\isir$ can be parallelized by increasing the number $N$ of proposals per iteration.}

Another line of work exploits both normalizing flows and common local MCMC kernels for sampling \cite{Parno2018,Hoffman2019,Noe2019,Zhang2022}, {yet following the different paradigm of using the flow as a reparametrization map, a method sometimes referred to as neural transport:} the flow $T$ is trained to transport a simple distribution $\proposalbase$ near $\target$, which is equivalent to bringing $\pushforward[T^{-1}][\target]$ (the pushforward of the original target distribution $\target$ by the inverse flow $T^{-1}$) close to $\proposalbase$. If $\proposalbase$ is simple enough to be efficiently sampled by local samplers, the hope is that local samplers can also obtain high-quality samples of $\pushforward[T^{-1}][\target]$ -- samples which can be transported back through $T$ to obtain samples of $\target$. This method {attempts} to reparametrize the space to disentangle problematic geometries for local kernels. Yet, it is unclear what will happen in the tails of the distribution for which the flow is likely poorly learned. Furthermore, in order to derive an ergodicity theory for these transported samplers, \cite{Parno2018} necessitated substantial constraints on maps (see section 2.2.2.).

\vspace{-3mm}
\section{Numerical experiments}
\label{sec:numerics}
\vspace{-3mm}
\subsection{Synthetic examples}
\label{subsec:toy_examples}
\vspace{-3mm}
\paragraph{Multimodal distributions.}
Let us start with a toy example highlighting differences between purely global \isir, purely local MALA and \XTryM\ combining both. We consider sampling from a mixture of $3$ equally weighted Gaussians in dimension $d = 2$. In \Cref{fig:gaus_example_2d}, we compare single chains produced by each algorithms. The global proposal is a wide Gaussian, with pools of $N=3$ candidate. The MALA stepsize is chosen to reach a target acceptance rate of $\sim 0.67$. This simple experiment illustrates the drawbacks of both approaches: \isir\ samples reach all the modes of the target, but the chains often get stuck for several steps hindering variability. MALA allows for better local exploration of each particular mode, yet it fails to cover all the target support. Meanwhile, \XTryM\ retains the benefits of both methods, combining the \isir-based global exploration with MALA-based local exploration.

\vspace{-1mm}
{In larger dimensions, an adaptive proposal is necessary. In \Cref{app:subsec:mixture-highd} we show that \FlXTryM~can mix between modes of a $50d$ Gaussian mixture, provided that the rough location of all the modes is known and used to initialize walkers. We also stress the robustness of the on-the-fly training exploiting running MCMC chains to evaluate the forward KL term of the loss.}

To illustrate further the performance of the {combined kernel}, we keep the $2d$ target mixture model yet assigning the uneven weights $(2/3, 1/6, 1/6)$ to the $3$ modes. We start $M$ chains drawing from the initial distribution $\xi \sim \mathcal{N}(0,4\Id_d)$ and use the same hyper-parameters as above. In \Cref{fig:mog_burnin} we provide a simple illustration to the statement \eqref{eq:tv_dist_isir} and \Cref{theo:main-geometric-ergodicity}, namely we compare  the target density to the instantaneous distributions for each sampler propagating $\xi$ during burn-in steps. As MALA does not mix easily between modes, the different statistical weights of the different modes can hardly be rendered in few iterations and KL and TV distances stalls after a few iterations. \isir\ can visit the different modes, yet it does not necessarily move at each step which slows down its covering of the modes full support, which again shows in the speed of decrease of the TV and KL. Overcoming both of these shortcomings, \XTryM\ instantaneous density comes much closer to the target. Finally, \Cref{fig:mog_single_chain} evaluates the same metrics yet for the density estimate obtained with single chain samples after burn-in. Results demonstrate once again the superiority of \XTryM.
Further details on these experiments can be found in \Cref{supp:subsec:mixtures}.

\begin{figure}[t]
\begin{minipage}{0.69\textwidth}
\centering
{\includegraphics[width=1.0\textwidth]{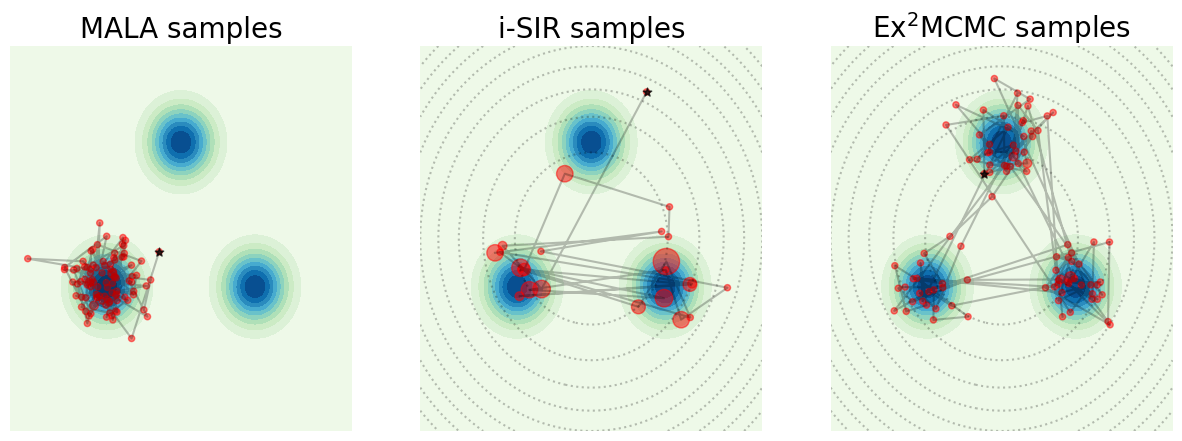}\label{fig:gaus_example_2d}}
\end{minipage}
\begin{minipage}{0.3\textwidth}
\centering
\begin{subfigure}{1.0\textwidth}
    \centering
    \captionsetup{justification=centering}
    \includegraphics[width=0.6\textwidth]{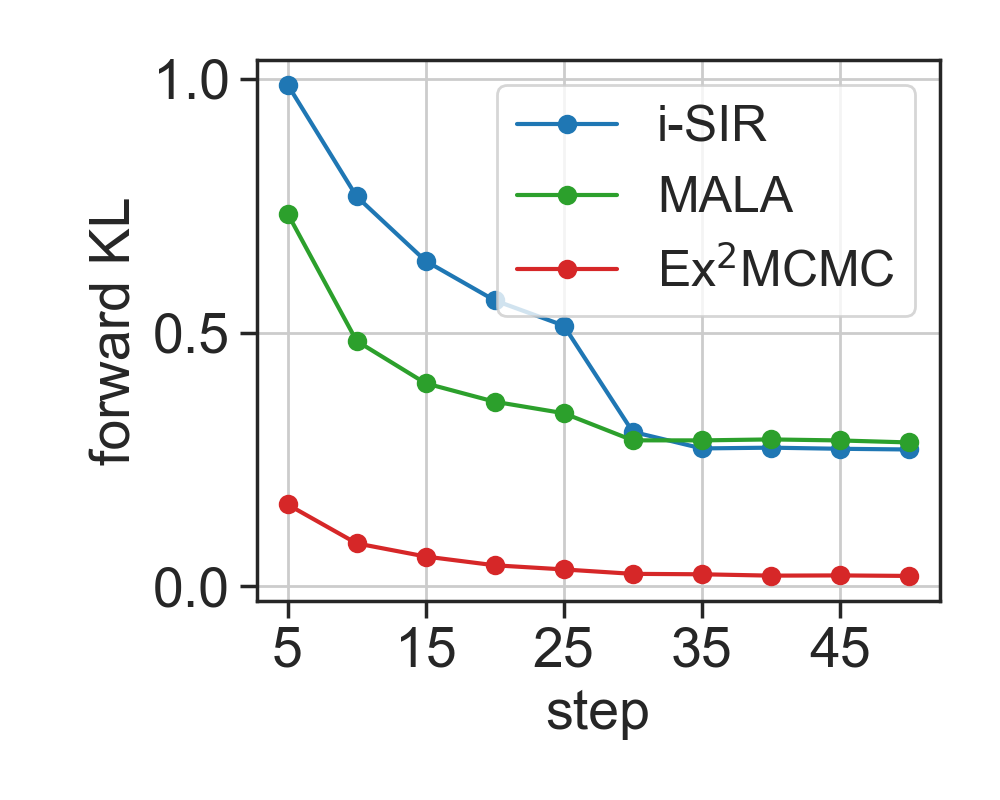}
    \vspace{-3mm}
    \caption{}
    \label{fig:mog_burnin}
\end{subfigure}
\\
\vspace{-1mm}
\begin{subfigure}{1.0\textwidth}
    \centering
    \captionsetup{justification=centering}
    \includegraphics[width=0.6\textwidth]{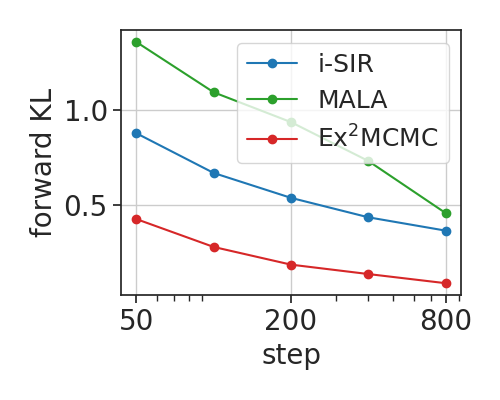}
    \vspace{-3mm}
    \caption{}
    \label{fig:mog_single_chain}
\end{subfigure}
\label{fig:gaus_2d_loss_plots}
\end{minipage}
\vspace{-1mm}
\caption{(a) -- Single chain mixing visualization. -- Blue color levels represent the target 2d density. Random chain initialization is noted in black, $100$ steps are plotted per sampler: the size of each red dot corresponds to the number of consecutive steps the walkers remains at a given location. 
Note that the variance of the global proposal (dotted countour lines) should be relatively large to cover well all the modes. 
(b - c) -- Inhomogeneous 2d Gaussian mixture. -- Quantitative analysis during burn-in of parallel chains (b, $M=500$ chains KDE) and for after burn-in for single chains statistics (c, $M=100$ average).}
\vspace{-5mm}
\end{figure}
\vspace{-4mm}
\paragraph{Distributions with complex geometry.}
Next, we turn to highly anisotropic distributions in high dimensions. Following \cite{neal:slice:2003} and \cite{haario1999banana}, we consider the \emph{funnel} and the \emph{banana-shape} distributions. We remind densities in \Cref{supp:subsec:complex-geometry} along with providing experiments details. For $d \in [10; 200]$, we run \isir, MALA, \XTryM, \FlXTryM, adaptive \isir~ (using the same proposal as \FlXTryM, but without interleaved local steps) and the versatile sampler NUTS~\cite{hoffman2014no} as a baseline. Here the parameter adaptation for $\FlXTryM$ is performed in a pre-run and parameters are frozen before sampling.
For the adaptive samplers, a simple RealNVP-based normalizing flow \cite{dinh:2016} is used such that total running times, including training, are comparable with NUTS. For \XTryM~and \isir~the global proposal is a wide Gaussian with a pool of $N = 2000$ candidates drawn at each iteration. For MALA we tune the step size in order to keep acceptance rate approximately at $0.5$. We report the average sliced TV distance and ESS in Figure~\ref{fig:res_100d} (see \Cref{subsec:supp:metrics} for metrics definition). In most cases, \FlXTryM~is the most reliable algorithm. The only exception is at very high dimension for the banana where NUTS performs the best: in this case, tuning the flow to learn tails in high-dimension faithfully was costly such that we proceeded to an early stopping to maintain comparability with the baseline. Remarkably, \FlXTryM\ compensates significantly for the imperfect flow training, improving over adaptive-\isir, but NUTS eventually performs better. Conversely, for the funnel, most of the improvement comes from well-trained proposal flow, leading to similar behaviors of adaptive \isir\ and \FlXTryM, while both algorithms clearly outperforms NUTS in terms of metrics.


\begin{figure}[t]
\centering
\begin{subfigure}{0.56\textwidth}
    \includegraphics[width=1.0\textwidth]{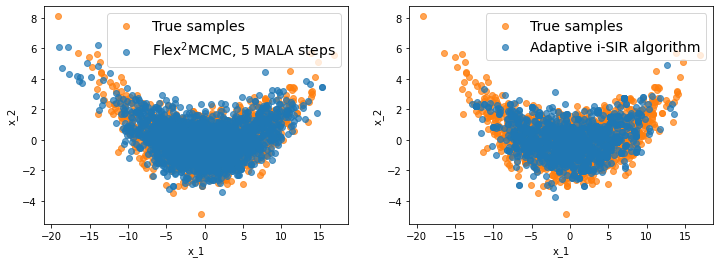}
    \caption{$d = 100$, $2000$ samples projection}
    \label{fig:samples_100d_banana}
\end{subfigure}
\begin{subfigure}{0.42\textwidth}
    \includegraphics[width=1.0\textwidth]{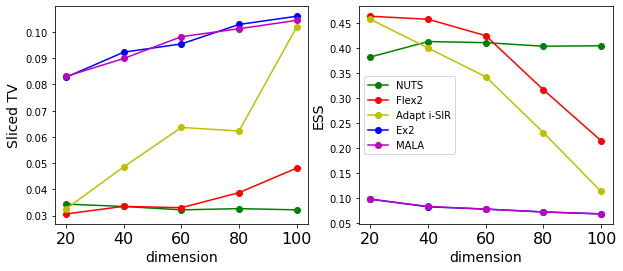}
    \caption{Banana-shape distribution}
    \label{fig:metrics_banana}
\end{subfigure}
\begin{subfigure}{0.56\textwidth}
    \includegraphics[width=1.0\textwidth]{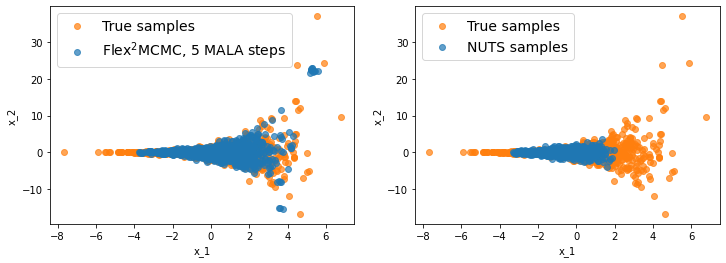}
    \caption{$d = 100$, $1000$ samples projection}
    \label{fig:samples_100d_funnel}
\end{subfigure}
\begin{subfigure}{0.42\textwidth}
    \includegraphics[width=1.0\textwidth]{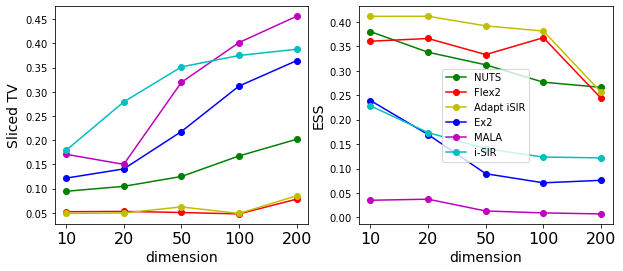}
    \caption{Neal's funnel}
    \label{fig:metrics_funnel}
\end{subfigure}
\caption{Anisotropic Funnel and Banana-shape distributions -- (a) and (b) visualize samples projected onto the first 2 coordinates of tested algorithms (blue) versus true samples obtained by reparametrization (orange). (c) and (d) compare Sliced Total Variation and Effective Sample Size as a function of dimension. \isir\ is removed from (b) as corresponding metrics for $d > 20$ are significantly worse.}
\label{fig:res_100d}
\vspace{-5mm}
\end{figure}


\vspace{-4mm}
\subsection{Sampling from GANs as Energy-based models (EBMs)}
\vspace{-3mm}
Generative adversarial networks (GANs \cite{10.5555/2969033.2969125}) are a class of generative models defined by a pair of a generator network $G$ and a discriminator network $D$. The generator $G$ takes a latent variable $z$ from a prior density $p_0(z)$, $z \in \rset^d$, and generates an observation $G(z) \in \rset^{\dim}$ in the observation space. The discriminator takes a sample in the observation space and aims to discriminate between true examples and false examples produced by the generator. Recently, it has been advocated to consider GANs as Energy-Based Models (EBMs)~\cite{Turner2019,Che2020}.
Following~\cite{Che2020}, we consider the EBM model induced by the GAN in latent space.
Recall that an EBM is defined by a Boltzmann distribution $p(z) = \rme^{-E(z)}/\normconst$, $z \in \rset^d$, where $E(z)$ is the energy function and $\normconst$ is the normalizing constant. Note that Wasserstein GANs also allow for an energy-based interpretation (see~\cite{Che2020}), although the interpretation of the discriminator in this case is different. The energy function is given by
\begin{equation}
\label{eq:energy_funs_GANS}
\textstyle{E_{JS}(z)= -\log p_0(z) - \logit\bigl(D(G(z))\bigr)\eqsp, \quad E_{W}(z) = -\log p_0(z) - D(G(z)\bigr), \quad z \in \rset^d}\eqsp,
\end{equation}
for the vanilla Jensen-Shannon and Wasserstein GANs, respectively. Here $\logit(y),\, y \in (0,1)$ is the inverse of the sigmoid function and $p_0(z) = \mathcal{N}(0,\Id_{d})$.
\vspace{-3mm}
\paragraph{MNIST results.}
We consider a simple Jensen-Shannon GAN model trained on the MNIST dataset with latent space dimension $d = 2$. We compare samples obtained by \isir, MALA, and \XTryM\ from the energy-based model associated with $E_{JS}(z)$, see \eqref{eq:energy_funs_GANS}. We use a wide normal distribution as the global proposal for \isir\ and \XTryM, and pools of candidates at each iteration $N = 10$. The step-size of MALA is tuned to keep an acceptance rate $\sim 0.5$. We visualize chains of $100$ steps in the latent space obtained with each method in \Cref{fig:mnist_energy_plots_js}. Note that the poor agreement between the proposal and the landscape makes it difficult for \isir\ to accept from the proposal and for MALA to explore many modes of the latent distribution, as shown in \Cref{fig:mnist_energy_plots_js}. \XTryM\ combines effectively global and local moves, encouraging better diversity associated with a better mixing time. The images corresponding to the sampled latent space locations are displayed in \Cref{fig:mnist_images_visualize} and reflect the diversity issue of MALA and $\isir$. Further details and experiments are provided in \Cref{supp:sec:MNIST}, including similar results for WGAN-GP \cite{gulrajani:2017:wgan} and the associated EBM $E_{W}(z)$.
\vspace{-3mm}

\begin{figure}[t]
\centering
\begin{subfigure}{0.8\textwidth}
\captionsetup{justification=centering}
\includegraphics[width=1.0\textwidth]{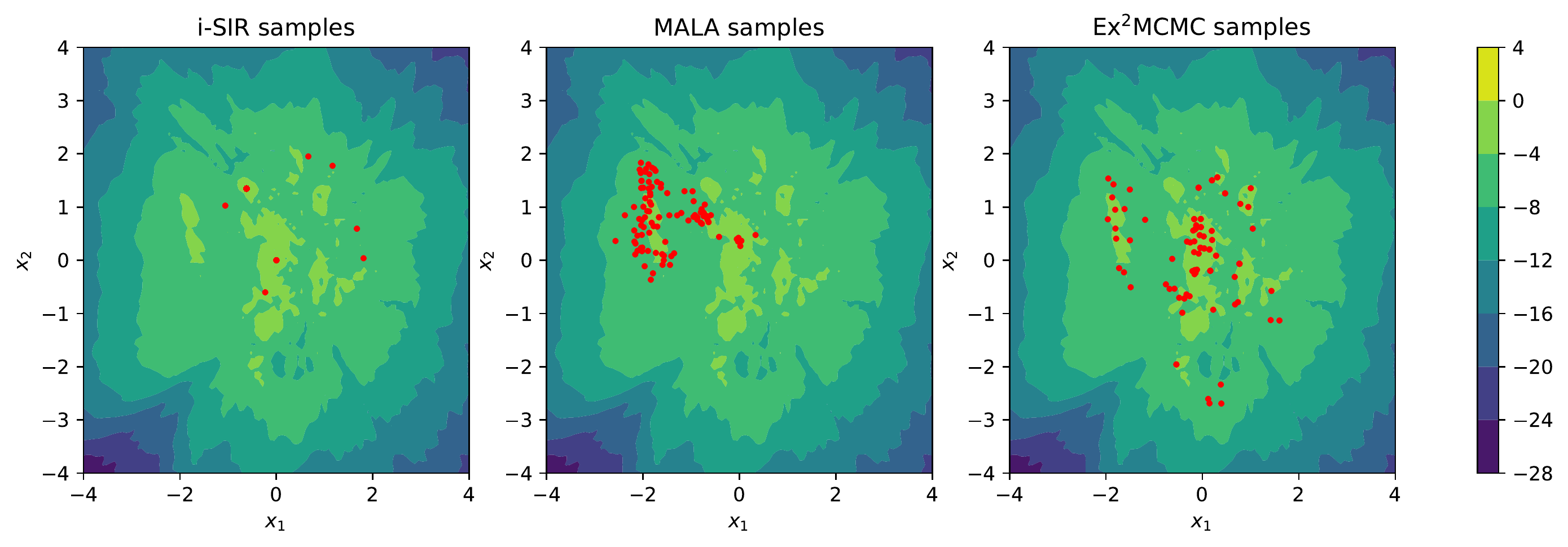}
\end{subfigure}
\vspace{-2mm}
\caption{MNIST energy landscape and single chain latent samples visualizations.}
\label{fig:mnist_energy_plots_js}
\vspace{-5mm}
\end{figure}

\paragraph{Cifar-$10$ results.}
We consider two popular architectures trained on Cifar-10, DC-GAN~\cite{radford2016unsupervised} and SN-GAN~\cite{miyato2018spectral}. In both cases the dimension of the latent space equals $d = 128$. Together with the non-trivial geometry of the corresponding energy landscapes, the large dimension makes sampling with NUTS unfeasible in terms of computational time. We perform sampling from mentioned GANs as energy-based models using \isir, MALA, \XTryM, and \FlXTryM. In \isir\ and \XTryM\ we use the prior $p_{0}(z)$ as a global proposal with a pool of $N=10$ candidates. For \FlXTryM\, we perform training and sampling simultaneously. Implementation details are provided in \Cref{supp:sec:cifar}. To evaluate sampling quality, we report the values of the energy function $E(z)$, averaged over $500$ independent runs of each sampler. {We also visualize the inception score (IS) dynamics calculated over $10000$ independent trajectories}. We present the results in \Cref{fig:cifar_10_pictures} together with the images produced by each sampler. Note that \XTryM\ and \FlXTryM\ reach low level of energies faster than other methods, {and reach high IS samples in a limited number of iterations}. Visualizations indicate that MALA is unlikely to escape the mode of the distribution $p(z)$ it started from, while \isir\ and \XTryM/\FlXTryM\ better explores the target support. However, global move appear to become more rare after some number of iterations for \XTryM/\FlXTryM, which then exploit a particular mode with MALA steps. We here hit the following limitation: $\isir$ remains at relatively high-energies, failing to explore well modes basins but still accepting global moves, while \XTryM/\FlXTryM\ explores well modes basins but eventually remains trapped. We predict that improving further the quality of the \FlXTryM\ proposal by scaling the normalizing flow architecture would allow for more global moves.See \Cref{supp:sec:cifar} for additional experiments (including ones with SN-GAN), {FID dynamics}, and visualizations.

\begin{figure}[t]
\centering
\begin{subfigure}{0.29\textwidth}
\includegraphics[width=1.0\textwidth]{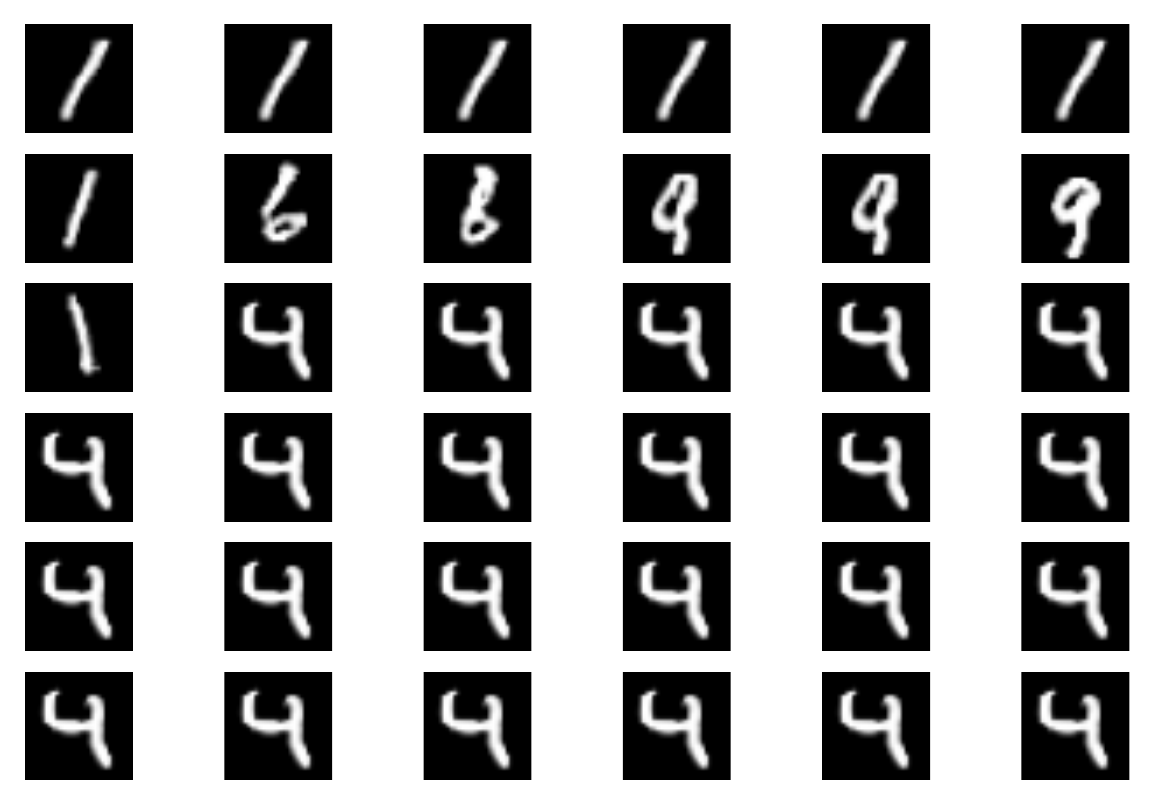}
\caption{\isir\ samples}
\end{subfigure}
\hspace{4mm}
\begin{subfigure}{0.29\textwidth}
\includegraphics[width=1.0\textwidth]{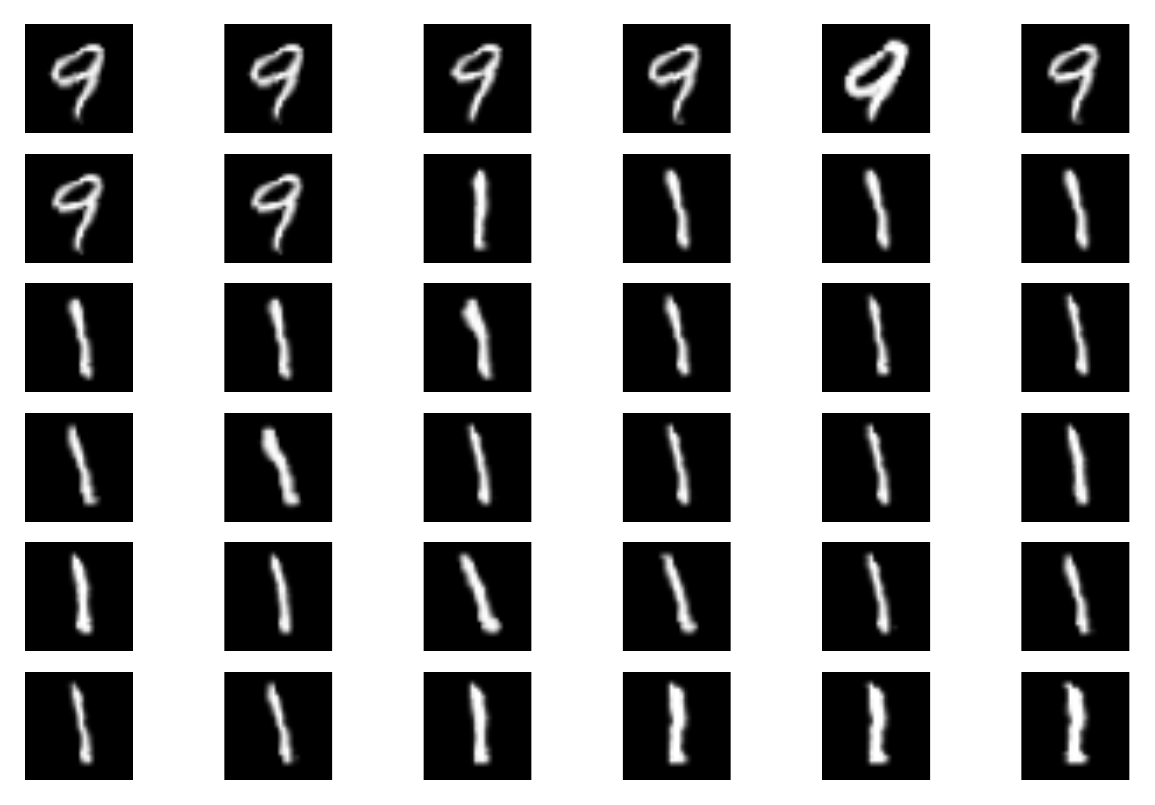}
\caption{MALA samples}
\end{subfigure}
\hspace{4mm}
\begin{subfigure}{0.29\textwidth}
\includegraphics[width=1.0\textwidth]{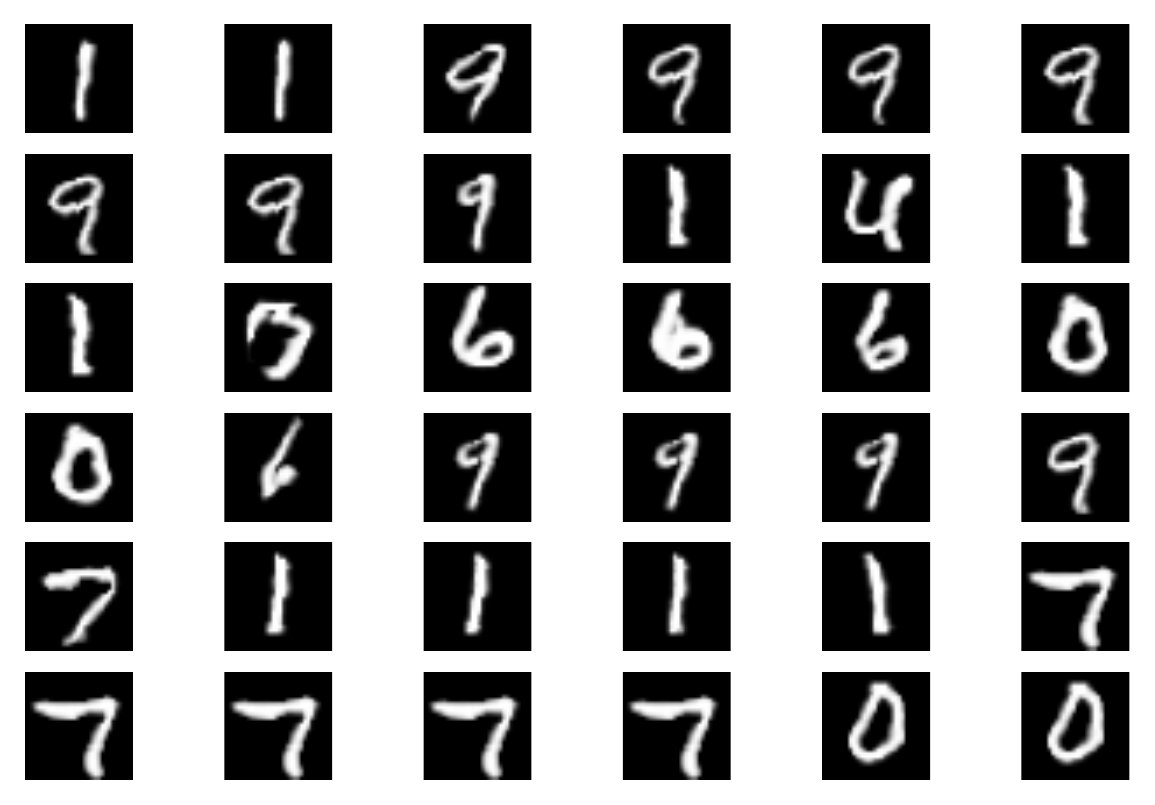}
\caption{\XTryM\ samples}
\end{subfigure}
\caption{MNIST samples visualization. -- Single chains run, sequential steps.}
\label{fig:mnist_images_visualize}
\vspace{-5mm}
\end{figure}
\begin{figure}[t]
\centering
\begin{subfigure}{0.3\textwidth}
\captionsetup{justification=centering}
\includegraphics[width=1.0\textwidth]{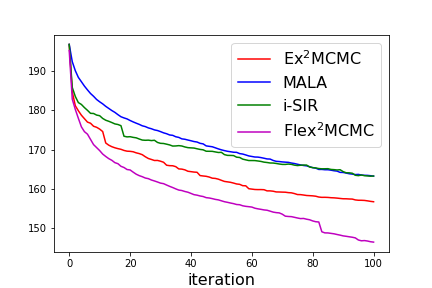}
\caption{Energy decay for $100$ iterations}
\end{subfigure}
\begin{subfigure}{0.25\textwidth}
\captionsetup{justification=centering}
\includegraphics[width=1.0\textwidth]{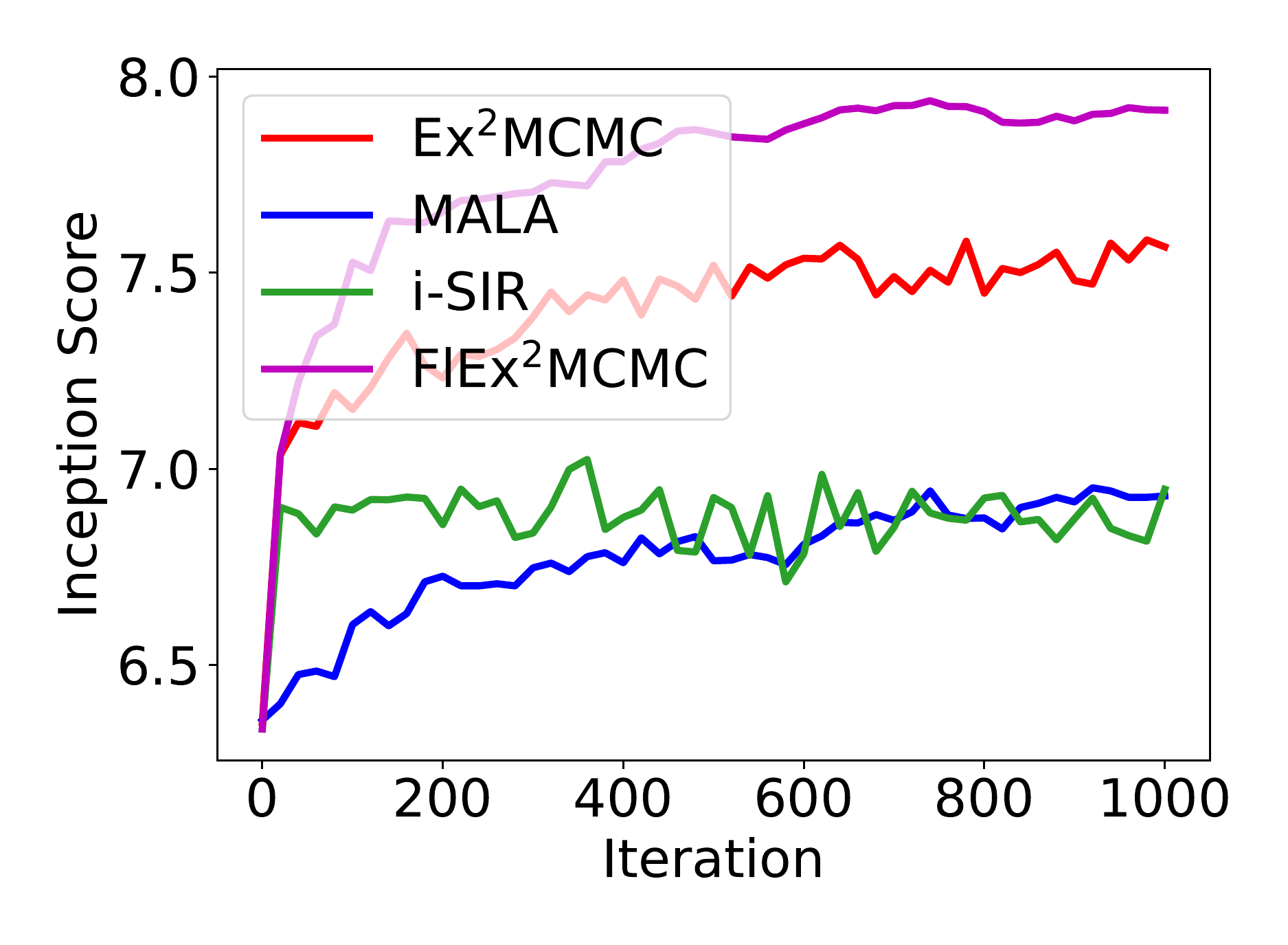}
\caption{IS dynamics, $1000$ iterations}
\end{subfigure}
\begin{subfigure}{0.4\textwidth}
\includegraphics[width=1\textwidth]{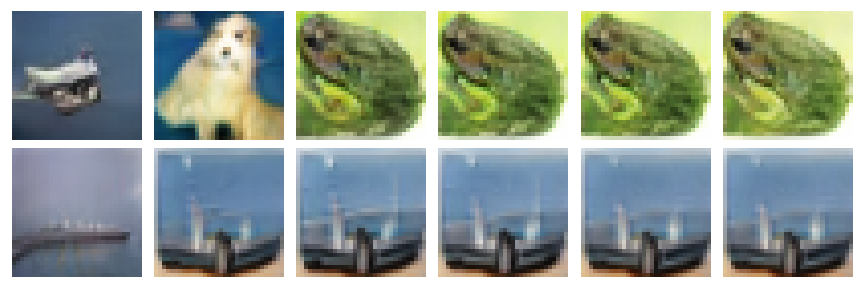}
\caption{\XTryM\ samples}
\label{fig:samples_cifar_iters}
\end{subfigure}

\begin{subfigure}{0.48\textwidth}
\captionsetup{justification=centering}
\includegraphics[width=1.0\textwidth]{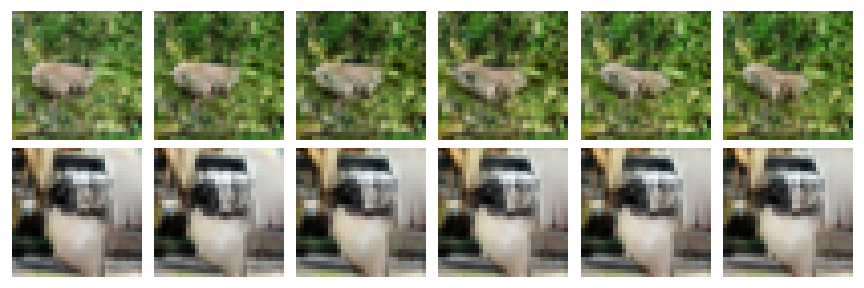}
\caption{MALA samples}
\label{fig:mala_cifar_10_vis}
\end{subfigure}
\begin{subfigure}{0.48\textwidth}
\captionsetup{justification=centering}
\includegraphics[width=1.0\textwidth]{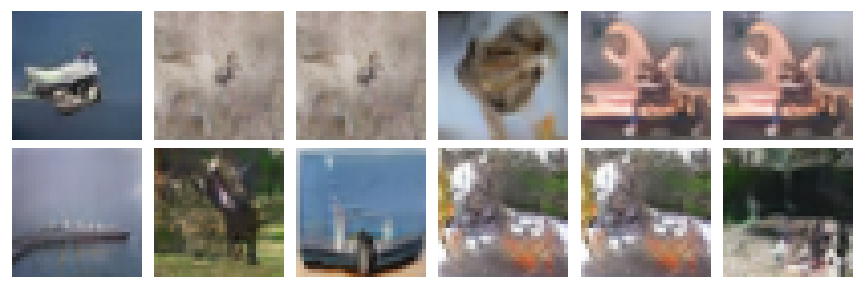}
\caption{\isir\ samples}
\label{fig:isir_cifar_10_vis}
\end{subfigure}
\caption{Cifar-$10$ energy and sampling results with DC-GAN architecture. Along the horizonthal lines we visualize each $10$th sample from a single trajectory.}
\label{fig:cifar_10_pictures}
\vspace{-5mm}
\end{figure}

\vspace{-5mm}
\section{Conclusions and further research directions}
\label{sec:consclusion}
\vspace{-3mm}
The present paper examines the benefits of combining local and global samplers. From a theoretical point of view, we show that global samplers are more robust when coupled with local samplers. Namely, a $V$-geometric ergodicity is obtained for the $\XTryM$ kernel under minimal assumptions. Meanwhile, the global samplers drives exploration when properly adjusted. Therefore, we also describe the adaptive version $\FlXTryM$ of the strategy involving the learning of a global proposal parametrized by a normalizing flow. We also check for the learning convergence along the adaptive MCMC run. Finally, a series of numerical experiments confirms the superiority of the strategy, including the high-dimensional examples.
While the startegy was described and analyzed for the \isir~global kernel, we note that it would be possible to extend the theory to other independent global samplers. We expect that the benefit of the combination would remain. Further studies of \FlXTryM, in particular the derivation of its mixing rate, is an interesting direction for future work.

\clearpage
\newpage
\bibliography{main}

\clearpage
\section*{Checklist}
\begin{enumerate}

\item For all authors...
\begin{enumerate}
  \item Do the main claims made in the abstract and introduction accurately reflect the paper's contributions and scope?
    \answerYes
  \item Did you describe the limitations of your work?
    \answerYes
  \item Did you discuss any potential negative societal impacts of your work?
    \answerNA
  \item Have you read the ethics review guidelines and ensured that your paper conforms to them?
    \answerYes{The paper suggests novel MCMC technique and is validated on artificial and standard datasets.}
\end{enumerate}

\item If you are including theoretical results...
\begin{enumerate}
  \item Did you state the full set of assumptions of all theoretical results?
    \answerYes{See \Cref{sec:mcmc_w_theorems} and \Cref{sec:adaptive_flow} in the main text.}
    \item Did you include complete proofs of all theoretical results?
    \answerYes{Yes, the proofs of \Cref{sec:mcmc_w_theorems} and \Cref{sec:adaptive_flow} are provided in \Cref{supp:subsec:proof:lem:invariance-P-N} and \Cref{sec:thm:KL-simplified}.}
\end{enumerate}

\item If you ran experiments...
\begin{enumerate}
  \item Did you include the code, data, and instructions needed to reproduce the main experimental results (either in the supplemental material or as a URL)?
    \answerYes{Code to reproduce experiments is attached to the supplement. Due to size constraints, we are not available to present all the pre-trained GANs models for the section \Cref{sec:numerics}, but we intend to do so when possible.}
  \item Did you specify all the training details (e.g., data splits, hyperparameters, how they were chosen)?
    \answerYes{The hyperparameters are provided in the supplement paper.}
        \item Did you report error bars (e.g., with respect to the random seed after running experiments multiple times)?
    \answerYes{Partially yes, but not for all experiments.}
        \item Did you include the total amount of compute and the type of resources used (e.g., type of GPUs, internal cluster, or cloud provider)?
    \answerYes{We provide this information in the supplement paper.}
\end{enumerate}

\item If you are using existing assets (e.g., code, data, models) or curating/releasing new assets...
\begin{enumerate}
  \item If your work uses existing assets, did you cite the creators?
    \answerNA{We use only the common knowledge datasets.}
  \item Did you mention the license of the assets?
    \answerNA{}
  \item Did you include any new assets either in the supplemental material or as a URL?
    \answerNo{}
  \item Did you discuss whether and how consent was obtained from people whose data you're using/curating?
    \answerNA{}
  \item Did you discuss whether the data you are using/curating contains personally identifiable information or offensive content?
    \answerNA{}
\end{enumerate}

\item If you used crowdsourcing or conducted research with human subjects...
\begin{enumerate}
  \item Did you include the full text of instructions given to participants and screenshots, if applicable?
    \answerNA{}
  \item Did you describe any potential participant risks, with links to Institutional Review Board (IRB) approvals, if applicable?
    \answerNA{}
  \item Did you include the estimated hourly wage paid to participants and the total amount spent on participant compensation?
    \answerNA{}
\end{enumerate}
\end{enumerate}

\clearpage

\clearpage
\newpage
\appendix
\input{main-supplement}

\end{document}

%% file: def.tex
\def\taumix{t_{\operatorname{mix}}}
\newcommand{\dobrush}{\mathsf{\Delta}}
\newcommandx{\dobru}[3][1=,3=]{\dobrush_{#1}^{#3}( #2)}

\def\msa{\mathsf{A}}

\def\msk{\mathsf{K}}

\def\msc{\mathsf{C}}

\def\msv{\mathsf{V}}

\newcommandx{\CPE}[3][1=]{{\mathbb E}_{#1}\left[\left. #2 \, \right| #3 \right]} 
\newcommandx{\CPP}[3][1=]{{\mathbb P}_{#1}\left[\left. #2 \, \right| #3 \right]} 

\def\mcbb{\mathcal{B}}  

\def\mcf{\mathcal{F}}

\def\rset{\mathbb{R}}

\def\nset{\mathbb{N}}
\def\nsets{\mathbb{N}^*}

\def\Zset{\mathbb{Z}}

\def\rmd{\mathrm{d}}

\def\rme{\mathrm{e}}

\def\rmc{\mathrm{C}}

\newcommandx{\functionspace}[2][1=+]{\mathbb{F}_{#1}(#2)}

\newcommand{\1}{\mathbbm{1}}

\newcommand{\LeftEqNo}{\let\veqno\@@leqno}

\newcommand{\PE}{\mathbb{E}}
\def\PVar{\operatorname{Var}}
\newcommandx{\bconst}[1][1=]{\ensuremath{\mathsf{b}_{#1}}}
\newcommandx{\cconst}[1][1=]{\ensuremath{c_{#1}}}
\newcommandx{\rate}[1][1=]{\ensuremath{\tilde{\kappa}_{#1}}}
\newcommandx{\dconst}[1][1=]{\ensuremath{d_{#1}}}

\def\Rmala{R}
\def\rmala{r}

\def\Ltt{\mathtt{L}}

\def\Mtt{\mathtt{M}}

\def\rmD{\mathrm{D}}

\def\mtt{\mathtt{m}}
\def\Ktt{\mathtt{K}}
\def\tKtt{\tilde{\mathtt{K}}}

\newcommandx{\pushforward}[2][1=T,2=\Lambda]{ {#1} \# {#2}}
\newcommandx{\VnormEq}[2][1=V]{\left\| #2 \right\|_{#1}}
\newcommandx{\norm}[2][1=]{\ifthenelse{\equal{#1}{}}{\Vert #2 \Vert}{\Vert #2 \Vert^{#1}}}

\def\bareta{\bar{\eta}}
\def\bgamma{\bar{\step}}
\def\tildem{\varpi}

\def\step{\gamma}

\newcommand{\parentheseDeux}[1]{\left[ #1 \right]}
\def\alphaMALA{\alpha_{\gamma}}

\newcommand{\PP}{\mathbb{P}}
\newcommand{\tvnorm}[1]{\| #1 \|_{\mathrm{TV}}}
\newcommandx{\Vnorm}[2][1=V]{\| #2 \|_{#1}}

\newcommand{\ps}[2]{\left\langle#1,#2 \right\rangle}

\newcommand{\eqdef}{{:=}}

\def\ie{\textit{i.e.}}
\def\eqsp{\;}
\newcommand{\coint}[1]{\left[#1\right)}
\newcommand{\ocint}[1]{\left(#1\right]}

\newcommand{\ccint}[1]{\left[#1\right]}

\newcommand{\ocintLigne}[1]{(#1]}

\newcommand{\indi}[1]{\1_{#1}}

\newcommand{\indiaccl}[1]{\1\left\{ #1 \right\}}
\newcommand{\ball}[2]{\operatorname{B}(#1,#2)}

\def\as{\ensuremath{\text{a.s.}}}

\newcommandx\sequenceg[3][2=,3=]
{\ifthenelse{\equal{#3}{}}{\ensuremath{( #1_{#2})}}{\ensuremath{( #1_{#2})_{ #2 \geq #3}}}}

\newcommandx\sequence[3][2=,3=]
{\ifthenelse{\equal{#3}{}}{\ensuremath{\{ #1_{#2}\}}}{\ensuremath{\{ #1_{#2}, \eqsp #2 \in #3 \}}}}

\newcommandx{\sequencen}[2][2=n\in\N]{\ensuremath{\{ #1_n, \eqsp #2 \}}}

\newcommandx\dsequence[4][3=,4=]{\ensuremath{ ((#1_{#3}, #2_{#3}))_{#3 \in #4}}}

\newcommand{\wrt}{w.r.t.}

\def\iid{\text{i.i.d.}}

\def\eg{e.g.}

\def\Id{\operatorname{I}}

\def\pdftarget{\target}

\def\Psemigroup{\mathbf{P}}

\def\varepsilonula{\varepsilon}
\newcommand{\ceil}[1]{\left\lceil #1 \right\rceil}
\def\tGamma{\tilde{\Gamma}}
\def\Phibf{\mbox{\protect\boldmath$\Phi$}}
\newcommand{\parenthese}[1]{\left(#1 \right)}
\def\tbdriftulatGamma{\tilde{b}_{\tGamma_{1/2}}^{\mathrm{U}}}

\def\target{\pi}
\def\normconst{\operatorname{Z}}
\def\pdfproposal{\proposal}
\def\proposal{\lambda}
\def\proposalbase{\varphi}
\def\pdfproposalbase{\varphi}

\def\kernelMTM{\mathsf{T}}

\def\bound{\mathrm{L}}
\def\epssmallisir{\epsilon_N}

\def\driftconstisir{\kappa_N}

\def\bfX{\mathbf{X}}
\def\bfB{\mathbf{B}}
\def\YrULA{Y}
\def\YrMALA{X}
\def\gaStep{\gamma}
\def\ZrGaussian{Z}
\def\bdriftula{b_{\bgamma}^{\mathrm{U}}}

\def\bdriftmala{b_{\bgamma}^{\mathrm{M}}}
\def\ray{K}
\def\rayula{\ray^{\mathrm{U}}}
\def\raymaladrift{\ray^{\mathrm{M}}}
\def\Veta{V}
\def\tCtt{\tilde{\mathtt{C}}}
\def\bGamma{\bar{\Gamma}}

\def\bareta{\bar{\eta}}

\def\MKP{{\rm P}}

\def\MKQ{\mathsf{Q}}
\def\MKK{{\rm K}}
\def\dim{\mathsf{D}}

\newcommandx{\MKisir}[1][1=N]{\mathsf{P}_{#1}}
\def\Xset{\mathbbm{X}}
\def\Xsigma{\mathcal{X}}
\def\Eset{\mathbbm{E}}
\def\Esigma{\mathcal{E}}
\def\Zset{\mathbbm{Z}}

\newcommand{\chunku}[3]{#1^{#2:#3}}

\newcommand{\chunkum}[4]{#1^{#2:#3 \setminus \{#4\}}}

\newcommand{\range}[1]{[#1]}

\def\weightfunc{w}

\def\normweight{\omega}

\def\rejukernel{\mathsf{R}}
\newcommandx{\XtryK}[1][1=N]{\mathsf{K}_{#1}}

\def\logit{\operatorname{logit}}
\newcommand{\jacop}[1]{\operatorname{J}_{#1}}

\def\isir{\ensuremath{\operatorname{i-SIR}}}

\def\XTryM{\ensuremath{\operatorname{Ex^2MCMC}}}
\def\FlXTryM{\ensuremath{\operatorname{FlEx^2MCMC}}}

\newcommandx{\ePriorLatent}[1][1=N]{\bar{\Xi}_{#1}}

\def\densgauss{\operatorname{g}}

\def\eTarget{\bar{\target}_N}

\newcommand{\supnorm}[1]{| #1 |_{\infty}}

\newcommand{\KL}[2]{{\rm KL}(#1||#2)}

\def\partopopN{\boldsymbol{\Lambda}_N}
\def\eTarget{\boldsymbol{\varphi}_N}
\def\eTargetmarginx{\boldsymbol{\pi}_N}
\newcommandx{\targetkern}[1][1=N]{\Pi_{#1}}
\newcommandx{\utargetkern}[1][1=N]{\Gamma_{#1}}
\newcommandx{\MKisirjoint}[1][1=N]{\mathsf{\mathbf{P}}_{#1}}
\def\xijoint{{\boldsymbol{\xi}}}
\def\simiid{\overset{\text{i.i.d.}}{\sim}}

%% file: algo_flex.tex
 \begin{algorithm2e}[t!]
    \caption{Single stage of \FlXTryM. Steps of \XTryM\ use the NF proposal with parameters $\theta_k$. Step $4$ updates the parameters using the gradient estimate obtained from all the chains.}
    \label{alg:NF-X-Try-MCMC}
    \SetKwInOut{In}{Input}
    \SetKwInOut{Out}{Output}
    \SetKwFunction{KwPrint}{print}
    \SetKwProg{Fn}{Function}{:}{end}
    \SetKwFunction{X-Try}{X-Try SIR}
    \SetKwProg{myalg}{Algorithm}{}{}
    \In{weights $\theta_k$, batch $Y_{k}[1:M]$}
    \Out{new weights $\theta_{k+1}$, batch $Y_{k+1}[1:M]$}
    \For{$j \in \range{M}$}{
        $Y_{k+1}[j] 
        = \XTryM\ (Y_k,\pushforward[T_{\theta_k}][\proposalbase],\rejukernel)$
    }
    Draw ${Z}[1:M] \sim \proposalbase$.\\
    Update $\theta_k= \theta_{k-1} + \gamma_k M^{-1} \sum_{j=1}^M  H(\theta_{k-1},\chunku{X_k}{1}{N}[j],\chunku{Z_k}{2}{N}[j])$ 
 \end{algorithm2e}

%% file: main-supplement.tex
\section{\isir\ Algorithm}
\input{isir-proof}

\section{Proofs of \Cref{sec:mcmc_w_theorems}}
\label{supp:subsec:proof:lem:invariance-P-N}

\subsection{Uniform geometric ergodicity of the \isir\ Markov kernel}
\label{supp:subsec:proof:theo:isir_unfirom_ergodicity}
Here we provide a simple direct proof of the bound \eqref{eq:tv_dist_isir}. We preface the proof by a technical lemma.
\begin{lemma}
\label{lemma:aux_isir}
Let $\chunku{Y}{1}{M}$ be $M$ independent random variables, satisfying $\PE[Y_i] =1$, and $\PVar[Y_i] < \infty$ for $i \in \{1,\ldots,M\}$. Then, for $S_{M} = \sum_{i=1}^{M} Y_i$ and $a, b>0$
\begin{equation}
\PE\left[ \left(a+b S_{M} \right)^{-1} \right]\leq (a+b M/2)^{-1} + (4/a) \PVar[S_{M}] / M^{2}\eqsp.
\end{equation}
\end{lemma}
\begin{proof}
Let $K \geq 0$. Then we get
\begin{align}
\label{eq:inverse_sum}
\frac{1}{a+bS_{M}} &= \frac{1}{a+bS_{M}} \indiaccl{S_{M}< K} + \frac{1}{a+bS_{M}}\indiaccl{S_{M} \geq K} \leq \frac{1}{a+bK} + \frac{1}{a}\indiaccl{S_{M} < K}
\end{align}
and in particular, $\PE[(a+b S_{M})^{-1}]\leq (a+bK)^{-1} + a^{-1} \PP(S_{M} < K)$. By Markov's inequality,
\begin{equation}
\PP(S_{M} < K) = \PP(S_{M}-M< -(M-K)) \leq \frac{\PVar [S_{M}]}{(M-K)^{2}}
\end{equation}
In particular, for $K=M/2$, we have $\PP(S_{M} < K) \leq 4 \PVar [S_{M}] / M^{2}$.
\end{proof}

\begin{proof}[Proof of \eqref{eq:tv_dist_isir}]
For $(x, \msa) \in \Xset \times \Xsigma$, we get
\begin{align}
    \MKisir(x, \msa) &=\int \updelta_x(\rmd x^{1}) \sum_{i=1}^N \frac{\weightfunc(x^i)}{\sum_{j=1}^N \weightfunc(x^j)}\indi{\msa}(x^i) \prod_{j=2}^N\proposal(\rmd x^j) \\
    &=  \int \frac{\weightfunc(x)}{\weightfunc(x) + \sum_{j=2}^N \weightfunc(x^j)}\indi{\msa}(x) \prod_{j=2}^N\proposal(\rmd x^j) + \int \sum_{i=2}^N \frac{\weightfunc(x^i)}{\weightfunc(x) + \sum_{j=2}^N \weightfunc(x^j)}\indi{\msa}(x^i) \prod_{j=2}^N\proposal(\rmd x^j)\\
    &\geq \sum_{i=2}^N\int  \frac{\weightfunc(x^i)}{\weightfunc(x) +\weightfunc(x^i)+ \sum_{j=2, j\neq i}^N \weightfunc(x^j)}\indi{\msa}(x^i) \prod_{j=2}^N\proposal(\rmd x^j)\\
    &\overset{(a)}{\geq} \sum_{i=2}^N\int\target(\rmd x^i)\indi{\msa}(x^i)\int \frac{\proposal(\weightfunc)}{\weightfunc(x) +\weightfunc(x^i) +\sum_{j=2, j\neq i}^N \weightfunc(x^j)} \prod_{j=2, j\neq i}^N\proposal(\rmd x^j) \label{eq:P_N_lower_bound} \eqsp.
  \end{align}
Here in (a) we used Fubini's theorem together with $\weightfunc(x) \proposal(\rmd x) = \target(\rmd x) \proposal(\weightfunc)$. Finally, since the function $f\colon z\mapsto (z+a)^{-1}$ is convex on $\rset_+$ and $a>0$, we get for $i\in\{2,\dots, N\}$,
  \begin{align}
  \label{eq:lower-bound}
    &\int \frac{\proposal(\weightfunc)}{\weightfunc(x) + \weightfunc(x^i) +\sum_{j=2, j\neq i}^N \weightfunc(x^j)} \prod_{j=2, j\neq i}^N\proposal(\rmd x^j) \\
    &\quad \geq \frac{\proposal(\weightfunc)}{\int \weightfunc(x) + \weightfunc(x^i) +\sum_{j=2, j\neq i}^N \weightfunc(x^j)\prod_{j=2, j\neq i}^N\proposal(\rmd x^j)}  \\
    &\quad \geq \frac{\proposal(\weightfunc)}{\weightfunc(x) + \weightfunc(x^i) + (N-2)\proposal(\weightfunc)}\geq \frac{1}{2\bound + N-2}\eqsp.
  \end{align}
  With the bound above we obtain the inequality
  \begin{equation}
  \label{eq:minorise_condition}
    \MKisir(x, \msa)\geq \target(\msa)\times \frac{N-1}{2\bound + N -2} =  \epssmallisir \target(\msa)\eqsp.
  \end{equation}
  This means that the whole space $\Xset$ is $(1,\epssmallisir \target)$-small (see \cite[Definition~9.3.5]{douc:moulines:priouret:2018}). Since $\MKisir(x, \cdot)$ and $\target$ are probability measures, \eqref{eq:minorise_condition} implies
  \begin{equation}
  \tvnorm{\MKisir(x, \cdot) - \target} = \sup_{\msa \in \Xsigma} |\MKisir(x, \msa) - \target(\msa)| \leq 1-\epssmallisir = \driftconstisir \eqsp.
  \end{equation}
  The statement follows from \cite[Theorem~18.2.4]{douc:moulines:priouret:2018} applied with $m=1$.
\end{proof}

\subsection{Proof of  \Cref{theo:main-geometric-ergodicity}}
\label{supp:subsec:proof:main-geometric-ergodicity}
We preface the proof with some preparatory lemmas.
\begin{lemma}
\label{lemma:no_assum_minoration}
Let $\msk\subset\Xset$, such that $\weightfunc_{\infty, \msk} := \sup_{x\in\msk}\{\weightfunc(x)/\proposal(\weightfunc)\} <\infty$ and $\target(\msk)>0$. Then, for all $(x, \msa)\in\msk\times\Xsigma$, we get that
\begin{equation}
\MKisir (x, \msa) \geq \epsilon_{N,K} \target_K(\msa)\eqsp,
\end{equation}
with $\epsilon_{N,\msk} = (N-1)\target(\msk)/[2w_{\infty, \msk} + N-2]$ and $\target_\msk(\msa) = \target(\msa\cap\msk)/\target(\msk)$.
\end{lemma}
Note that if the weight function $\weightfunc$ is upper semi-continuous, then for any compact $\msk$, $w_{\infty, \msk}= \sup_{x\in\msk}\weightfunc(x) <\infty$. Moreover, $\lim_{N \to \infty} \epsilon_{N,K}= \target(\msk)$.
\begin{proof}
Let $(x, \msa)\in \Xset \times\Xsigma$. Then, using the lower bound \eqref{eq:P_N_lower_bound}, we obtain
\begin{align}
\MKisir (x, \msa)
    &\geq \sum_{i=2}^N\int\target(\rmd x^i)\indi{\msa}(x^i)\int \frac{\proposal(\weightfunc)}{\weightfunc(x) +\weightfunc(x^i) +\sum_{j=2, j\neq i}^N \weightfunc(x^j)} \prod_{j=2, j\neq i}^N\proposal(\rmd x^j) \\
    &\geq (N-1) \int\target(\rmd y)\indi{\msa}(y) \frac{1}{\weightfunc(x)/\proposal(\weightfunc) +\weightfunc(y)/\proposal(\weightfunc) +N-2} \eqsp,
    \end{align}
    where the last inequality follows from Jensen's inequality and the convexity of the function $z\mapsto (z+a)^{-1}$ on $\rset_+$.
    We conclude by noting that
    \begin{align}
        P_N(x, \msa)
        &\geq (N-1)  \int\target(\rmd y)\indi{\msa\cap \msk}(y) \frac{1}{\weightfunc(x)/\proposal(\weightfunc) +\weightfunc(y)/\proposal(\weightfunc) +N-2}\\
        &\geq  \frac{N-1}{2w_{\infty, \msk} +N-2}\int\target(\rmd y)\indi{\msa\cap \msk}(y)= \frac{(N-1)\target(\msk)}{2w_{\infty, \msk} +N-2}\target_\msk(\msa)\eqsp.
    \end{align}
  \end{proof}

\begin{lemma}
\label{lem:drift-condition-iSIR-kernel}
Assume \Cref{assum:rejuvenation-kernel}. Then for all $x \in \Xset$, any function $V: \Xset \to \coint{1,\infty}$ with $\target(V) < \infty$, $\proposal(V) < \infty$, and $N \geq 3$, it holds that
\begin{equation}
\label{eq:simple_drift_P_N}
\MKisir V(x) \leq V(x) + \bconst[\MKisir]\eqsp,
\end{equation}
where $\bconst[\MKisir]$ is given in \eqref{eq:bconst-MKisir}.
\end{lemma}
Note that
\begin{equation}
\label{eq:MKisir-infty}
\bconst[{\MKisir[\infty]}]:= \lim_{N \to \infty} \bconst[\MKisir] = 2 \target(V) + 4 \PVar_{\proposal}[\weightfunc]/\proposal(V)\eqsp.
\end{equation}
\begin{proof}
Note first that
  \begin{align}
  \label{eq:weak_drift_isir}
    \MKisir V(x) &= V(x) \int\frac{\weightfunc(x)}{\weightfunc(x) + \sum_{j=2}^N \weightfunc(x^j)} \prod_{j=2}^N\proposal(\rmd x^j) + \int \sum_{i=2}^N \frac{\weightfunc(x^i)}{\weightfunc(x) + \sum_{j=2}^N \weightfunc(x^j)}V(x^i) \prod_{j=2}^N\proposal(\rmd x^j) \\
    &\leq V(x) + (N-1) U_N
 \end{align}
 where we have set
 \begin{equation}
 \label{eq:equivalent-expression}
 U_N=   \int \frac{\weightfunc(x^2) V(x^2) \proposal(\rmd x^2)}{\weightfunc(x^2) + \sum_{j=3}^N \weightfunc(x^j)} \prod_{j=3}^N\proposal(\rmd x^j)\eqsp.
 \end{equation}
 Since the function $z\mapsto z/(z+a)$ is concave on $\rset_+$ for $a>0$, we have
 \begin{align}
 \label{eq:equivalent-expression-1}
    &   \int \frac{\weightfunc(x^2) }{\weightfunc(x^2) + \sum_{j=3}^N \weightfunc(x^j)}  V(x^2) \proposal(\rmd x^2) = \proposal(V)  \int \frac{\weightfunc(x^2)}{\weightfunc(x^2) + \sum_{j=3}^N \weightfunc(x^j)} \frac{V(x^2) \proposal(\rmd x^2)}{\proposal(V)}\\
 \nonumber
    &\quad \leq  \proposal(V)  \frac{\int \weightfunc(x^2) V(x^2) \proposal(\rmd x^2) / \proposal(V) }{\int \weightfunc(x^2) V(x^2) \proposal(\rmd x^2)/\proposal(V) +  \sum_{j=3}^N \weightfunc(x^j)} \leq  \frac{\target(V) \proposal(\weightfunc)}{\target(V)\proposal(\weightfunc)/\proposal(V)  + \sum_{j=3}^N \weightfunc(x^j)}\eqsp.
\end{align}
The bound above implies that, with renormalization,
 \[
 U_N \leq  \int \frac{\target(V)}{\target(V)/\proposal(V) +  \sum_{j=3}^N \weightfunc(x^j)/\proposal(\weightfunc)}  \prod_{j=3}^N\proposal(\rmd x^j)
 \]
 Applying now \Cref{lemma:aux_isir} with $a =\target(V)/\proposal(V)$, $b = 1$, $M = N-2$, and $Y_j = \weightfunc(x^j)/\proposal(\weightfunc)$, we obtain that
\begin{align}
U_N \leq \frac{\target(V)}{\target(V)/\proposal(V) +  (N-2)/2} + \frac{4 \PVar_{\proposal}[\weightfunc]}{(N-2)\proposal(V)} \eqsp.
\end{align}
Combining the bounds above yields \eqref{eq:simple_drift_P_N} with
\begin{equation}
\label{eq:bconst-MKisir}
\bconst[\MKisir] = \frac{(N-1)\target(V)}{\target(V)/\proposal(V) +  (N-2)/2} + \frac{4 (N-1) \PVar_{\proposal}[\weightfunc]}{(N-2)\proposal(V)} \eqsp.
\end{equation}
\end{proof}

  \begin{lemma}
  \label{lemma:composition_small_set}
  Let $\MKP$ be a Markov kernel on $(\Xset, \Xsigma)$, $\gamma$ be a probability measure on $(\Xset, \Xsigma)$, and $\epsilon > 0$. Let  $\msc\in\Xsigma$ be an $(1, \epsilon\gamma)$-small set for $\MKP$. Then for arbitrary Markov kernel $\MKQ$ on $(\Xset, \Xsigma)$, the set $\msc$ is an $(1, \epsilon \gamma_{\MKQ})$-small set for $\MKP\MKQ$, where $\gamma_{\MKQ}(\msa) = \int \gamma(\rmd y)\MKQ(y,\msa)$ for $\msa \in \Xsigma$.
  \end{lemma}
  \begin{proof}
  Let $(x, \msa)\in\msc\times\Xsigma$.
  Then it holds
  \begin{equation}
      \MKP\MKQ(x, \msa) = \int \MKP(x, \rmd y) \MKQ(y, \msa) \geq \epsilon\int \gamma(\rmd y) \MKQ(y, \msa) = \epsilon \gamma_{\MKQ}(\msa)\,.
  \end{equation}
  \end{proof}

  \begin{lemma}
  \label{lemma:composing_two_kernels}
   Let $\MKP$ and $\MKQ$ be two irreducible Markov kernels with $\target$ as their unique invariant distribution. Let $V:\Xset\to \coint{1,\infty}$ be a measurable function. Suppose that there exist $ \lambda_Q\in \coint{0,1}$ and $\bconst[\MKP], \bconst[\MKQ] \in \rset_+$ such, that $\MKP V(x) \leq V(x) + \bconst[\MKP]$ and $\MKQ V(x) \leq \lambda_\MKQ V(x) + \bconst[\MKQ]$. Let $r_0 \geq 1$. Also assume that for all $r \geq r_0$,  there exist $\epsilon_{r} > 0$ and a probability measure $\gamma_{r}$ such that for all $(x, \msa)\in \msv_r \times\Xsigma$, $\MKP(x, \msa)\geq \epsilon_{r}\gamma_r(\msa)$, where $\msv_r= \{x \in \Xset: V(x) \leq r \}$. Define $\MKK = \MKP \MKQ$ and $\lambda_{\MKK} = \lambda_{\MKQ}$, $\bconst[\MKK] = \bconst[\MKP] + \bconst[\MKQ]$.
    Then,
    \[
    \text{$\MKK V(x) \leq \lambda_{\MKK} V(x) + \bconst[\MKK]$ and,  for all $x\in \msv_r$, $\MKK(x, \msa) \geq \epsilon_{r}\gamma_{\MKQ,r}(\msa)$},
    \]
    where $\gamma_{\MKQ,r}(\msa) = \int \gamma_{r}(\rmd y)\MKQ(y,\msa)$.
Moreover, let $r \geq r_0$ be such that $\lambda_{\MKK} + 2\bconst[\MKK]/(1+r) < 1$. Then, for any $x\in\Xset$ and $k\in\nset$,
$$
\Vnorm{\MKK^k(x, \cdot)-\target}\leq \cconst[\MKK] \{V(x) + \target(V)\} \rho_\MKK^k\eqsp,
$$
where
\begin{align}
    \label{eq:definition-rho-cconst}
      \rho_{\MKK} &= \frac{\log(1-\epsilon_{r}) \log\Bar{\lambda}_{\MKK}}{\log(1-\epsilon_{r})+ \log\Bar{\lambda}_{\MKK}-\log\Bar{b}_{\MKK}}\eqsp,
    \quad
    \cconst[\MKK] = (\lambda_{\MKK} + \bconst[\MKK])(1 + \Bar{b}_{\MKK}/[(1 - \epsilon_{r})(1-\Bar{\lambda}_{\MKK})]), \\
    \Bar{\lambda}_{\MKK} &= \lambda_{\MKK} + 2\bconst[\MKK]/(1+r)\eqsp, \quad \Bar{b}_{\MKK} = \lambda_{\MKK} r + \bconst[\MKK]\eqsp.
\end{align}
\end{lemma}
\begin{proof}
  By \Cref{lemma:composition_small_set}, it holds  that for any $(x, \msa)\in\msv_r \times \Xsigma$, $\MKK(x, \msa) \geq \epsilon_{r}\gamma_{\MKQ,r}(\msa)$.
  Moreover, for any $x \in \Xset$, $\MKK V(x) = \MKP \MKQ V(x) \leq \lambda_{\MKQ} \MKP V(x) + \bconst[\MKQ] \leq \lambda_{\MKQ} V(x) + \bconst[\MKQ] + \bconst[\MKP]$. The proof is completed with \cite[Theorem~19.4.1]{douc:moulines:priouret:2018}.
\end{proof}

\begin{proof}[Proof of \Cref{theo:main-geometric-ergodicity}]
The proof consists of the $3$ main steps:
\begin{enumerate}[leftmargin=*,nosep]
\item \Cref{lemma:no_assum_minoration} implies that for all $r \geq r_{\rejukernel}$, the level sets $\msv_r$ for the Markov kernel $\MKisir$ are $(1,\epsilon_{r,N}\gamma_{r})$-small for the Markov kernel $\MKisir$, where
\begin{equation}
  \label{eq:gamma_eps_def_th_4}
  \epsilon_{r,N} = (N-1)\target(\msv_r)/[2\weightfunc_{\infty,r} + N-2],
\end{equation}
and $\gamma_r(\msa) = \int \target_{\msv_r}(\rmd y)\rejukernel(y,\msa)$ with $\target_{\msv_r}(B) = \target(B \cap \msv_r) / \target(\msv_r)$, for any $B \in \Xsigma$.
\item \Cref{lem:drift-condition-iSIR-kernel} implies that for all $x \in \Xset$,
$\MKisir V(x) \leq V(x) + \bconst[\MKisir]$, where $\bconst[\MKisir]$ is given in \eqref{eq:bconst-MKisir}.
\item We finally show (see \Cref{lemma:composing_two_kernels}) that the Markov kernel $\XtryK$ also satisfies a Foster-Lyapunov condition with the same drift function $V$ as $\rejukernel$, that is, $\XtryK V \leq \lambda_{\rejukernel} V + \bconst[\XtryK]$ with $\bconst[\XtryK]= \bconst[\rejukernel] + \bconst[\MKisir]$.
\end{enumerate}
We conclude by using \Cref{lemma:composing_two_kernels}. We choose $r_{N} =
r_{\rejukernel} \vee \{ 4 \bconst[\XtryK]/(1-\lambda_{\rejukernel}) - 1 \}$.
Then $\lambda_{\rejukernel} + 2\bconst[\XtryK]/(1+r_{N}) \leq  (1 + \lambda_{\rejukernel})/2 < 1$, and \Cref{lemma:composing_two_kernels} implies
\eqref{eq:geometric-ergodicity-XtryK} with
 \begin{align}
   \label{eq:ex2mcmc_erg_const}
        &\log \rate[\XtryK] = \frac{\log(1-\epsilon_{r,N}) \log\Bar{\lambda}_{\XtryK}}{\log(1-\epsilon_{r,N})+ \log\Bar{\lambda}_{\XtryK}-\log\Bar{b}_{\XtryK}}\eqsp,\\
   \nonumber
        &\cconst[\XtryK] = (\lambda_\rejukernel + \Bar{b}_{\XtryK})(1 + \Bar{b}_{\XtryK}/[2(1- \epsilon_{r_{N},N})(1-\Bar{\lambda}_{\XtryK})])\eqsp, \\
    \nonumber
        &\Bar{\lambda}_{\XtryK} = (1 + \lambda_{\rejukernel})/2 \eqsp, \quad \Bar{b}_{\XtryK} = \lambda_{\rejukernel} r_{N} + \bconst[\XtryK]\eqsp.
\end{align}

Set $\bconst[{\XtryK[\infty]}] = \lim_{N \to \infty} \bconst[\XtryK]= \bconst[\rejukernel] + \bconst[{\MKisir[\infty]}]$,
where $\bconst[{\MKisir[\infty]}]$ is defined in \eqref{eq:MKisir-infty}, $r_{\infty} = r_{\rejukernel} \vee [4 \bconst[{\XtryK[\infty]}]/(1-\lambda_{\rejukernel}) - 1]$ and $\epsilon_{\infty}= \target(\msv_{r_{\infty}})$.
With these notations, we have
\begin{equation}
\label{eq:supplement_constants_ex2}
\begin{aligned}
        &\log \rate[{\XtryK[\infty]}] = \frac{\log(1-\epsilon_{\infty}) \log\Bar{\lambda}_{{\XtryK[\infty]}}}{\log(1-\epsilon_{\infty})+ \log\Bar{\lambda}_{{\XtryK[\infty]}}-\log\Bar{b}_{{\XtryK[\infty]}}}\eqsp,\\
        &\cconst[{\XtryK[\infty]}] = (\lambda_\rejukernel + \Bar{b}_{{\XtryK[\infty]}})(1 + \Bar{b}_{{\XtryK[\infty]}}/[(1- \epsilon_{\infty})(1-\Bar{\lambda}_{{\XtryK[\infty]}})]) \\
        &\Bar{\lambda}_{{\XtryK[\infty]}} = (1 + \lambda_{\rejukernel})/2 \eqsp,
        \eqsp \Bar{b}_{{\XtryK[\infty]}} = \lambda_{\rejukernel} r_{\infty} + \bconst[{\XtryK[\infty]}]
        \eqsp.
\end{aligned}
\end{equation}
\end{proof}

\section{Metropolis-Adjusted Langevin rejunevation kernel}
\label{supp:sec:MALA}
This section addresses the convergence of the Metropolis Adjusted Langevin algorithm (MALA) for sampling from a positive target probability density $\pdftarget$ on $(\rset^d,\mcbb(\rset^d))$, where $\mcbb(\rset^d)$ is the Borel $\sigma$ field of $\rset^d$ endowed with the Euclidean topology.
For simplicity, let $U = -\log \pdftarget$ be the associated potential function.
MALA is a Markov chain Monte Carlo (MCMC) method based on Langevin diffusion associated with $\pdftarget$:
\begin{equation} \label{eq:def_langevin}
\rmd \bfX_t = -\nabla U(\bfX_t) \rmd t + \sqrt{2} \rmd \bfB_t \eqsp,
\end{equation}
where $\sequenceg{\bfB}[t][0]$ is a $d$-dimensional Brownian motion. It is known that under mild conditions this diffusion admits a strong solution $\sequenceg{\bfX^{(x)}}[t][0]$ for any starting point $x\in \rset^d$ and defines a Markov semigroup $\sequenceg{\Psemigroup}[t][0]$ for any $t \geq 0$, $x \in \rset^d$ and $\msa \in \mcbb(\rset^d)$ by $\Psemigroup_t(x,\msa) = \ PP (\bfX_t^{(x)} \in \msa)$. Moreover, this Markov semigroup admits $\pdftarget$ as its unique stationary measure, is ergodic and even $V$-uniformly geometrically ergodic with additional assumptions on $U$ (see \cite{roberts:tweedie-Langevin:1996,mattingly2002Ergodicity}). However, sampling a path solution of \eqref{eq:def_langevin} is a real challenge in most cases, and discretizations are used instead to obtain a Markov chain with similar long-term behaviour. Here we consider the
Euler-Maruyama discretization, which is given by
\eqref{eq:def_langevin}, defined for all $k \geq 0$ by
\begin{equation} \label{eq:def_ula} \YrULA_{k+1} = \YrULA_k - \gaStep \nabla U(\YrULA_k) + \sqrt{2\gaStep} \ZrGaussian_{k+1} \eqsp,
\end{equation}
where $\gaStep$ is the step size of the discretization and.
$\sequence{\ZrGaussian}[k][\nsets]$ is a \iid~sequence of $d$-dimensional standard Gaussian random variables. This algorithm was proposed by \cite{ermak:1975,parisi:1981} and later studied by \cite{grenander:1983,grenander:miller:1994,neal:1992,roberts:tweedie-Langevin:1996}. According to
\cite{roberts:tweedie-Langevin:1996}, this algorithm is called the
Unadjusted Langevin algorithm (ULA). A drawback of this method is that even if the Markov chain $\sequence{\YrULA}[k][\nset]$ has a unique stationary distribution $\pdftarget_{\gaStep}$ and is ergodic (which is guaranteed under mild assumptions about $U$), $\pdftarget_{\gaStep}$ is different from $\pdftarget$ most of the time. To solve this problem, in
\cite{rossky:doll:friedman:1978,roberts:tweedie-Langevin:1996} it is proposed to use the Markov kernel associated with the recursion defined by the Euler-Maruyama discretization \eqref{eq:def_ula} as a proposal kernel in a Metropolis-Hastings algorithm that defines a new Markov chain $\sequence{\YrMALA}[k][\nset]$ by:
\begin{equation}
\label{eq:def_MALA}
\YrMALA_{k+1}= \YrMALA_{k} + \1_{\rset_+}(U_{k+1}-\alpha_{\gamma}( \YrMALA_{k}, \tilde{\YrULA}_{k+1}))\{\tilde{\YrULA}_{k+1}-\YrMALA_k\} \eqsp,
\end{equation}
where $\tilde{\YrULA}_{k+1} = \YrMALA_k-\gamma \nabla U(\YrMALA_k) + \sqrt{2\gamma} \ZrGaussian_{k+1}$, $\sequence{U}[k][\nset^*]$ is a sequence of \iid~uniform random variables on $\ccint{0,1}$ and $\alpha_{\gamma} : \rset^{2d} \to \ccint{0,1}$ is the usual Metropolis acceptance ratio. This algorithm is called Metropolis Adjusted Langevin Algorithm (MALA) and has since been used in many applications.

Denote by $\rmala_{\gamma}$ the proposal transition density associated  to the Euler-Maruyama discretization \eqref{eq:def_ula} with stepsize $\gamma >0$, \ie, for any $x,y \in \rset^d$,
\begin{equation}
  \label{eq:def_r_gamma_ULA}
  \rmala_{\gamma}(x,y) = (4\uppi \gamma)^{-d/2} \exp \left(-(4 \gamma)^{-1}\norm[2]{y-x+ \gamma \nabla U(x)} \right) \eqsp.
\end{equation}
Then,  the Markov kernel $\Rmala_{\gamma}$ of the MALA algorithm \eqref{eq:def_MALA} is given for $\gaStep>0$, $x\in\rset^d$, and $\msa \in \mcbb(\rset^d)$ by
\begin{align}
\label{eq:def-kernel-MALA}
  \Rmala_{\gamma}(x ,\msa)& = \int_{\rset^d}  \1_{\msa}(y) \alphaMALA(x,y)\rmala_{\gamma}(x,y) \rmd y + \updelta_{x}(\msa) \int_{\rset^d} \{1 - \alphaMALA(x,y)\}\rmala_{\gamma}(x,y) \rmd y  \eqsp,\\
  \label{eq:def-alpha-MALA_0}
 \alphaMALA(x,y) &= 1\wedge \frac{\pi(y)r_{\gamma}(y,x)}{\pi(x)r_{\gamma}(x,y)} \eqsp.
\end{align}
It is well-known, see \eg~\cite{roberts:tweedie-Langevin:1996}, that
for any $\gamma >0$, $\Rmala_{\gamma}$ is reversible with respect to $\pdftarget$ and $\pdftarget$-irreducible.
\begin{assumptionH}
\label{ass:regularity-U}
The function $U : \rset^d \to \rset$ is three times continuously differentiable. In addition, $\nabla U(0) = 0$ and there exists $\Ltt \geq 0$ and $\Mtt \geq 0$ such that $\sup_{x \in \rset^d}\norm{\rmD^2 U(x)} \leq \Ltt$  such that  $\sup_{x \in \rset^d}\norm{\rmD^3 U(x)} \leq \Mtt$.
\end{assumptionH}
The condition $\nabla U(0) = 0$ is satisfied (up to a translation) as soon as $U$ has a local minimum, which is the case  when $\lim_{\norm{x} \to +\infty} U(x) = +\infty$, since $U$ is continuous.
\begin{assumptionH}
  \label{ass:curvature_U}
  There exist $\mtt >0$ and $\Ktt \geq 0$ such that for any $x,y \in \rset^d$, $\norm{x} \geq \Ktt$ and $\norm{y} = 1$,
  \begin{equation}
    \label{eq:7}
\rmD^2U(x) \{y\}^{\otimes 2}  \geq  \mtt  \eqsp.
  \end{equation}
\end{assumptionH}
Note that under \Cref{ass:regularity-U} and \Cref{ass:curvature_U}, for any $x,y \in \rset^d$, $\norm{y} = 1$, it holds that
  \begin{equation}
    \label{eq:7_2}
\rmD^2U(x) \{y\}^{\otimes 2}  \geq  \mtt -(\mtt + \Ltt) \1_{\ball{0}{\Ktt}}(x) \eqsp.
  \end{equation}
In the case $\Ktt = 0$, \Cref{ass:curvature_U} amounts to $U$ being strongly convex and the convexity constant being equal to $\mtt$. However, if $\Ktt > 0$, \Cref{ass:curvature_U} is a slight strengthening of the condition of strong convexity at infinity considered in \cite{chen1997estimation,eberle:2015}: there is $\mtt' > 0$ and $\Ktt' \geq 0$ such that for each $x,y \in\rset^d$, $\norm{x-y} \geq \Ktt'$
\begin{equation} \label{eq:strong_convex_infi}
\text{ $\ps{\nabla U(x) - \nabla U(y)}{x-y} \geq \mtt' \norm{x-y}^2$ } \eqsp.
\end{equation}
Indeed, if \eqref{eq:strong_convex_infi} holds for any $x,y\in\rset^d$ that $\norm{x} \vee \norm{y} \geq \Ktt'$ instead of $\norm{x-y} \geq \Ktt'$, then a simple calculation implies that \Cref{ass:curvature_U} holds with $\mtt \leftarrow \mtt'$ and $\Ktt \leftarrow \Ktt'+1$.
Finally, while the condition \eqref{eq:strong_convex_infi} holds for $x,y \in\rset^d$, $\norm{x-y} \geq \Ktt'$, is weaker than \Cref{ass:curvature_U}, it may be more convenient in many situations to check whether the latter holds.
\begin{lemma}
\label{lem:quadratic_behaviour}
Assume \Cref{ass:regularity-U} and \Cref{ass:curvature_U} hold. The function $U$
satisfies for any $x \in \rset^d$,
\[
\ps{\nabla U(x)}{x} \geq  (\mtt/2) \norm[2]{x} -\tCtt \1_{\ball{0}{\tKtt}}(x)\eqsp,
\] with  $\tKtt = 2 \Ktt(1+\Ltt/\mtt)$ and  $\tCtt = \Ltt \tKtt^2$.
\end{lemma}

Note that under \Cref{ass:regularity-U} and \Cref{ass:curvature_U}, $\mtt \leq \Ltt$.
Define for any $\eta >0$, $V_{\eta}: \rset^d \to \coint{1,+\infty}$ for any $x \in \rset^d$ by
\begin{equation}
  \label{eq:def_V_eta}
  \Veta_{\eta}(x) = \exp(\eta \norm[2]{x}) \eqsp.
\end{equation}
The analysis of MALA is naturally  related to the study of the ULA algorithm. More precisely, since for any $x \in\rset^d$ and $\msa \in \mcbb(\rset^d)$, the Markov kernel corresponding to ULA \eqref{eq:def_ula} is given by
\begin{equation}
  \label{eq:def_kernel_ula}
  Q_{\gamma}(x,\msa) = \int_{\rset^d}  \1_{\msa}(x- \gamma \nabla U(x) + \sqrt{2\gamma} z)   \densgauss(z) \rmd z.
\end{equation}
To show that MALA satisfies a Lyapunov condition, we first state a drift condition for the ULA algorithm.
\begin{proposition}
\label{propo:super_lyap_ula}
  Assume \Cref{ass:regularity-U} and \Cref{ass:curvature_U} and let $\bgamma \in \ocint{0,\mtt/(4\Ltt^2)}$. Then, for any $\gamma \in \ocint{0,\bgamma}$, $x \in\rset^d$,
  \begin{equation}
    Q_{\gamma} V_{\bareta}(x) \leq \exp(-\bareta \mtt \gamma \norm[2]{x}/4) V_{\bareta}(x) + \bdriftula \gamma \1_{\ball{0}{\rayula}}(x) \eqsp,
  \end{equation}
  where $V_{\bareta}$ is defined in \eqref{eq:def_V_eta}, $\bareta  = \mtt/16$, $\rayula = \max(\tKtt,4\sqrt{d/\mtt})$, $\tKtt$ is defined in \Cref{lem:quadratic_behaviour} and
  \begin{equation}
  \label{eq:coeffs_super_lyap_mala}
  \begin{aligned}
    \bdriftula &= \parentheseDeux{\bareta \lbrace\mtt/4+   (1+16\bareta\bgamma)(4\bareta + 2 \Ltt + \bgamma \Ltt^2)\rbrace (\rayula)^2 +4 \bareta d  } \\
    &  \qquad \times \exp(\bgamma\bareta\lbrace\mtt/4+   (1+16\bareta\bgamma)(4\bareta + 2 \Ltt + \bgamma \Ltt^2)\rbrace (\rayula)^2 + 4\bareta\bgamma d)\eqsp.
  \end{aligned}
\end{equation}
\end{proposition}
\begin{proof}
The proof follows from \cite[Proposition~6]{moulines:durmus:2022}.
\end{proof}
 We now introduce for $\bgamma > 0$ the auxiliary constant
\begin{equation}
\label{eq:C_1_gamma_const}
C_{1,\bgamma} = 2 (2^{1/2}\Mtt \vee \bgamma^{1/2} \Mtt \Ltt \vee 2 \Ltt^2 [1\vee\bgamma^{1/2}\vee \bgamma \Ltt \vee (\bgamma \Ltt^{4/3})^{3/2}])\eqsp.
\end{equation}
For $\bgamma \in \ocint{0, \mtt^3/(4\Ltt^4)}$, we also define $C_{2,\bgamma}$ as
\begin{equation}
\label{eq:C_2_gamma_const}
C_{2,\bgamma} = 2 \Ltt + (\bgamma/2)\Ltt^2 + 2^{-3/2}  \bgamma^{3/2} \Ltt^3 + \{2^{1/2} \Ltt^2 + (2^{1/2} \Ltt^2 + 2^{-3/2} \bgamma^{1/2}) \Ltt^3\}^{2} (2^{4}/\mtt^3)\eqsp.
\end{equation}
Using \Cref{propo:super_lyap_ula}, we state a drift condition for the MALA kernel $R_{\gamma}$.

\begin{proposition}
\label{propo:lyap_mala_total}
  Assume \Cref{ass:regularity-U} and \Cref{ass:curvature_U}. Then, there exist $\Gamma > 0$ (given in \eqref{eq:def_const_drift_1}) such that for any $\bgamma \in \ocint{0,\Gamma}$, $\gamma \in\ocint{0,\bgamma}$ and $x \in\rset^d$,
  \begin{equation}
     \Rmala_{\gamma}V_{\bareta}(x) \leq (1-\tildem \gamma)V_{\bareta}(x)+ \bdriftmala \gamma \1_{\ball{0}{\raymaladrift}}(x) \eqsp,
  \end{equation}
  where $V_{\bareta}$ is defined by \eqref{eq:def_V_eta}, $\Rmala_{\gamma}$ is the Markov kernel of MALA defined by \eqref{eq:def-kernel-MALA},  $\bareta=\mtt/16$, $\tildem  = \bareta \mtt (\raymaladrift)^2/16 $, and
\begin{align}
\label{eq:def_const_drift_1}
\Gamma_{1/2} &= \min\left(1, \mtt^3/(4\Ltt^4), d^{-1}\right) \eqsp, \quad \Gamma = \min\left(\Gamma_{1/2}, 4/\{\mtt \bareta (\raymaladrift)^2\}\right) \eqsp, \\
\raymaladrift  &= \max(2^4,2\Ktt,\rayula,\tKtt,4b_{1/2}^{1/2}/(\mtt \bareta)^{1/2})  \eqsp, \quad     b_{1/2} = C_{2,\Gamma_{1/2}} d + \sup_{u \geq 1} \{u \rme^{-u/2^7}\} \eqsp, \\
    \label{eq:def_b_drift}
    \bdriftmala & = \bdriftula +  \bareta \mtt (\raymaladrift)^2 \rme^{\bareta (\raymaladrift)^2} / 16  + C_{1,\bgamma}  \bgamma^{1/2} \left\{d + \sqrt{3}d^2 + (\raymaladrift)^2\right\}\eqsp,
\end{align}
where $\rayula,\bdriftula$ are defined in \Cref{propo:super_lyap_ula}, and $\tKtt$ is defined in \Cref{lem:quadratic_behaviour}.
\end{proposition}
\begin{proof}
The proof follows from \cite[Proposition~7]{moulines:durmus:2022}.
\end{proof}

Quantitative bound on the mixing rate of the MALA sampler requires also the \emph{minorization condition} for the MALA kernel. The result below is due to \cite[Proposition~12]{moulines:durmus:2022}.
\begin{proposition}
\label{propo:small_set_mala}
Assume \Cref{ass:regularity-U} and \Cref{ass:curvature_U}. Then for any $\ray \geq 0$ there exists $\tGamma_{\ray} > 0$ (given in \eqref{eq:small_set_const_mala_def} below), such that for any $x,y \in \rset^d$, $\norm{x}\vee \norm{y} \leq \ray$, and $\gamma \in \ocintLigne{0,\tGamma_{\ray}}$ we have
\begin{equation}
\label{eq:small_set_mala_propo}
\tvnorm{\updelta_x \Rmala_\gamma^{\ceil{1/\gamma}} - \updelta_y \Rmala_\gamma^{\ceil{1/\gamma}}}   \leq 2 (1-\varepsilonula(\ray)/2) \eqsp,
\end{equation}
where
\begin{align}
\label{eq:small_set_const_mala_def}
\varepsilonula(\ray) &= 2\Phibf\parenthese{-\sqrt{3} (\Ltt+1)^{1/2}\ray}, \quad \tGamma_{1/2} = \mtt/(4\Ltt^2)\eqsp, \\
\tGamma_{\ray} &= \tGamma_{1/2} \wedge \left[\frac{\varepsilon(\ray)}{2 C_{1,\tGamma_{1/2}}(d+\sqrt{3} d^2+\ray^2+2\tbdriftulatGamma/\mtt)}\right]^2\eqsp, \\
\tbdriftulatGamma &=  2d + [\max(\tKtt, 2\sqrt{(2d)/\mtt})]^2 \parenthese{\tGamma_{1/2} \Ltt^2 + 2\Ltt + \mtt/2}\eqsp,
\end{align}
where $C_{1,\tGamma_{1/2}}$ is defined in \eqref{eq:C_1_gamma_const}, $\tKtt$ is defined in \Cref{lem:quadratic_behaviour}, and $\Phibf(\cdot)$ is the cumulative distribution function of the Gaussian distribution with zero mean an unit variance on $\rset$.
\end{proposition}
{It is interesting to note that $\gamma$ is the discretization step of the underlying Langevin diffusion. We have to iterate the kernel $1/\gamma$ times for the diffusion to progress by one time unit.}
Combining \Cref{propo:lyap_mala_total} and \Cref{propo:small_set_mala} yields the following ergodicity result in $\Veta_{\bareta}$-norm.

\begin{theorem}
\label{theo:V-geom_ergo_MALA}
Assume \Cref{ass:regularity-U} and \Cref{ass:curvature_U}. Then, there exist $\bGamma > 0$ (defined in \eqref{eq:cst_bornes_mala} below), such that for any $\gamma \in\ocint{0,\bGamma}$, there exist $C_{\bGamma} \geq 0$ and $\rho_{\bGamma} \in \coint{0,1}$ (given in \eqref{eq:cst_bornes_mala}) satisfying for any $x \in \rset^d$,
  \begin{equation}
    \label{eq:V-geom_ergo_MALA}
    \Vnorm[\Veta_{\bareta}]{\updelta_x \Rmala^{k}_{\gamma} - \pi} \leq C_{\bGamma} \rho_{\bGamma}^{\gamma k} \{ \Veta_{\bareta}(x) + \pi(\Veta_{\bareta})\} \eqsp,
  \end{equation}
  where $    \bareta = \mtt/ 16$,
\begin{equation}
    \label{eq:cst_bornes_mala}
      \begin{aligned}
        &\log \rho_{\bGamma}  = \frac{\log(1-2^{-1}\varepsilon(\ray_{\bGamma})) \log\bar\lambda} { \log(1-2^{-1}\varepsilon(\ray_{\bGamma})) +
          \log\bar\lambda-\log {\bar{b}^{\operatorname{M}}_{\bGamma}}} \eqsp ,\\
        &\bar\lambda = (1 + \lambda)/2 \eqsp, \quad \lambda = \rme^{-\tildem}\eqsp, \quad \bar{b}^{\operatorname{M}}_{\bGamma} = \lambda b^{\operatorname{M}}_{\bGamma} + M_{\bGamma}
        \eqsp, \quad \bGamma = \Gamma \wedge \tGamma_{\ray_{\Gamma}}\eqsp, \\
        & M_{\bgamma}  = \parenthese{\frac{4\bdriftmala (1+\bgamma)}{1-\lambda}} \vee 1 \eqsp, \quad \ray_{\bgamma} = (\log(M_{\bgamma})/\bareta)^{1/2}\eqsp, \quad \bgamma \in \{\bGamma,\Gamma\}\eqsp, \\
        & C_{\bGamma}  = \rho_{\bGamma}^{-1}\{\lambda+1\}\{1+
        \bar{b}^{\operatorname{M}}_{\bGamma}/[1-2^{-1}\varepsilon(\ray_{\bGamma})(1-\bar\lambda)]\}\eqsp,
      \end{aligned}
\end{equation}
and $\tildem$ is given in \Cref{propo:lyap_mala_total}.
\end{theorem}
\begin{proof}
The proof follows from \cite[Theorem~2]{moulines:durmus:2022}. For completeness we repeat here the main steps of the proof. \Cref{propo:lyap_mala_total} shows that there exist $\Gamma > 0$ (given in \eqref{eq:def_const_drift_1}) such that for any $\bgamma \in \ocint{0,\Gamma}$, $\gamma \in\ocint{0,\bgamma}$ and $x \in\rset^d$,
\begin{equation}
     R_{\gamma}V_{\bareta}(x) \leq (1-\tildem \gamma)V_{\bareta}(x)+ \bdriftmala \gamma  \eqsp,
\end{equation}
where the constants $\tildem$ and $\bdriftmala$ are given in \Cref{propo:lyap_mala_total}. Hence, setting $\lambda = \rme^{-\tildem} < 1$, we obtain by induction that
\begin{equation}
\label{eq:n_step_drift_mala}
R_{\gamma}^{\ceil{1/\gamma}} V_{\bareta}(x) \leq \lambda V_{\bareta}(x) +  \bdriftmala \eqsp.
\end{equation}
Now we set $M_{\bgamma}$ and $\ray_{\bgamma}$ as in \eqref{eq:cst_bornes_mala}. Then \Cref{propo:small_set_mala} implies that for any $\bgamma \in \ocint{0,\tGamma_{\ray_{\Gamma}}}$, any $x,y\in \{V_{\bareta}(\cdot) \leq M_{\bgamma}\}$, and $\gamma \in \ocint{0,\bgamma}$,
\begin{equation}
\label{eq:n_step_minor_condition}
\tvnorm{\updelta_{x} R_{\gamma}^{\ceil{1/\gamma}} - \updelta_{y} R_{\gamma}^{\ceil{1/\gamma}}} \leq 2(1-\varepsilon(\ray_{\bgamma}))\eqsp.
\end{equation}
Now it remains to combine both statements with $\bgamma = \Gamma \wedge \tGamma_{\ray_{\Gamma}}$ and apply \cite[Theorem~19.4.1]{douc:moulines:priouret:2018} to the Markov kernel $R_{\gamma}^{\ceil{1/\gamma}}$.
\end{proof}

\paragraph{Comparison with \XTryM\ kernel. } Based on the results above, we first state the quantitative mixing rate bounds for \XTryM\ algorithm with the MALA kernel {$\Rmala^{\ceil{1/\gamma}}_{\gamma}$ (iterated $\ceil{1/\gamma}$ times)} applied as rejuvenation kernel. The corresponding Markov kernel writes for $x \in \rset^{d}$ and $\msa \in \mcbb(\rset^d)$ as
\begin{equation}
\label{eq:xtrym_mala_kernel}
\XtryK[N,\gamma](x,\msa)= \MKisir \Rmala^{\ceil{1/\gamma}}_{\gamma}(x,\msa)= \int \MKisir(x, \rmd y) \Rmala^{\ceil{1/\gamma}}_{\gamma}(y,\msa)\eqsp,
\end{equation}
where $\Rmala_{\gamma}(x,\msa)$ is defined in \eqref{eq:def-kernel-MALA}. Note also that, for $r \geq 1$, and $V_{\bareta}$ defined in \eqref{eq:def_V_eta}, the level sets
\begin{equation}
\label{eq:level_sets_V_func_MALA}
\msv_{\bareta,r} = \{x\colon V_{\bareta}(x) \leq r \} = \{x\colon \norm{x} \leq \sqrt{\log{r} / \bareta}\}\eqsp.
\end{equation}
The result above allows to state the following ergodicity result for $\XtryK[N,\gamma]$ kernel.
\begin{theorem}
\label{th:ergodicity_ex2_mala_rates}
Assume \Cref{ass:regularity-U}, \Cref{ass:curvature_U}, and \Cref{assum:rejuvenation-kernel},\Cref{assum:independent-proposal} with $V_{\bareta}$ defined in \eqref{eq:def_V_eta}. Then there exist $\bGamma$ (defined in \eqref{eq:cst_bornes_mala}), such that for any $\gamma \in \ocint{0,\bGamma}$, $x \in\rset^d$, and $k \in \nset$,
\begin{equation}
\label{eq:geometric-ergodicity-XtryK-MALA}
\Vnorm{\XtryK[N,\gamma]^{k}(x, \cdot) - \target} \leq \cconst[N] \{ \target(V_{\bareta}) + V_{\bareta}(x) \} \rate[N]^k\eqsp,
\end{equation}
where $V_{\bareta}$ is defined in \eqref{eq:def_V_eta}, and the constants $\cconst[N] $, $\rate[N] \in \coint{0,1}$ are given by
\begin{align}
\label{eq:ex2mcmc_erg_const_mala}
        &\log \rate[N] = \frac{\log(1-\epsilon_{r_{N},N}) \log\bar\lambda}{\log(1-\epsilon_{r_{N},N})+ \log\bar\lambda-\log\Bar{b}_{N}}\eqsp, \quad r_{N} =
1 \vee \{ 4 b_{N}/(1-\lambda) - 1\}\eqsp, \\
   \nonumber
        &\epsilon_{r_N,N} = (N-1)\target(\msv_{\bareta,r_N})/[2\weightfunc_{\infty,r_N} + N-2], \quad b_{N} = \bconst[\MKisir] +  \bar{b}^{\operatorname{M}}_{\bGamma}\eqsp, \\
    \nonumber
        & \cconst[N] = (\lambda + \Bar{b}_{N})(1 + \Bar{b}_{N}/[2(1- \epsilon_{r_{N},N})(1-\bar\lambda)]) \\
        & \bar\lambda = (1 + \lambda)/2\eqsp,\quad \Bar{b}_{N} = \lambda r_{N} + b_{N}\eqsp, \\
\end{align}
and $\lambda$ is defined in \eqref{eq:cst_bornes_mala}.
\end{theorem}
\begin{proof}
The proof follows from the combination of \Cref{theo:main-geometric-ergodicity} and \Cref{propo:lyap_mala_total}.
\end{proof}
{To derive the geometric ergodicity rates in \Cref{th:ergodicity_ex2_mala_rates}, it is not required to identify the small sets of the MALA rejuvenation kernel $\Rmala_{\gamma}$. The only quantity of interest is the Foster-Lyapunov drift condition satisfied by $\Rmala_{\gamma}^{\ceil{1/\gamma}}$}. 
\Cref{theo:V-geom_ergo_MALA} implies that the rate of convergence of MALA is $\gamma \log \rho_{\bGamma}$. The following statement allows to quantify the improvement in the convergence rate of $\XtryK[N,\gamma]$ compared to $\Rmala_{\gamma}^{\ceil{1/\gamma}}$. {Following \cite{paulin_concentration_spectral}, we consider the relative improvement of the \emph{mixing time} of the considered Markov kernels. To introduce formally the mixing time, we need an auxiliary definition of the $V$-Dobrushin coefficient. We refer the reader to \cite[Section~18.3]{douc:moulines:priouret:2018} for more detailed exposition. Recall that $M_{1,V}(\Xset)$ is a set of probability measures on $(\Xset,\Xsigma)$, such that $\xi(V) < \infty$.

\begin{definition}[$V$-Dobrushin coefficient] Let $V: \Xset \mapsto [1;+\infty)$ be a measurable function, and $\MKQ$ be a Markov kernel on $(\Xset,\Xsigma)$, such that $\xi(V) < \infty$ implies $\xi \MKQ (V) < \infty$ for any
measure $\xi \in M_{1,V}(\Xset)$. Then the $V$-Dobrushin coefficient of the Markov kernel $\MKQ$, is defined by 
\begin{equation}
\label{eq:V_dobru_coef}
\dobru[V]{\MKQ} = \sup_{\xi \neq \xi^{\prime} \in M_{1,V}(\Xset)} \frac{\Vnorm{\xi\MKQ - \xi^\prime \MKQ}}{\Vnorm{\xi - \xi'}}\eqsp.
\end{equation}
\end{definition}
It is known (see e.g. \cite[Theorem~18.4.1]{douc:moulines:priouret:2018}), that $V$-geometric ergodicity of the Markov kernel $\MKQ$ (see \Cref{def:geometric-ergodicity}) is equivalent to the fact, that 
\begin{equation}
\label{eq:_step_contc_V_norm}
\dobru[V]{\MKQ^{m}} \leq 1 - \varepsilon
\end{equation}
for some $m \in \nsets$ and $0 < \varepsilon < 1$.
\begin{definition}
\label{def:mix_time}
Let $\MKQ$ be $V$-geometrically ergodic Markov kernel. Then the corresponding mixing time $\taumix \in \nsets$ is defined as 
\begin{equation}
\label{eq:tau_mix_v_norm_def} \taumix = \inf_{m \in \nsets}\{m: \dobru[V]{\MKQ^{m}} \leq 1/4\}\eqsp.
\end{equation}
\end{definition}
Note that if $\MKQ$ is $V$-geometrically ergodic with factor $0 < \rho < 1$ given in $\Cref{def:geometric-ergodicity}$, its mixing time $\taumix$ is bounded as $\taumix \leq (\log(1/\rho))^{-1} \log(4M)$. 
\par 
Now we compare the mixing time of $\XtryK[N,\gamma]$, which is inversely proportional to $\log(1/\rate[N])$, to the mixing time of  $\Rmala_{\gamma}^{\ceil{1/\gamma}}$, which is inversely proportional to $\log(1/\rho_{\bGamma})$.}
\begin{theorem}
\label{th:mix_time_improvement}
Assume \Cref{ass:regularity-U}-\Cref{ass:curvature_U} and \Cref{assum:rejuvenation-kernel}-\Cref{assum:independent-proposal} with $V_{\bareta}$. Then there exist $\bGamma$ (defined in \eqref{eq:cst_bornes_mala}), such that for any $\gamma \in \ocint{0,\bGamma}$, it holds that
\begin{equation}
\label{eq:mix_rate_improvement}
\lim_{N \rightarrow \infty} \frac{\log(\rho_{\bGamma})}{\log(\rate[N])} =
\frac{\log(1-2^{-1}\varepsilon(\ray_{\bGamma})) }{\log(1-\epsilon_{\infty})} \times   
\frac{ \log(1-2^{-1}\varepsilon(\ray_{\bGamma})) + \log\bar\lambda-\log {\bar{b}^{\operatorname{M}}_{\bGamma}}}{\log(1-\epsilon_{\infty})+ \log\bar\lambda-\log\Bar{b}_{\infty}} 
\eqsp,
\end{equation}
where $\lambda, \bar{\lambda}$, and $\bar{b}^{\operatorname{M}}_{\bGamma}$ are defined in \eqref{eq:cst_bornes_mala}, $\varepsilonula(\cdot)$ is defined in \eqref{eq:small_set_const_mala_def}, and
\begin{align}
\label{eq:const_C_infty}
&r_{\infty} = 1 \vee \{ 4 b_{\infty}/(1-\lambda) - 1\}\eqsp, \quad \epsilon_{\infty} = \target(\msv_{\bareta, r_{\infty}})\eqsp, \quad b_{\infty} = \bconst[{\XtryK[\infty]}] + \bar{b}^{\operatorname{M}}_{\bGamma}\eqsp, \quad \Bar{b}_{\infty} = \lambda r_{\infty} + b_{\infty}\eqsp.
\end{align}
\end{theorem}
\begin{proof}
The proof follows by combining the expressions \eqref{eq:cst_bornes_mala} and \eqref{eq:ex2mcmc_erg_const_mala}.
\end{proof}
{The ratio $\log(1-2^{-1}\varepsilon(\ray_{\bGamma})) / \log(1-\epsilon_{\infty})$ is extremely small in most settings.
This explains the observed behavior: the mixing time of the Ex2MCMC kernel is much smaller than the mixing time of the MALA algorithm, which we observe in practice in all the examples we discuss. The difference is even more spectacular when the dimension increases. To illustrate this phenomenon, we consider the following numerical scenario for \eqref{eq:mix_rate_improvement}. We assume that \Cref{ass:regularity-U}-\Cref{ass:curvature_U} holds with $\mtt = 0.1, \Mtt = 2.0, \Ltt = 1.0$, and $\Ktt = 5.0$. One can evaluate that even for $d=2$ the respective value $\ray_{\bGamma} \approx 10^{3}$. We now show, how the bound for $\ray_{\bGamma}$ scales with the dimension $d$. The respective plot for $d \in [2;100]$ is given in \Cref{fig:K_gamma_scale}. It implies that $\ray_{\bGamma}$ grows as $\sqrt{d}$. At the same time, the standard bound $\Phibf\parenthese{-x} \leq \exp\{-x^2/2\}$, valid for $x \geq 0$, yields that $\varepsilonula(\ray_{\bGamma})/2 \leq \exp\{-(3/2)(L+1)\ray_{\bGamma}^2\}$. At the same time, $\epsilon_{\infty}$ typically does not decrease with the growth of $d$ due to the construction of $r_{\infty}$. Hence, the ratio \eqref{eq:mix_rate_improvement} decreases exponentially with the growth of $d$ in our model scenario.

\begin{figure}[h]
\centering
\begin{subfigure}{0.5\textwidth}
    \includegraphics[width=1\textwidth]{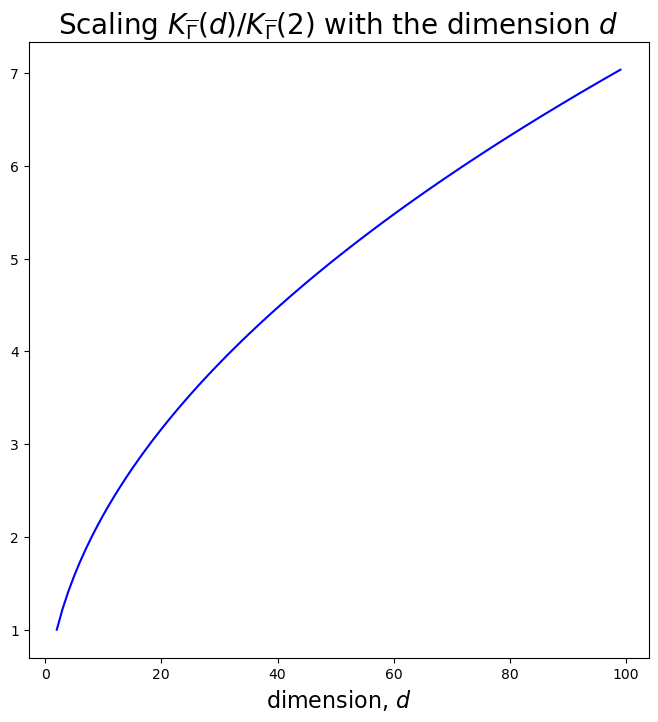} 
    \vspace{-2\baselineskip}
\end{subfigure}
\caption{Scaling of $\ray_{\bGamma}$ with dimension $d$, normalized by its value corresponding to $d=2$.}
\label{fig:K_gamma_scale}
\end{figure}
} 

\section{Proof of \Cref{thm:KL-simplified}}
\label{sec:thm:KL-simplified}
\input{proof_adapt}
\section{Numerical experiments}

\subsection{Metrics}
\label{subsec:supp:metrics}
\paragraph{ESTV} To compute Empirical sliced total variation distance (ESTV), we perform $25$ random one-dimensional projections and then perform Kernel Density Estimation there for reference and produced samples. We then take the TV-distance between two distributions over $1D$ grids of $1000$ points. We consider the value averaged over the projections to show the divergence between the MCMC distribution and the reference distribution.
\paragraph{EMD} We compute the EMD as the transport cost between sample and reference points in $L_2$ using the algorithm proposed in \cite{bonneel2011displacement}. Then we report the EMD rescaled by the target dimension $d$.

\paragraph{ESS} ESS (effective sample size) measures how many independent samples from target yield (approximately) the same variance for estimating the mean of some function. The closer ESS is to $1$, the better is the sampler. Following \cite{girolami:2011}, we compute ESS component-wise for multivariate distributions. Namely, given a sample $\{Y_t\}_{t=1}^M, Y_t \in \rset^d$ of size $M$, for $i = 1,\dots,d$, we compute
\[
\text{ESS}_{i} = \frac{1}{1 + \sum_{k=1}^M\rho_k^{(i)}}\eqsp.
\]
Here $\rho_k^{(i)} = \frac{\text{Cov}(Y_{t,i}, Y_{t+k,i})}{\text{Var}(Y_{t,i})}$ is the autocorrelation at lag $k$ for $i-$th component. We replace $\rho_k^{(i)}$ by its sample counterpart $\widehat\rho_k^{(i)}$, an report $\text{ESS} = d^{-1}\sum_{i=1}^{d} \widehat{\text{ESS}}_{i}$, where
\begin{equation}
\widehat{\text{ESS}}_{i} =  \frac{1}{1 + \sum_{k=1}^M \widehat{\rho}_k^{(i)}}\eqsp.
\end{equation}

{
\subsection{Unimodal Gaussian target and impact of dimension}
\label{app:subsec:single-gaussian}

With the simple experiment presented on \Cref{fig:gaussian_sampling}, we illustrate the sensitivity of the purely global $\isir$ to the match between the proposal and target, which typically worsens with dimension. Namely, the rate $\driftconstisir$ can be close to $1$ when the dimension $d$ is large, even when the restrictive condition that weights are uniformly bounded $\supnorm{\weightfunc} < \infty$ is satisfied. 

To illustrate this phenomenon, we consider a simple problem of sampling from the standard normal distribution $\mathcal{N}(0, \Id_d)$ with the proposal $\mathcal{N}(0, 2\Id_d)$ in increasing dimensions $d$ up to $300$. Results visualized in \Cref{fig:gaussian_sampling} show that the performance of  vanilla \isir\ quickly deteriorates as most proposals get rejected. This problem can be tackled by using the Explore-Exploit strategy coupling \isir\ with local MCMC steps to define a new sampler. This simple experiment previously considers \XTryM\ with MALA applied as $\rejukernel$.
}

\begin{figure}
\centering
\begin{subfigure}{0.95\textwidth}
\includegraphics[width=1.\textwidth]{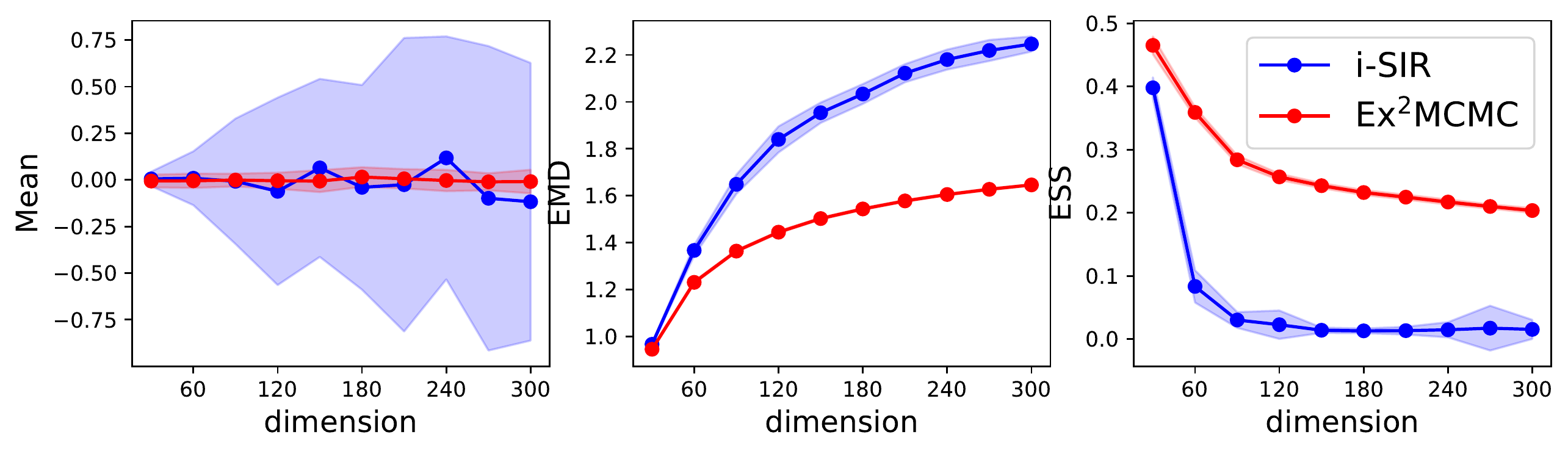}
\end{subfigure}
\caption{Sampling from $\mathcal{N}(0, \Id_d)$ with the proposal $\mathcal{N}(0, 2\Id_d)$. -- See \Cref{subsec:supp:metrics} for the definitions of EMD and ESS metrics. We display confidence intervals for \isir\ and \XTryM\ obtained from $100$ independent runs as blue and red regions, respectively. \XTryM\ helps to achieve efficient sampling even in high dimensions. }
\label{fig:gaussian_sampling}
\vspace{-4mm}
\end{figure}

\subsection{Mixtures of Gaussians}
\label{supp:subsec:mixtures}
\paragraph{Equally weighted Gaussians in two dimension}
\begin{figure}
\centering
\begin{subfigure}{0.495\textwidth}
    \centering
  \includegraphics[width=0.6\textwidth]{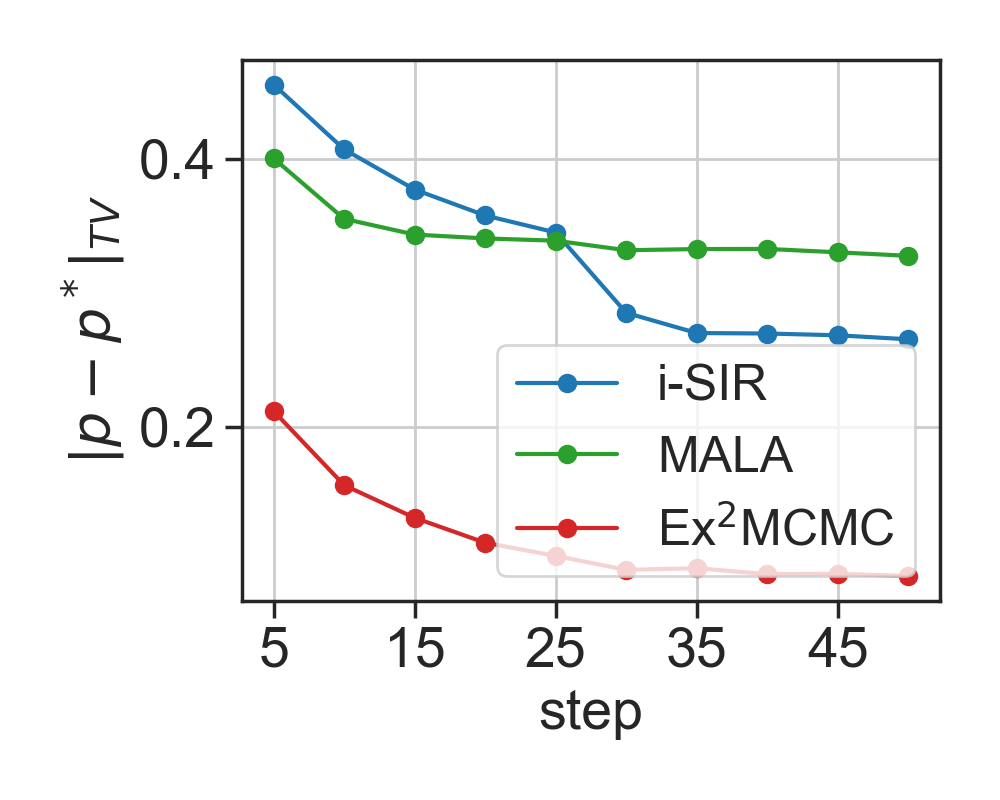}
    \vspace{-3mm}
    \caption{}
    \label{fig:mog_burnin_tv}
\end{subfigure}
\begin{subfigure}{0.495\textwidth}  
    \centering
    \includegraphics[width=0.6\textwidth]{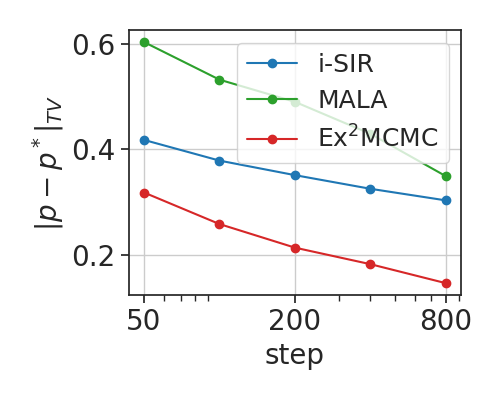}
    \vspace{-3mm}
    \caption{}
    \label{fig:mog_single_chain_tv}
\end{subfigure}
\caption{Inhomogeneous 2d Gaussian mixture. -- Quantitative analysis during burn-in of parallel chains (a, $M=500$ chains KDE) and for after burn-in for single chains statistics (b, $M=100$ average).}
\end{figure}
The target density is
\begin{equation}
\label{eq:gaus_mixture}
p_{\beta}(x) \propto \sum_{i=1}^{3}\beta_{i}\exp\bigl\{-\norm{x-\mu_i}^{2}/(2\sigma^2)\bigr\}\eqsp.
\end{equation}
Here we choose $\sigma = 1$, $\beta_i = 1/3$, and $\mu_i$, $i \in \{1,2,3,\}$ as vertices of an equilateral triangle with side length $4\sqrt{3}$ and center $(0,0)$. The contour representation of \eqref{eq:gaus_mixture} can be found in \Cref{fig:gaus_example_2d}. We compare $3$ sampling strategies:
\begin{itemize}[leftmargin=*,nosep]
\item \isir\ algorithm with $N=3$ particles and $\mathcal{N}(0,4\Id)$ proposal distribution;
\item MALA with step size $\gamma = 0.5$, tuned to obtain acceptance rate $~0.67$;
\item \XTryM\ algorithm with the same parameters as \isir\ and $3$ consecutive MALA steps with $\gamma = 0.5$ as rejuvenations.
\end{itemize}
We generate $100$ observations within each sampler and represent them in \Cref{fig:gaus_example_2d}. For the MALA sampler, we generate $300$ samples and select every $3$th to maintain compatibility with the \XTryM\ setup. Note that in this example, the variance of the global proposals in \isir\ should be relatively large to cover well all modes of the \eqref{eq:gaus_mixture} mixture. However, since the modes are narrow, the step size of MALA cannot be very large to obtain a sensible acceptance rate. Therefore, \Cref{fig:gaus_example_2d} shows the drawbacks of the two approaches: \isir\ covers all modes of the target, but the chain often gets stuck at a certain point, which affects the variability of the samples. MALA allows a better local exploration of each mode, but does not cover the whole support of the target. The \XTryM\ algorithm combines the advantages of both methods by combining \isir-based global exploration with MALA -based local exploration.

Now, the mixture model of \eqref{eq:gaus_mixture} is modified with the weights parameters $\beta = (\beta_1,\beta_2,\beta_3) = (2/3,1/6,1/6)$ and same values of $\mu_{i}$ and $\sigma$. To compare the quality of the methods, we perform the following procedure
\begin{itemize}[leftmargin=*,nosep]
    \item starting with the initial distribution $\mathcal{N}(0,4\Id)$, we generate the trajectory $(X_1,\ldots,X_n)$ for different values of $n \in [25,800]$ for each of the compared methods (\isir\, MALA, \XTryM\ ). Sampler hyperparameters are the same as above, and the burn-in period equals $50$;
    \item We perform the kernel density estimate (KDE) $\widehat{p}_{n}$ based on the observations $(X_1,\ldots,X_n)$, and compute the total variation distance between $\widehat{p}_{n}$ and the target density $p_{\beta}$, and the forward $\KL{\widehat{p}_{n}}{p_{\beta}}$. Then we average the results over $100$ independent runs of each sampler.
\end{itemize}
Now we use the same values for the means and covariances but set  the mixing weights to $\beta = (\beta_1,\beta_2,\beta_3) = (2/3,1/6,1/6)$. To compare  the different sampling methods, we perform the following procedure.
\begin{itemize}[leftmargin=*,nosep]
\item starting from the initial distribution $\mathcal{N}(0,4\Id)$, we generate the trajectory $(X_1,\ldots,X_n)$ for different values of $n \in [25,800]$ for each of the compared methods (\isir\, MALA, \XTryM\ ). The hyperparameters of the sampler are the same as above, and the burn-in period is $50$;
\item We perform kernel density estimation (KDE) $\widehat{p}_{n}$ based on the observations $(X_1,\ldots,X_n)$ and calculate the total variation distance between $\widehat{p}_{n}$ and the target density $p_{\beta}$, as well as the forward value $\KL{\widehat{p}_{n}}{p_{\beta}}$. We then average the results over $100$ independent runs of each sampler.
\end{itemize}
The results for each sampler are given in \Cref{fig:mog_single_chain}, \Cref{fig:mog_single_chain_tv}. We also provide a simple illustration to the statements of \eqref{eq:tv_dist_isir} and \Cref{theo:main-geometric-ergodicity}. Starting from the initial distribution $\xi \sim \mathcal{N}(0,4\Id)$, we draw $500$ independent chains of length $50$ for each of the compared methods. Using these $500$ observations, we create a KDE $\widehat{p}_{n}$ for the density corresponding to the distribution of $\xi \MKQ^{n}$ for different $n \in \{5,\ldots,50\}$ and $\MKQ$ corresponding to \isir\, MALA or \XTryM\. Then we calculate the total variation distance between $\widehat{p}_{n}$ and the target density $p_{\beta}$. Corresponding plots can be found in \Cref{fig:mog_burnin}, \Cref{fig:mog_burnin_tv}. Note that \XTryM\ significantly outperforms the results of both MALA and \isir\. Indeed, the inhomogeneous mixture model is a complicated target for the Langevin-based methods. The trajectories generated by MALA tend to remain in a single mode of mixture \eqref{eq:gaus_mixture}, which reduces the reliability of the estimates and requires the generation of long trajectories even for $d = 2$. At the same time, it is difficult for \isir\ type methods without local exploration trajectories to quickly cover all the modes.

\subsection{Normalizing flow RealNVP}
We use the RealNVP architecture (\cite{dinh:2016}) for our experiments with adaptive MCMC. The key element of RealNVP is a coupling layer, defined as a transformation $f: \mathbb{R}^D \rightarrow \mathbb{R}^D$:
\begin{align} y_{1:d} &= x_{1:d} \\ y_{d+1:D} &= x_{d_1:D} \odot \exp(s(x_{1:d})) + t(x_{1:d})
\end{align}
 where $s$ and $t$ are some functions from $\mathbb{R}^D$ to $\mathbb{R}^D$. Thus, it is clear that the Jacobian of such a transformation is a triangular matrix with nonzero diagonal terms. We use fully connected neural networks to parameterize the functions $s$ and $t$.

 In all experiments with normalizing flows, we use the optimizer Adam (\cite{kingma2014adam}) with $\beta_1=0.9,~\beta_2 = 0.999$ and weight decay $0.01$ to avoid overfitting.

{ 
\subsection{High-dimensional multi-modal distribution}
\label{app:subsec:mixture-highd}

In an additional experiment we consider a high-dimensional toy target distribution: a Gaussian mixture similar as \Cref{supp:subsec:mixtures} above in $50d$. Modes are equally weighted, isotropic and well-separated. 

A purely local sampler would not mix between modes, as in the $2d$ case. A unimodal Gaussian proposal also fails in large dimension because of the concentration of the target measure in a small fraction of the proposal's bulk. Hence we only examine the performance of \FlXTryM. We set the number of proposals per iterations to $N=20$.

Using a RealNVP flow, we compare in Figure \ref{fig:3-gaussian-50d} the different outcomes depending on the choices of initialization of the MCMC walkers and training loss. Training the proposal offline through uniquely the backward KL (i.e. $\alpha=0$ in the combinaison of KL losses) is typically unstable in this multimodal case and the network collapse on the first detected mode. Successful backward-KL training is probably possible, yet at the cost of designing a proper annealing schedule of the target distribution as in \cite{Wu2019}. Resorting instead to a loss involving the forward KL ($\alpha=0.9$ in this experiment), mixing between the well separated modes in high-dimension is possible, provided that chain initialization ensures that all modes can be reached by the local kernel. 

To summarize, the choice of loss composition depends on the information a priori available on the considered target distribution. If rough location of modes is available - as it might be the case in chemistry applications where isomers of interest are known but sampling is necessary for relative free energy calculations - relying on the forward KL to draw the proposal to the modes is a simple and efficient strategy. Conversely, if little is known, there is no free lunch with the local-global kernels and an annealing might be necessary to train the global proposal, possibly using only the backward KL loss. 
}

\begin{figure}
    \includegraphics[width=1\textwidth]{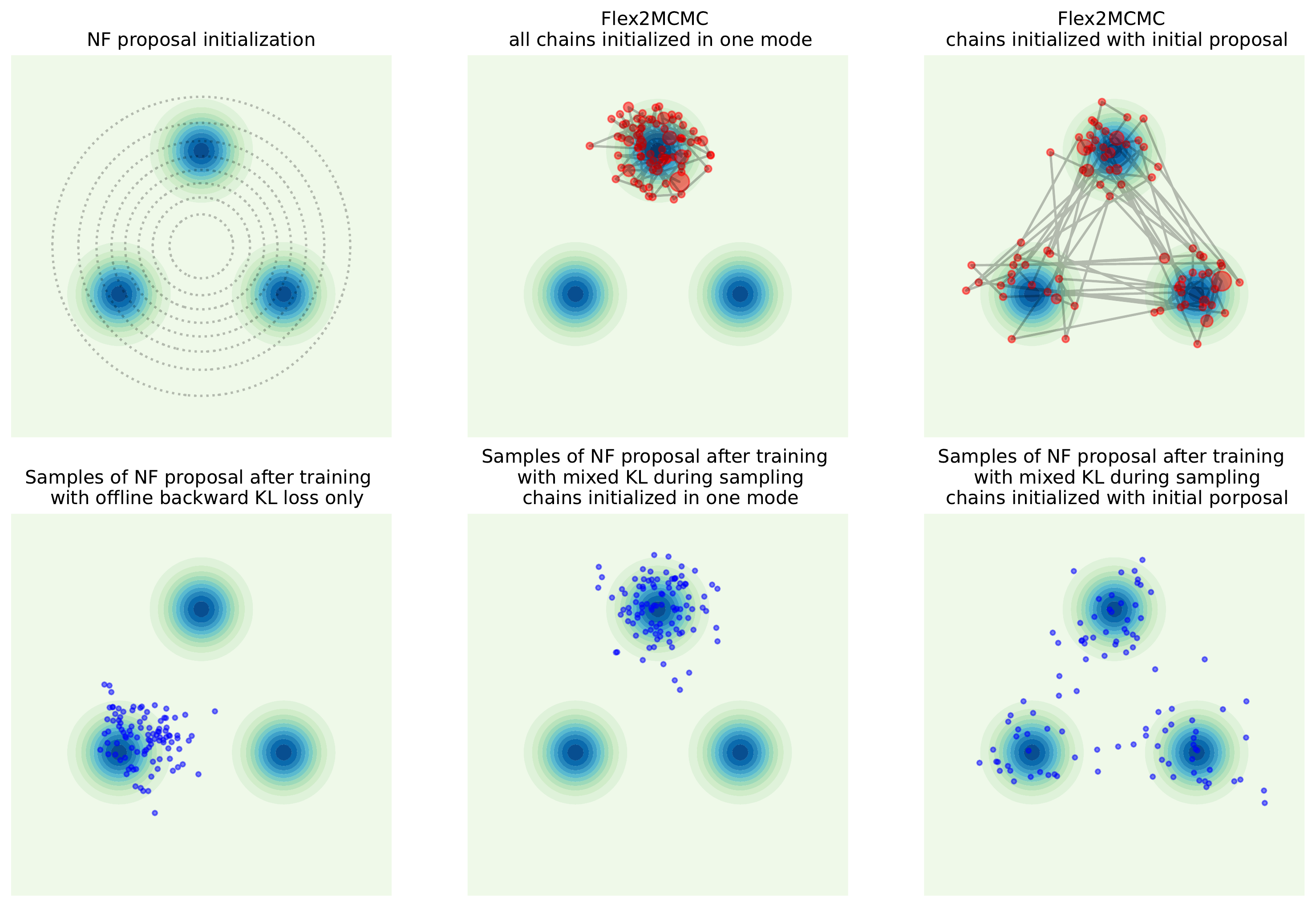}
    \caption{Importance of initialization and forward KL loss for multi-modal high-dimensional targets - All panels are 2$d$ projections of a $50d$ Euclidian space, with a target mixture of 3 isotropic Gaussian. Using a normalizing flow proposal distribution initialized as an isotropic Gaussian covering the 3 modes (top left), training with backward KL loss only still typically leads to mode collapse on one of the modes (bottom left). Running instead the simulataneous training and sampling of \FlXTryM with the mixture of backward and forward KL loss can lead to successful mixing between distant modes (top and bottom right), yet at the condition that chains are initialized such that all modes can be reached by the local-rejuvenation kernel- which is here enforced by an initialization as random draws of the initial proposal. Conversely, if all the chains are initialized in a single mode, the forward-KL estimated with states visited by the chains will not prevent a mode collapse (top and bottom center panels). \label{fig:3-gaussian-50d}}
\end{figure}

\subsection{Distributions with complex geometry}
\label{supp:subsec:complex-geometry}
In this section, we study the sampling quality from high-dimensional distributions, whose density levels have high curvature (Banana shaped and Funnel distributions, details below). With such distributions, standard MCMC algorithms like MALA or \isir, fail to explore fully the density support.


The corresponding densities are given for $x \in \rset^{d}$ by
\begin{equation}
\label{eq:complex_geometry_example}
\begin{split}
p_{f}(x) &= \normconst^{-1} \exp\left(- x_1^2/2a^2 - (1/2) \rme^{-2bx_1} \sum\nolimits_{i=2}^{d}\{x_i^2 + 2b x_1\}\right)\eqsp, \quad d \geq 2, \\
p_{b}(x) &= \normconst^{-1} \exp\left(- \sum\nolimits_{i=1}^{d/2}\bigl\{x_{2i}^2/2a^2 - (x_{2i-1} - bx_{2i}^{2} + a^2 b)^{2}/2 \bigr\}\right)\eqsp, \quad d = 2k, k \in \nset\eqsp.
\end{split}
\end{equation}
where $\normconst$ is a normalizing constant. We set $a = 2$, $b = 0.5$ for funnel and $a = 5$, $b = 0.02$ for banana-shape distributions, respectively. For MALA we use an adaptive step size tuning strategy to maintain acceptance rate approximately $0.5$. For \isir\ and \XTryM\ algorithms we use wide Gaussian global proposal $\mathcal{N}(0,\sigma^{2}_{p}\Id)$ with $\sigma^2_{p} = 4$ for Funnel and $\sigma^{2}_{p} = 9$ for Banana-shape distribution.

For \FlXTryM\, use a simple RealNVP-based normalizing flow \cite{dinh:2016} with $4$ hidden layers. Note that for $p_{f}(x)$ the energy landscape in the region with $x_{1} < 0$ is steep, so the distributions \eqref{eq:complex_geometry_example} are hard to capture, especially when the dimension $d$ is large. Moreover, due to the complex geometry of the distribution support, we cannot hope that local samplers (MALA) or global samplers (\isir\ ) alone will give good results. In this example, we want to compare \FlXTryM\ with \isir\, MALA and the HMC-based NUTS sampler~\cite{hoffman2014no}. We also add a vanilla version of the \XTryM\ algorithm to the comparison. To generate the ground-truth samples, we use the explicit reparametrisation of \eqref{eq:complex_geometry_example}. Indeed, given a random vector $(Z_1,\ldots,Z_d) \sim \mathcal{N}(0,\Id)$, we consider its transformation $(X_1,\ldots,X_d)$ under the formulas
\begin{equation}
\label{eq:funnel_distribution_reparam}
\begin{cases}
&X_1 = a Z_1 \\
&X_i = \rme^{bX_{1}} Z_{i}\eqsp, \quad i \in \{2,\ldots,d\}\eqsp.
\end{cases}
\end{equation}
It is easy to check that $(X_1,\ldots,X_d)$ follows the density $p_{f}(x), x \in \rset^{d}$. Similarly, for $d = 2k$ consider the transformation
\begin{equation}
\label{eq:banana_distribution_reparam}
\begin{cases}
&Y_{2i} = a Z_{2i} \\
&Y_{2i-1} = Y_{2i} + b Y_{2i}^{2} - ba^{2}\eqsp, \quad i \in \{1,\ldots,k\}\eqsp.
\end{cases}
\end{equation}
Then $(Y_1,\ldots,Y_d)$ follows the density $p_{b}(x), x \in \rset^{d}$. We provide the average computation time for NUTS, adaptive \isir\ and \FlXTryM\ algorithms in \Cref{tab:time_funnel} and \Cref{tab:time_banana} for the Funnel and Banana-shape distributions, respectively, averaged over $50$ runs. Note that different runs of NUTS algorithm yields high variance of the running time, especially for the Funnel distribution and dimensions $d \geq 50$.

We give the computation time for the above algorithms and additional implementation details in \Cref{supp:subsec:complex-geometry}. The implementation of \FlXTryM\ is based on the use of $5$ MALA steps as rejuvenation steps.

\begin{table}[h]
\centering
\begin{tabular}{c|c|c|c|c|c}
        \toprule
        Method & $d=10$ & $d=20$ & $d=50$ & $d = 100$ & $d = 200$ \\
        \midrule
         NUTS & $33.4 \pm 8.2$ & $41.1 \pm 12.3$ & $61.6 \pm 30.2$ & $82.3 \pm 73.2$ & $88.4 \pm 59.5$\\
         Adaptive \isir\ & $38.1 \pm 3.2$ & $39.4 \pm 2.8$ & $45.3 \pm 2.5$ & $59.8 \pm 0.7$ & $80.4 \pm 0.4$ \\
         \FlXTryM\ & $46.8 \pm 3.2$ & $48.2 \pm 2.8$ & $54.2 \pm 2.5$ & $68.8 \pm 0.8$ & $89.5 \pm 0.5$\\
\end{tabular}
\caption{Computational time for the Funnel distribution.}
\label{tab:time_funnel}
\end{table}

\begin{table}[h]
\centering
\begin{tabular}{c|c|c|c|c|c}
        \toprule
        Method & $d=20$ & $d=40$ & $d=60$ & $d = 80$ & $d = 100$ \\
        \midrule
         NUTS & $27.6 \pm 1.8$ & $32.1 \pm 1$ & $34.2 \pm 0.5$ & $35.2 \pm 0.5$ & $35.9 \pm 0.4$\\
         Adaptive \isir\ & $24.5 \pm 0.2$ & $26.8 \pm 0.3$ & $28.5 \pm 0.2$ & $30.1 \pm 0.2$ & $32.8 \pm 0.2$ \\
         \FlXTryM\ & $39.3 \pm 0.5$ & $41.8 \pm 0.3$ & $43.5 \pm 0.3$ & $45.1 \pm 0.3$ & $47.8 \pm 0.4$\\
\end{tabular}
\caption{Computational time for the Banana-shape distribution.}
\label{tab:time_banana}
\end{table}




\subsection{GANs as energy-based models}
\subsubsection{MNIST results}
\label{supp:sec:MNIST}
For this example, we consider both the Wasserstein GAN (WGAN) setup with energy function $E_{W}(z)$ and the classical Jensen-Shannon GAN with energy function $E_{ JS }(z)$. In both cases, we use fully connected networks with $3$ convolutional layers for discriminator and $3$ linear + $3$ convolutional layers for generator. For WGAN training, we use gradient penalty regularisation, following \cite{gulrajani:2017:wgan}. We provide additional visualisations of the latent space and samples along a given trajectory for Jensen-Shannon GAN in \Cref{fig:js_gan_visualize_mnist} and for Wasserstein GAN in \Cref{fig:was_gan_visualize_mnist}. Sampling hyperparameters are summarized in \Cref{tab:mnist_exp_details}. For fair comparison, we take each $3$-rd sample produced by the MALA, when running this algorithm separately. Both for WGAN-GP and vanilla GAN experiments we apply \isir\ and \XTryM\ with wide Gaussian global proposal $\mathcal{N}(0,\sigma^{2}_{p})$. The particular values of $\sigma_{p}^{2}$ are specified in \Cref{tab:mnist_exp_details}.

\begin{table}[b!]
\centering
\begin{tabular}{c|c|c|c|c|c}
        \toprule
        Method & \# iterations & MALA step size $\gamma$ & \# particles, $N$ & $\sigma_{p}^{2}$ & \# MALA steps \\
        \midrule
         JS-GAN & $100$ & $0.02$ & $10$ & $9$ & $3$\\
         WGAN-GP & $100$ & $0.02$ &  $10$ & $9$ & $3$  \\
\end{tabular}
\caption{MNIST hyperparameters.}
\label{tab:mnist_exp_details}
\end{table}

\begin{figure}
\begin{subfigure}{0.9\textwidth}
    \includegraphics[width=1\textwidth]{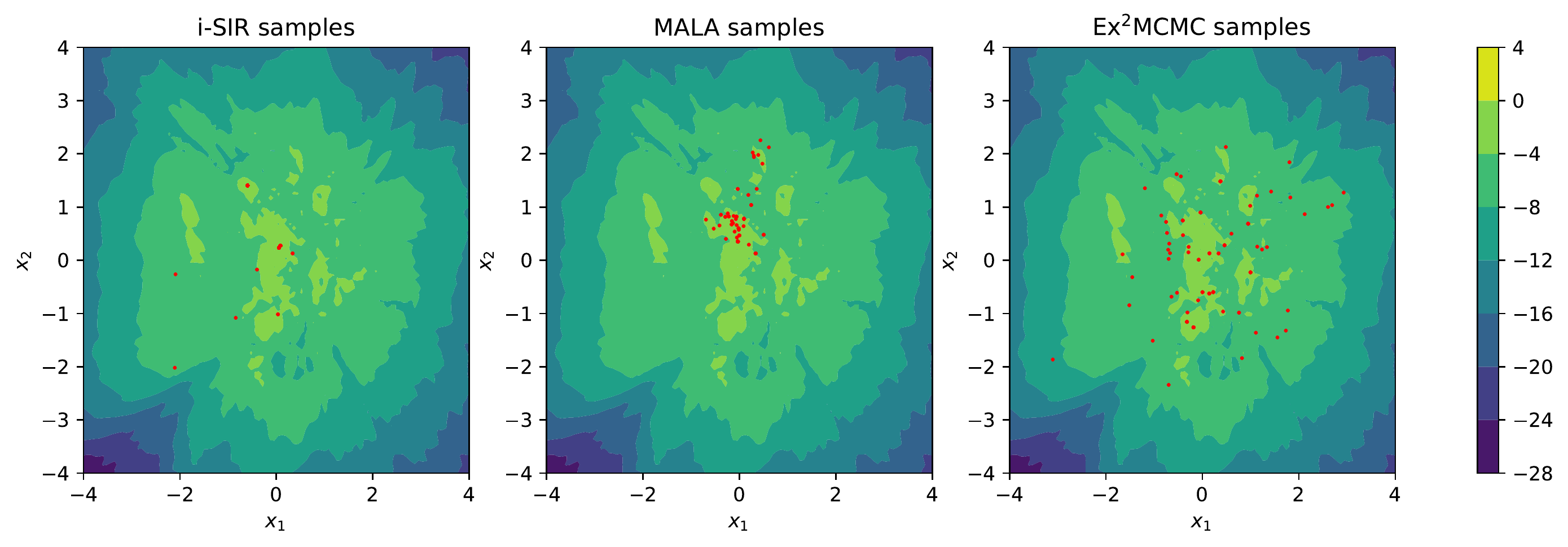}
    \caption{JS-GAN: latent space visualizations}
    \label{fig:samples_mnist_js_2d_latent}
\end{subfigure}
\centering
\begin{subfigure}{0.3\textwidth}
    \includegraphics[width=1\textwidth]{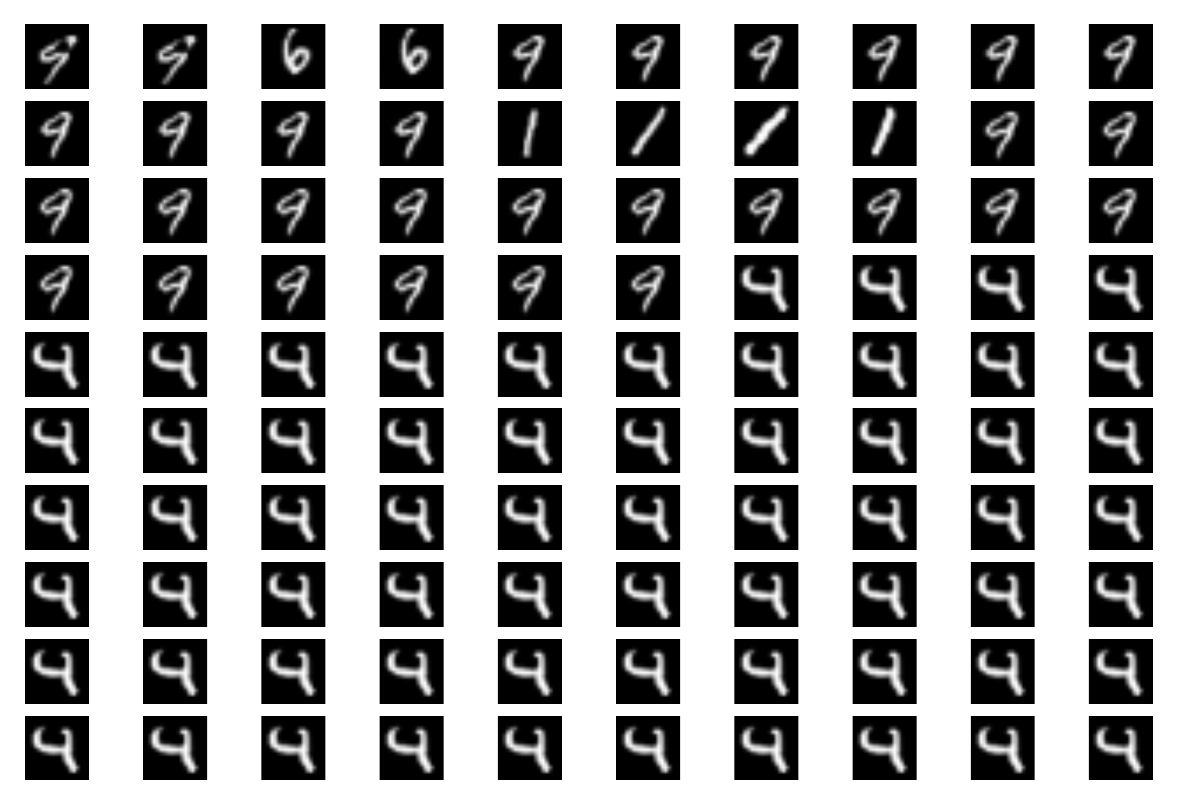}
    \caption{\isir\ samples}
    \label{fig:samples_mnist_js_2d_isir}
\end{subfigure}
\hspace{5mm}
\begin{subfigure}{0.3\textwidth}
    \includegraphics[width=1\textwidth]{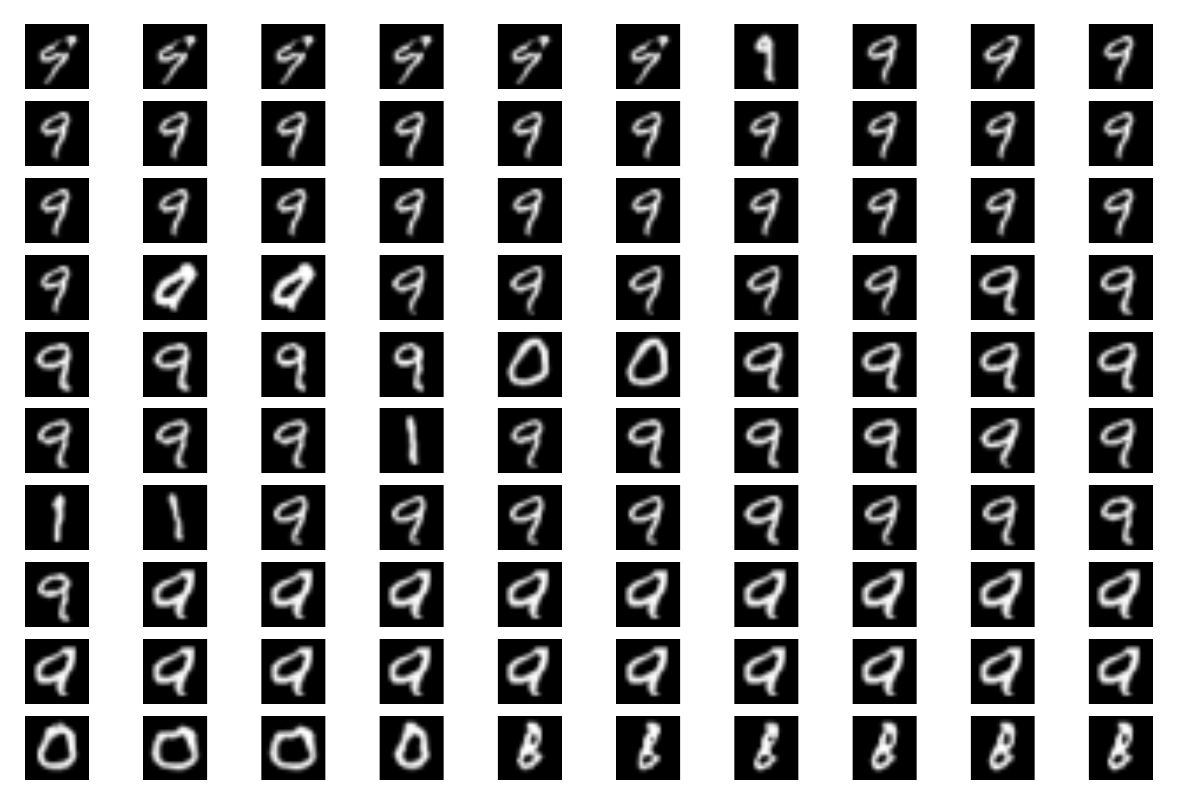}
    \caption{MALA samples}
    \label{fig:samples_mnist_js_2d_mala}
\end{subfigure}
\hspace{5mm}
\begin{subfigure}{0.3\textwidth}
    \includegraphics[width=1\textwidth]{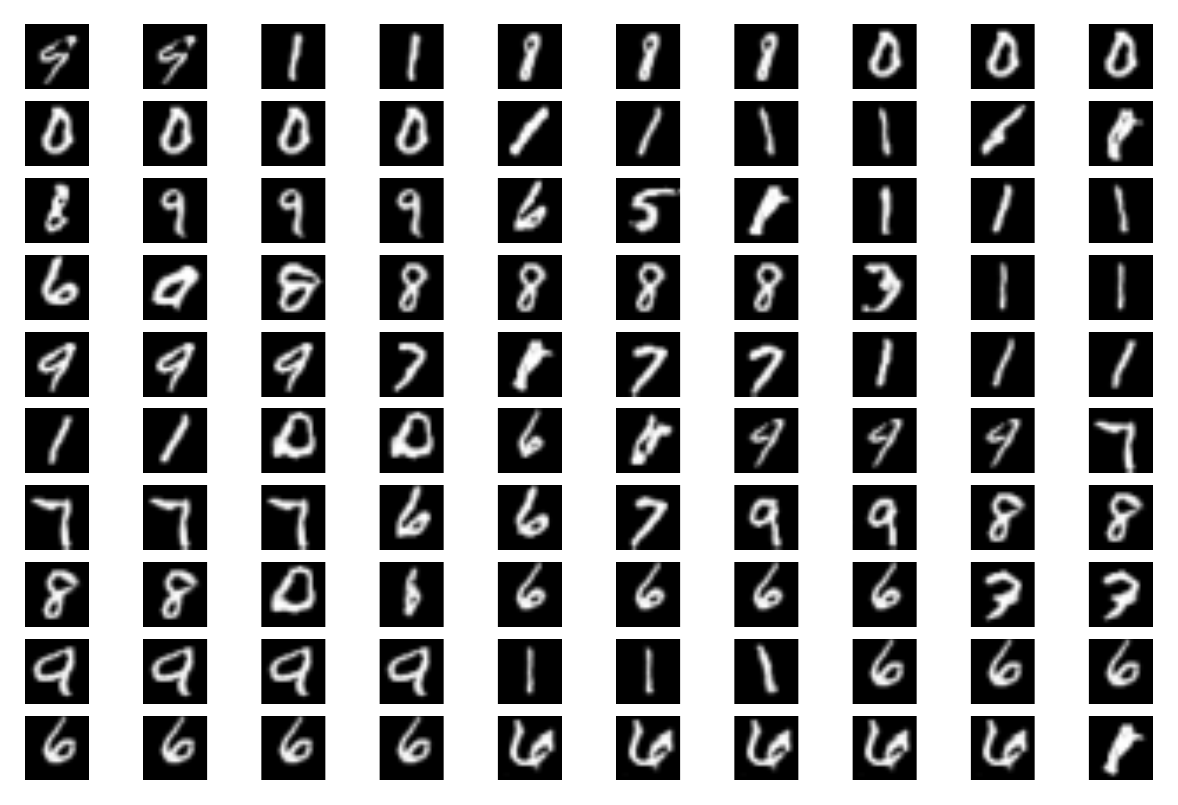}
    \caption{\XTryM\ samples}
    \label{fig:samples_mnist_js_2d_ex2}
\end{subfigure}
\label{fig:js_gan_visualize_mnist}
\end{figure}

\begin{figure}
\begin{subfigure}{0.9\textwidth}
    \includegraphics[width=1\textwidth]{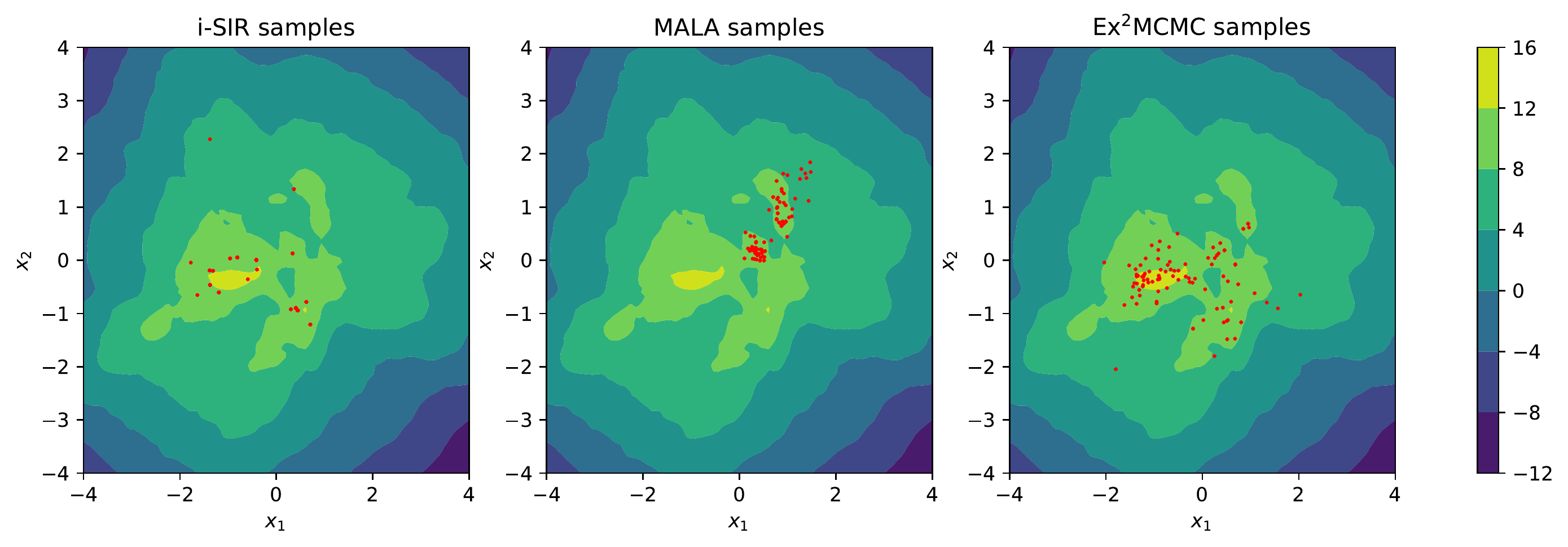}
    \caption{JS-GAN: latent space visualizations}
    \label{fig:samples_mnist_wgan_2d_latent}
\end{subfigure}
\centering
\begin{subfigure}{0.3\textwidth}
    \includegraphics[width=1\textwidth]{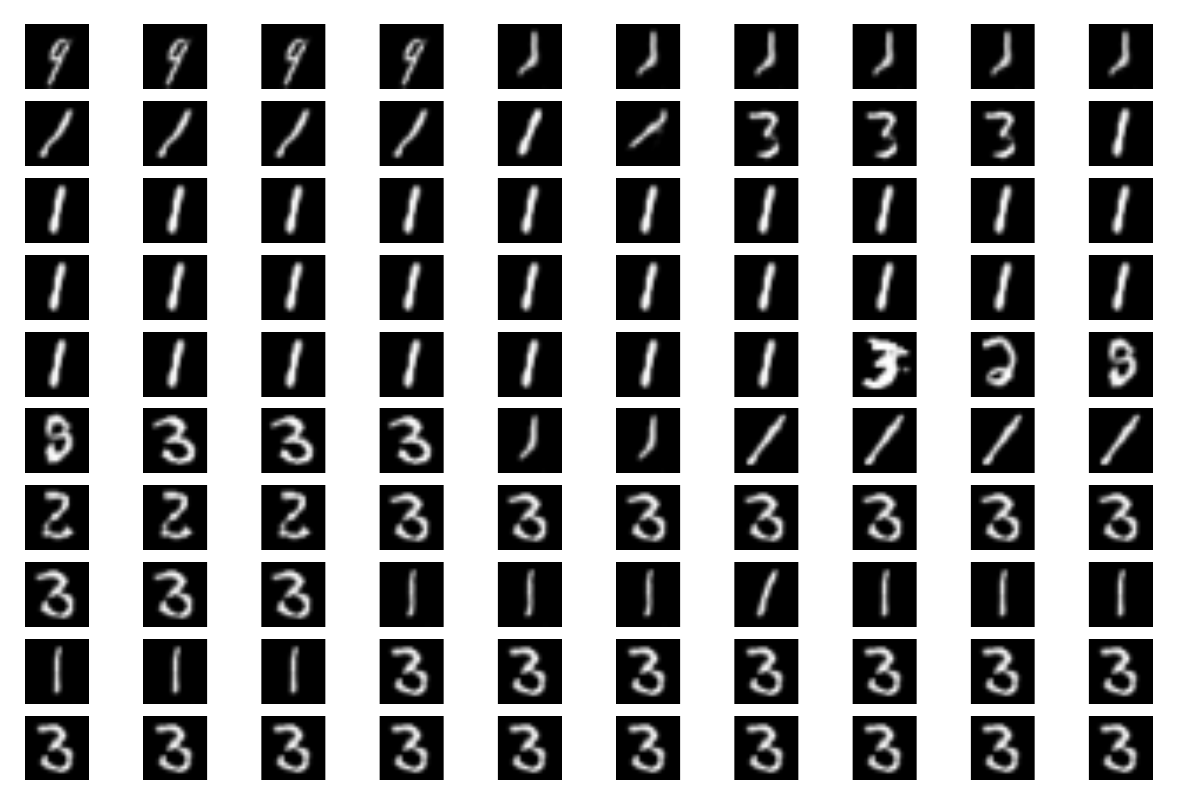}
    \caption{\isir\ samples}
    \label{fig:samples_mnist_wgan_2d_isir}
\end{subfigure}
\hspace{5mm}
\begin{subfigure}{0.3\textwidth}
    \includegraphics[width=1\textwidth]{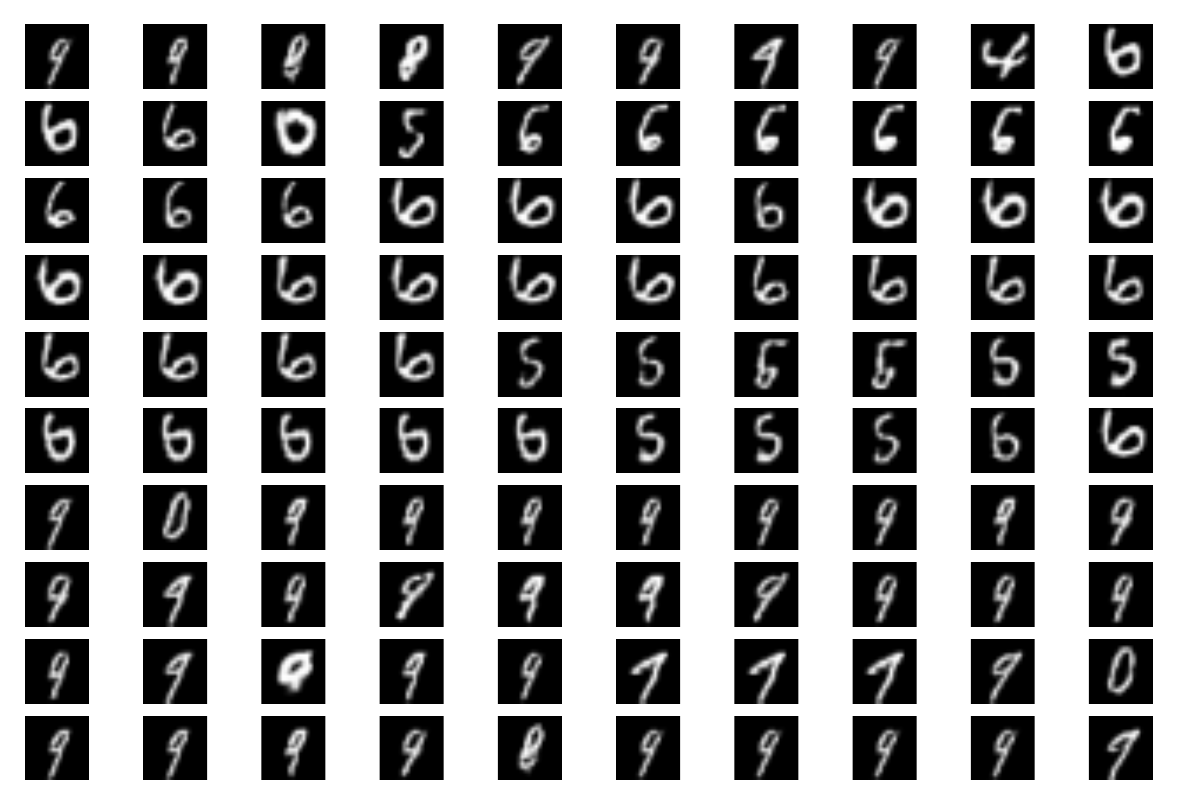}
    \caption{MALA samples}
    \label{fig:samples_mnist_wgan_2d_mala}
\end{subfigure}
\hspace{5mm}
\begin{subfigure}{0.3\textwidth}
    \includegraphics[width=1\textwidth]{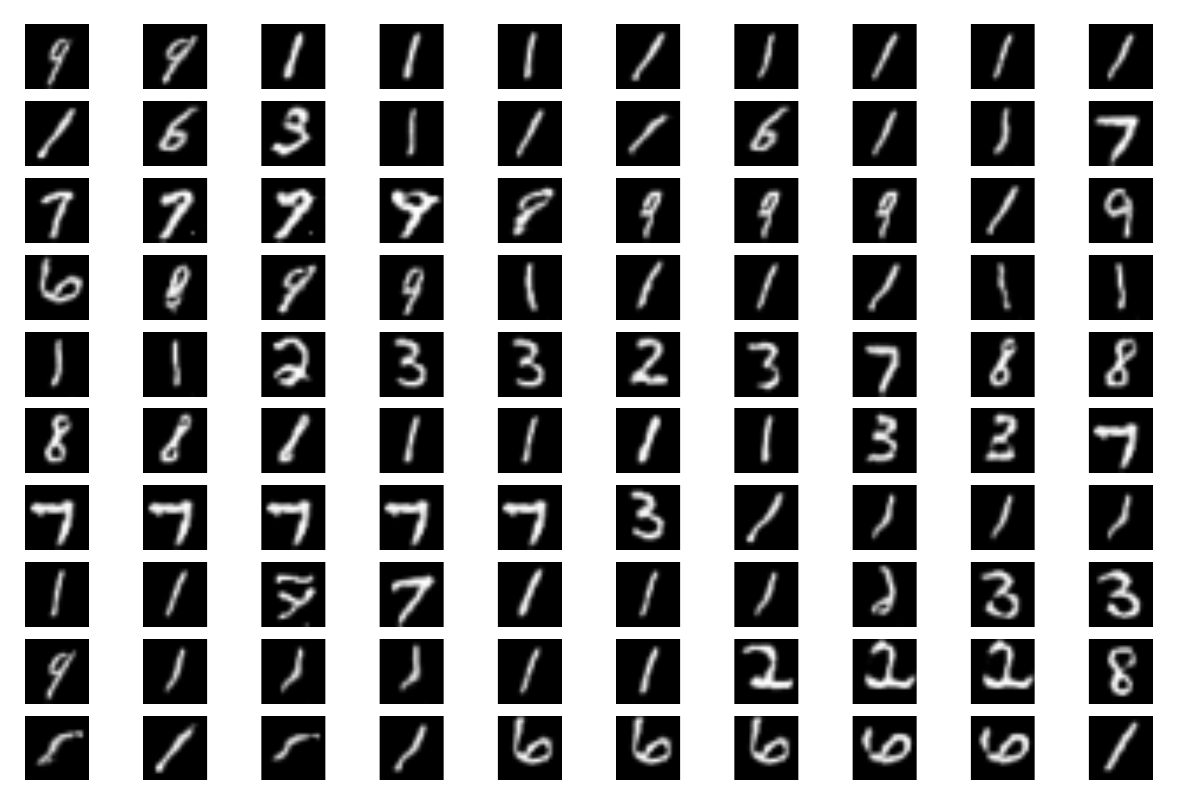}
    \caption{\XTryM\ samples}
    \label{fig:samples_mnist_wgan_2d_ex2}
\end{subfigure}
\label{fig:was_gan_visualize_mnist}
\end{figure}

\subsubsection{Cifar-$10$ results}
\label{supp:sec:cifar}

We consider two popular GAN architectures, DC-GAN~\cite{radford2016unsupervised} and SN-GAN~\cite{miyato2018spectral}. Below we provide the details on experimental setup and evaluation for both of the models.

\subsection{Training and sampling details.}
For DC-GAN and SN-GAN experiments, we took the implementation and training script of the models from Mimicry repository \url{https://github.com/kwotsin/mimicry}. Both models were trained on a single GPU GeForce GTX 1060 for approximately $20$ hours. 

Both for DC-GAN and SN-GAN, the latent dimension is equal to $d = 128$. Following \cite{Che2020}, for both models we consider sampling from the latent spatial distribution
\begin{equation}
p(z) = \rme^{-E_{JS}(z)}/Z\eqsp, \quad z \in \rset^d\eqsp, \quad E_{JS}(z) = -\log p_0(z) - \logit\bigl(D(G(z))\bigr)\eqsp,
\end{equation} where $\logit(y) = \log\left(y/(1-y)\right)\, y \in (0,1)$ is the inverse of the sigmoid function and $p_0(z) = \mathcal{N}(0,\Id)$.


\paragraph{Evaluation protocol} We perform $n = 100$ iterations of the algorithms MALA, \isir\, \XTryM\, and \FlXTryM\. For both the vanilla \XTryM\ algorithm (\Cref{alg:X-Try-MCMC}) and \FlXTryM\, we use the Markov kernel \eqref{eq:def-kernel-MALA}, which corresponds to $3$ MALA steps, as the rejuvenation kernel. The step size $\gamma$ given for the algorithm \XTryM\ corresponds to its rejuvenation kernel MALA. For more experimental details, see \Cref{tab:cifar_exp_details}. For \isir\ and \XTryM\ algorithms we use $\mathcal{N}(0,\sigma^{2}_{p}\Id)$ with $\sigma^{2}_{p} = 1$ as a global proposal distribution.

We run $M = 500$ independent chains for each of the above MCMC algorithms. Then, for the $j-$th iteration, we compute the average value of the energy function $E(z)$ averaged over $M$ chains. Hyperparameters are specified in \Cref{tab:cifar_exp_details}. Energy profiles for different algorithms for DC-GAN and SN-GAN are provided in \Cref{fig:energy_profile_dcgan} and \Cref{fig:energy_profile_sngan}, respectively. Note that in both cases \XTryM\ or \FlXTryM\ algorithms yields lower energy samples. We visualize $10$ randomly chosen trajectories obtained with each sampling methods in \Cref{fig:mala_isir_sngan}-\Cref{fig:ex2_flex2_sngan} for SN-GAN and \Cref{fig:mala_isir_dcgan}-\Cref{fig:ex2_flex2_dcgan} for DC-GAN, respectively. For each trajectory we visualize every $10$-th sample. Both architectures indicate the same findings: MALA typically is not available to escape the mode of the corresponding target density $p(z)$ during one particular run. \isir\ travels well across the support of $p(z)$, yet the corresponding energy values are higher then the ones of \XTryM\ or \FlXTryM\,. Some \isir\ trajectories can get trapped in one particular image due to the absence of local exploration moves. At the same time, \XTryM\, as illustrated in \Cref{fig:ex2_flex2_sngan}-\ref{fig:samples_sngan_cifar_ex2} and \Cref{fig:ex2_flex2_dcgan}-\ref{fig:samples_dcgan_cifar_ex2}, can both exploit the particular mode of the distribution and perform global moves over the support of $p(z)$. Of course, these global moves are more likely to occur during the first sampling iterations. {For the DC-GAN architecture, we provide also the dynamics of FID (Frechet Inception Distance, \cite{fid_metric}), and IS (Inception Score, \cite{inception_score}) values computed over $10000$ independent trajectories. We plot the metrics in \Cref{fig:is_dcgan_supp} and \Cref{fig:fid_dcgan_supp}. Metrics illustrate the image quality improvement achieved by \FlXTryM\ and \XTryM\ algorithms.}

\begin{table}[b]
\centering
\begin{tabular}{c|c|c|c|c|c}
        \toprule
        GAN type & \# iterations & MALA step size $\gamma$ & \# particles, $N$ & $\sigma_{p}^{2}$ & \# MALA steps \\
        \midrule
         SNGAN & $100$ & $ 5 \times 10^{-3}$ & $10$ & $1$ & $3$\\
         DCGAN & $100$ & $ 10^{-3}$ & $10$ & $1$ & $3$  \\
\end{tabular}
\caption{CIFAR-10 hyperparameters.}
\label{tab:cifar_exp_details}
\end{table}

\begin{figure}
\centering
\begin{subfigure}{0.49\textwidth}
    \includegraphics[width=1\textwidth]{figures/energy_dcgan.png}
    \caption{DC-GAN}
    \label{fig:energy_dcgan_cifar10}
\end{subfigure}
\begin{subfigure}{0.49\textwidth}
    \includegraphics[width=1\textwidth]{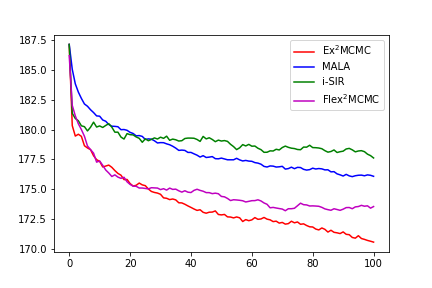}
    \caption{SN-GAN}
    \label{fig:energy_sngan_cifar10}
\end{subfigure}
\caption{Energy profile for DC-GAN and SN-GAN architectures on CIFAR-$10$ dataset.}
\label{fig:energy_dcgan_sngan.png}

\begin{subfigure}{0.49\textwidth}
    \centering
    \includegraphics[width=0.95\textwidth]{figures/is_dcgan.pdf}
    \caption{Inception Score dynamics for DC-GAN}
    \label{fig:is_dcgan_supp}
\end{subfigure}
\begin{subfigure}{0.49\textwidth}
    \centering
    \includegraphics[width=0.95\textwidth]{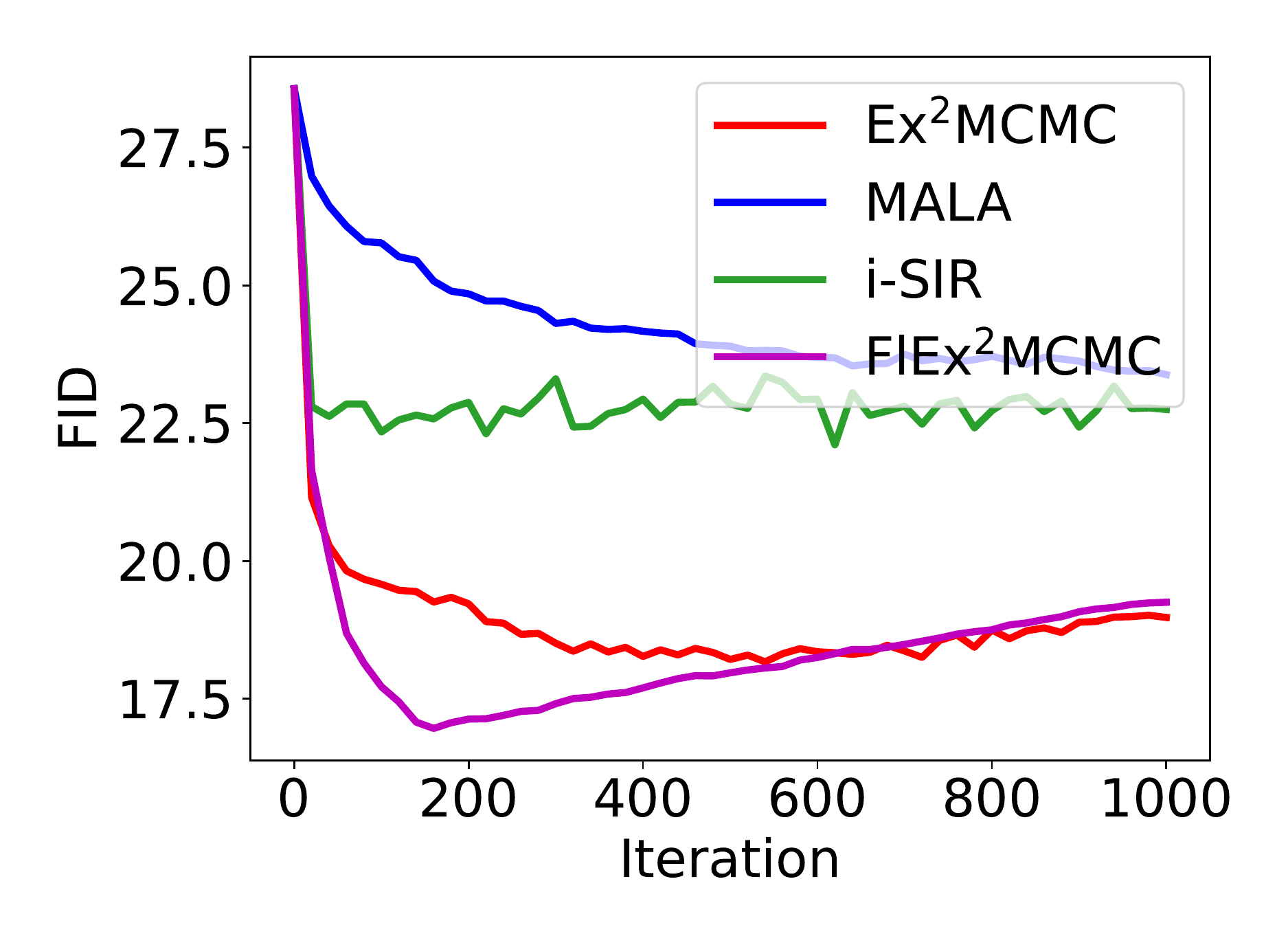}
    \caption{FID dynamics for DC-GAN}
    \label{fig:fid_dcgan_supp}
\end{subfigure}
\caption{Dynamics of Inception Score (a) and FID (b) computed over 10000 independent trajectories for DC-GAN trained on CIFAR-10 dataset.}
\label{fig:dynamics_dcgan_supp}
\end{figure}

\begin{figure}[h]
\centering
\begin{subfigure}{1.\textwidth}
    \includegraphics[width=1\textwidth]{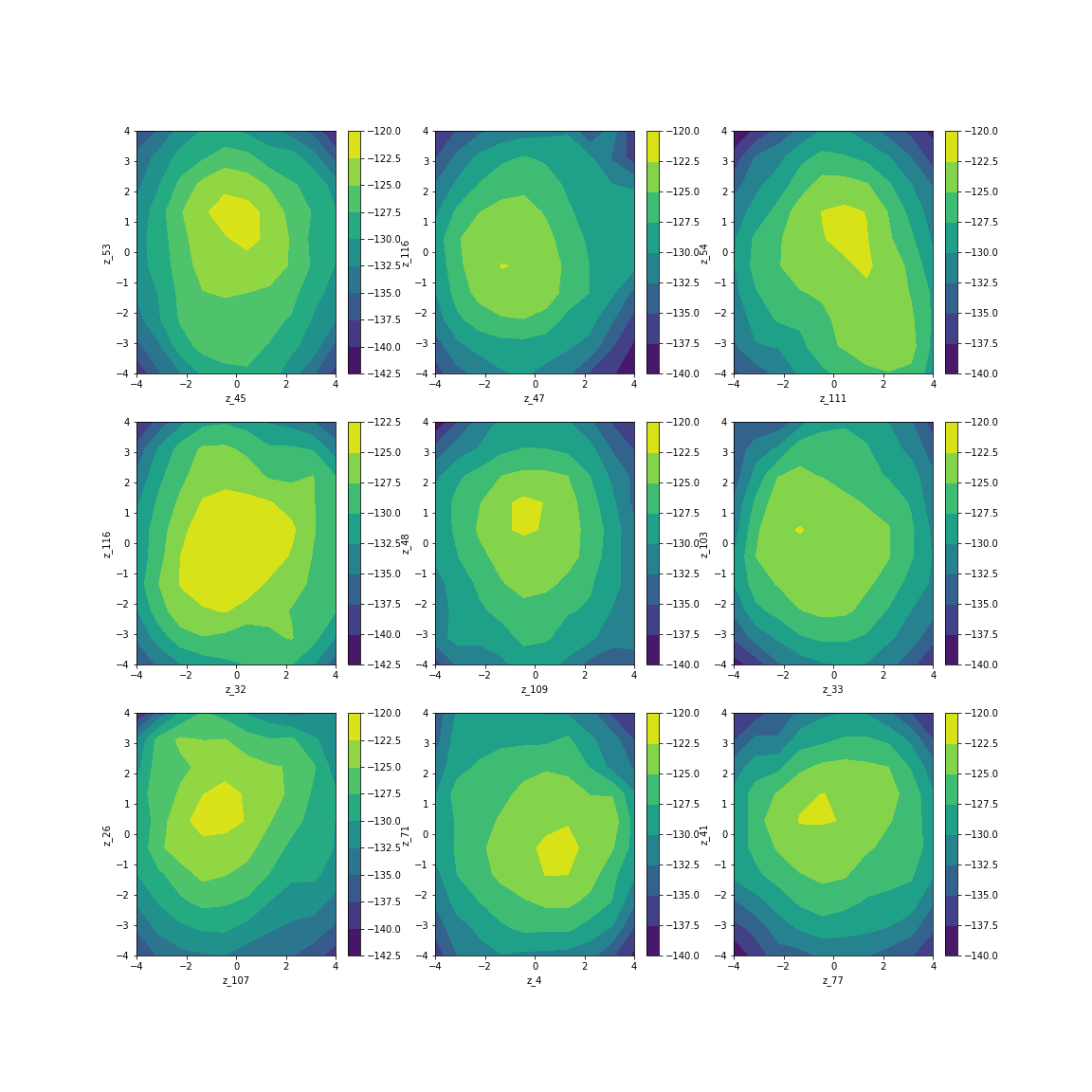}
    \vspace{-2\baselineskip}
\end{subfigure}
\caption{Energy profile for random axis pairs, SN-GAN}
\label{fig:energy_profile_sngan}
\end{figure}

\begin{figure}[h]
\vspace{-2\baselineskip}
\centering
\begin{subfigure}{0.8\textwidth}
    \includegraphics[width=1\textwidth]{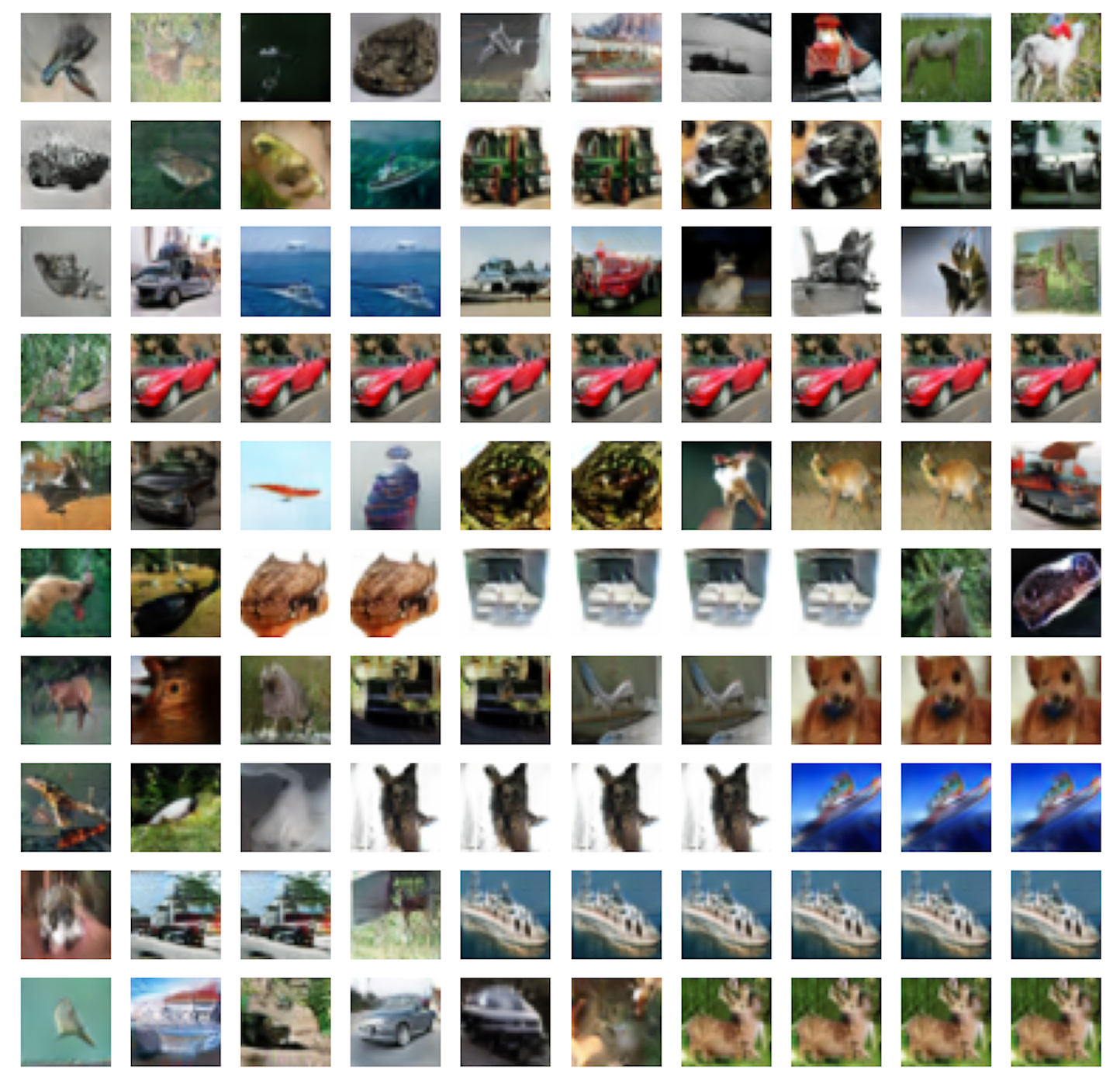}
    \caption{\isir\ samples}
    \label{fig:samples_sngan_cifar_isir}
\end{subfigure}
\centering
\begin{subfigure}{0.8\textwidth}
    \includegraphics[width=1\textwidth]{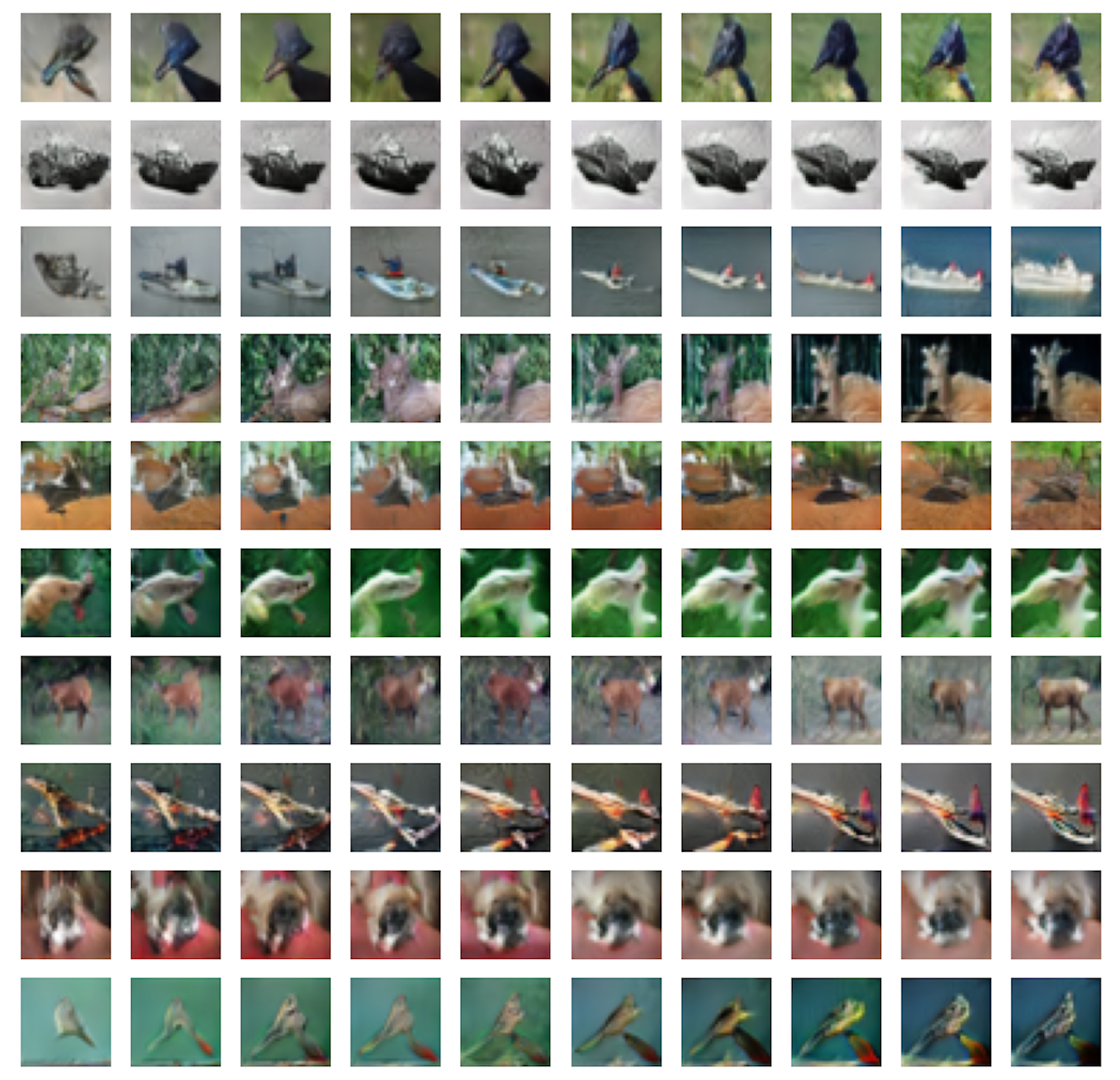}
    \caption{MALA samples}
    \label{fig:samples_sngan_cifar_mala}
\end{subfigure}
\caption{\isir\ and MALA samples, SN-GAN.}
\label{fig:mala_isir_sngan}
\end{figure}

\begin{figure}[h]
\vspace{-2\baselineskip}
\centering
\begin{subfigure}{0.8\textwidth}
    \includegraphics[width=1\textwidth]{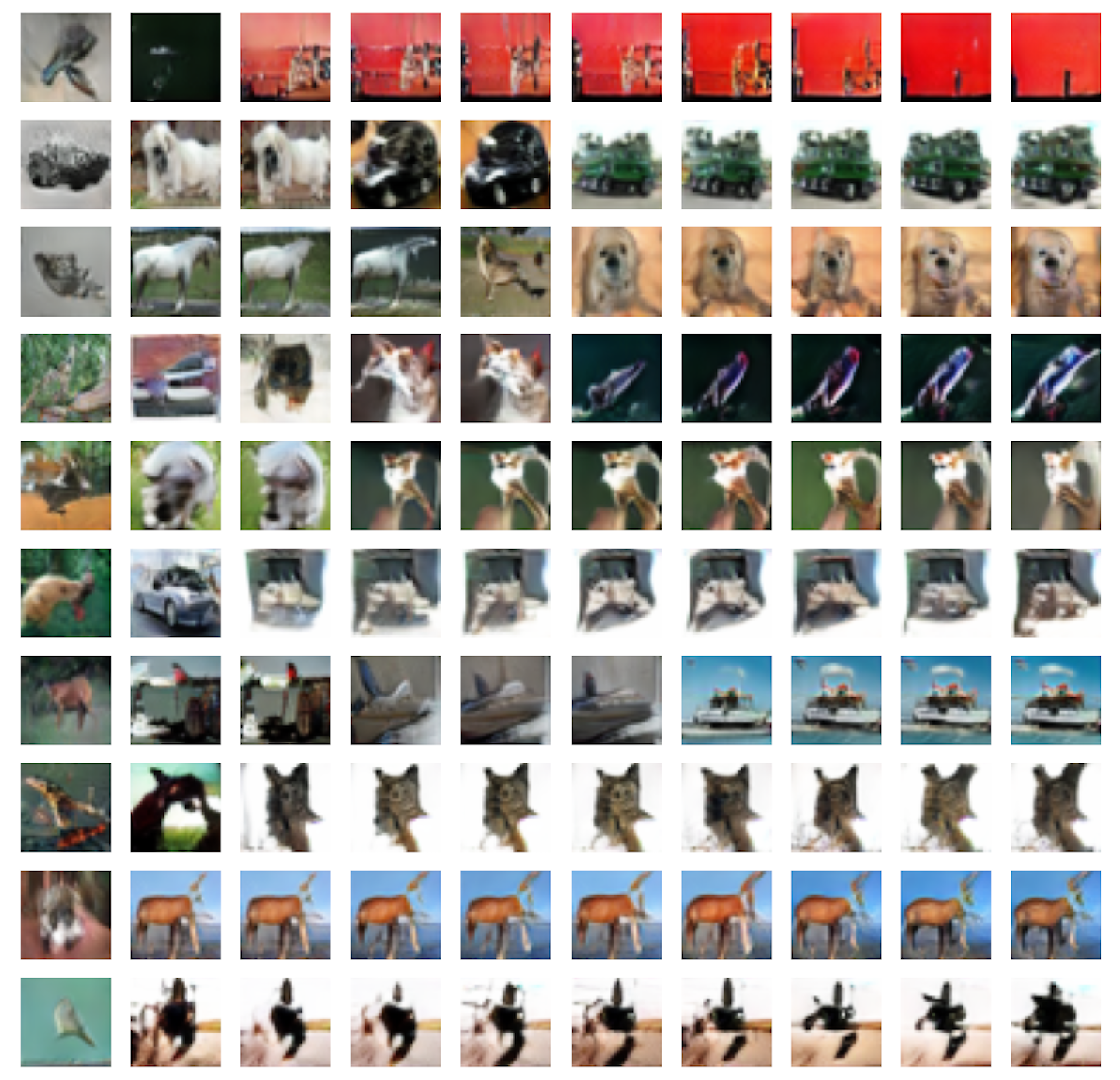}
    \caption{\XTryM\ samples}
    \label{fig:samples_sngan_cifar_ex2}
\end{subfigure}
\centering
\begin{subfigure}{0.8\textwidth}
    \includegraphics[width=1\textwidth]{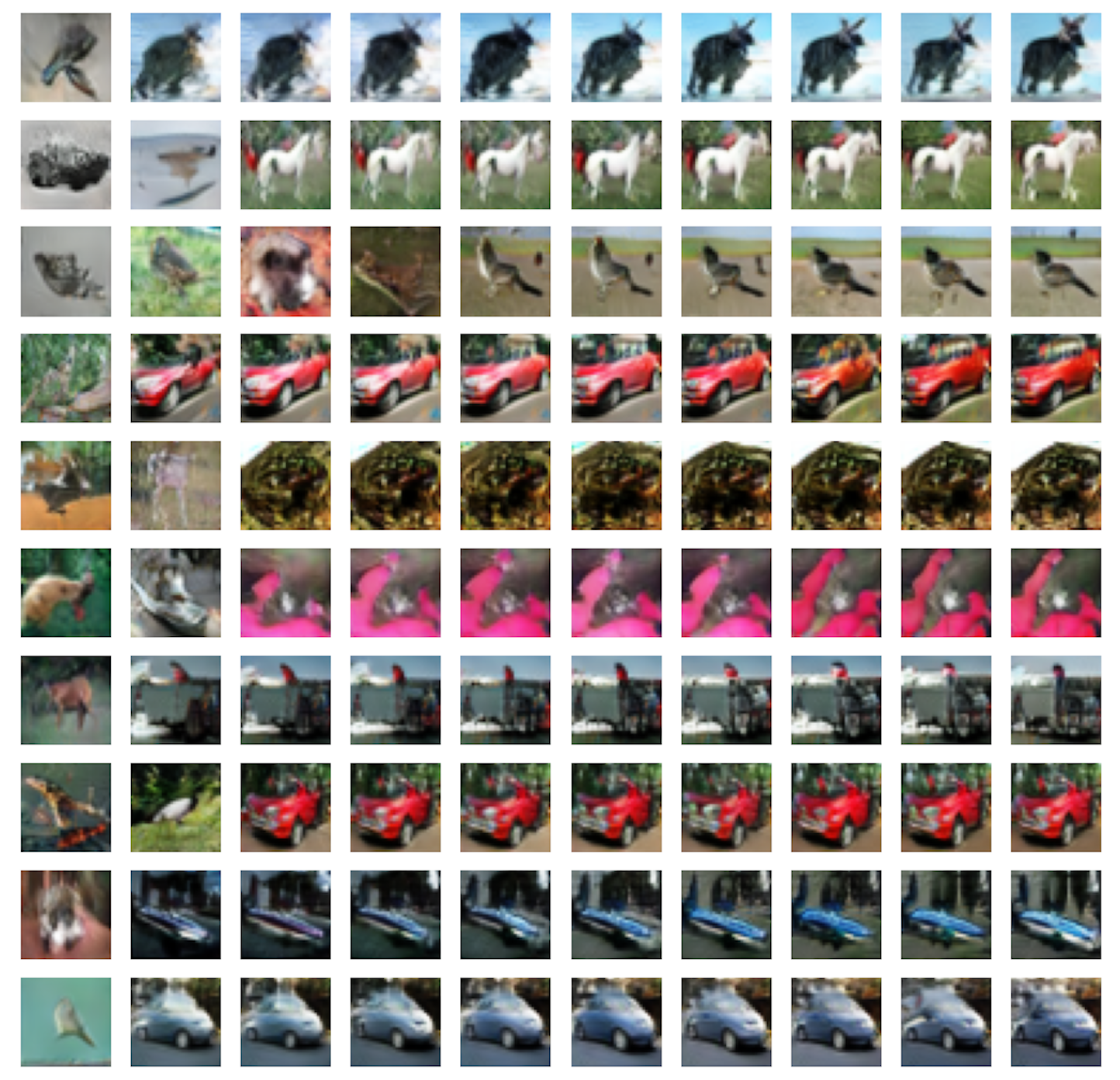}
    \caption{\FlXTryM\ samples}
    \label{fig:samples_sngan_cifar_flex2}
\end{subfigure}
\caption{\XTryM\ and \FlXTryM\ samples, SN-GAN.}
\label{fig:ex2_flex2_sngan}
\end{figure}

\begin{figure}[h]
\vspace{-2\baselineskip}
\centering
\begin{subfigure}{1.\textwidth}
    \includegraphics[width=1\textwidth]{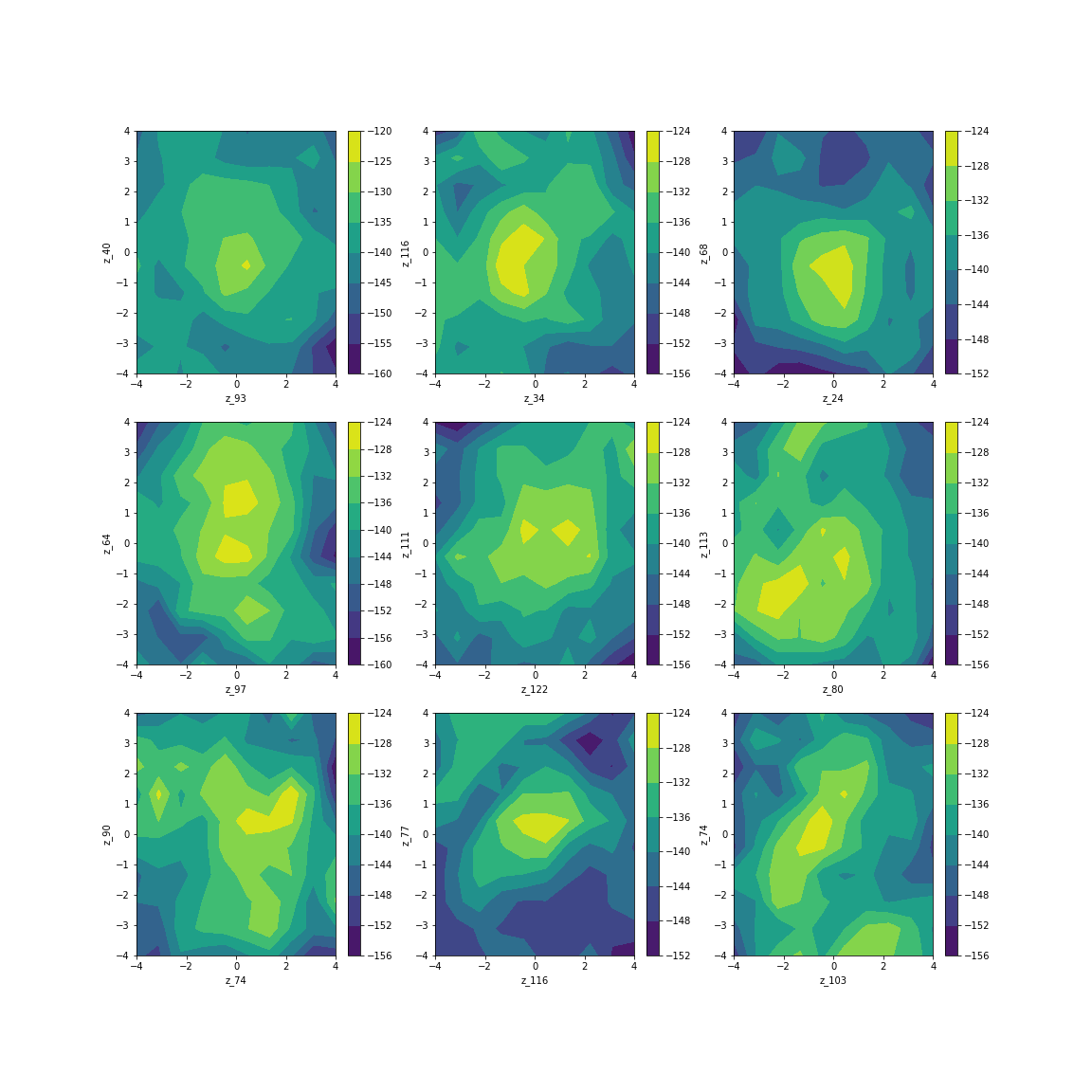}
    \vspace{-2\baselineskip}
\end{subfigure}
\caption{Energy profile for random axis pairs, DC-GAN}
\label{fig:energy_profile_dcgan}
\end{figure}

\begin{figure}[h]
\vspace{-2\baselineskip}
\centering
\begin{subfigure}{0.8\textwidth}
    \includegraphics[width=1\textwidth]{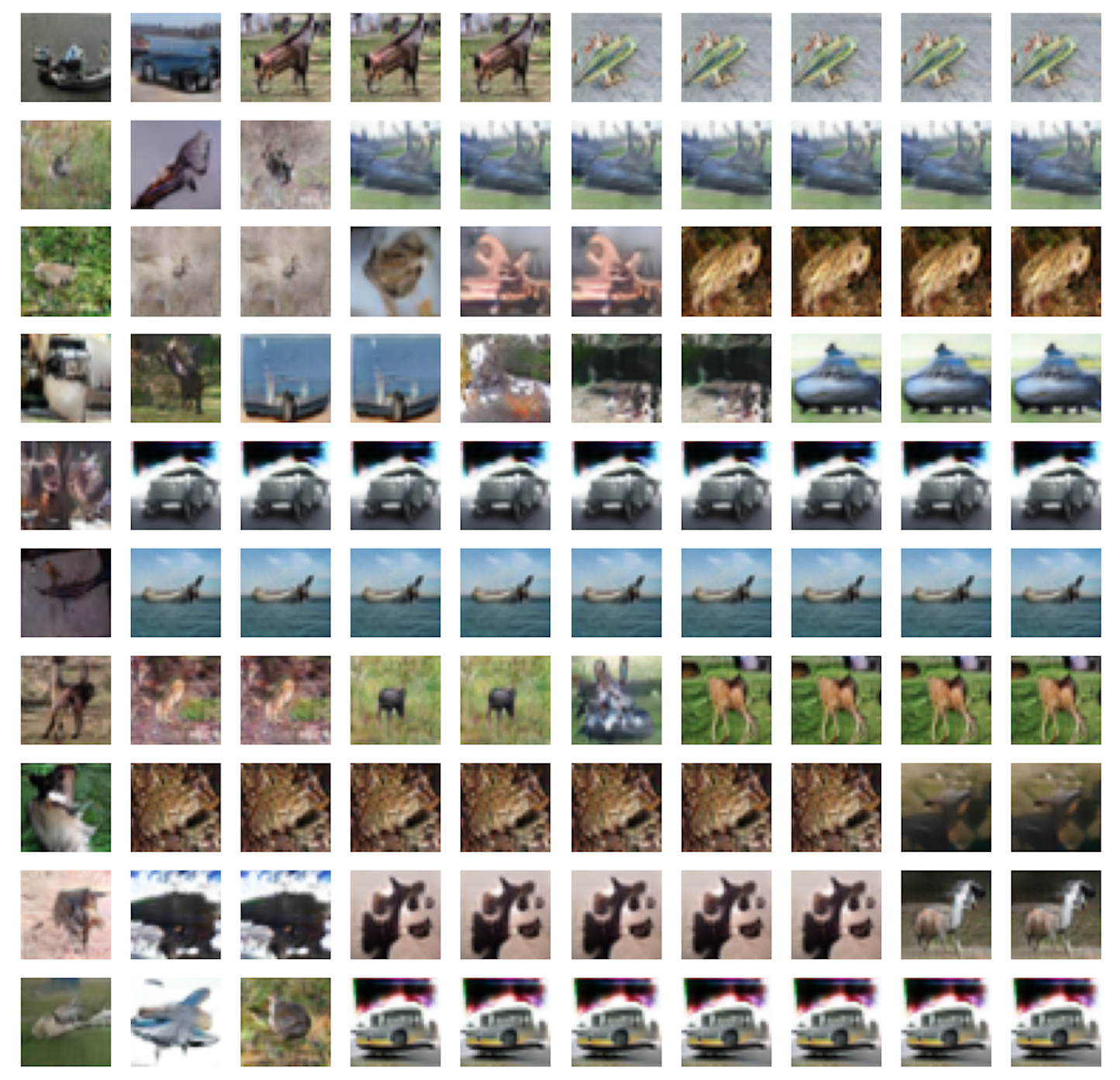}
    \caption{\isir\ samples}
    \label{fig:samples_dcgan_cifar_isir}
\end{subfigure}
\centering
\begin{subfigure}{0.8\textwidth}
    \includegraphics[width=1\textwidth]{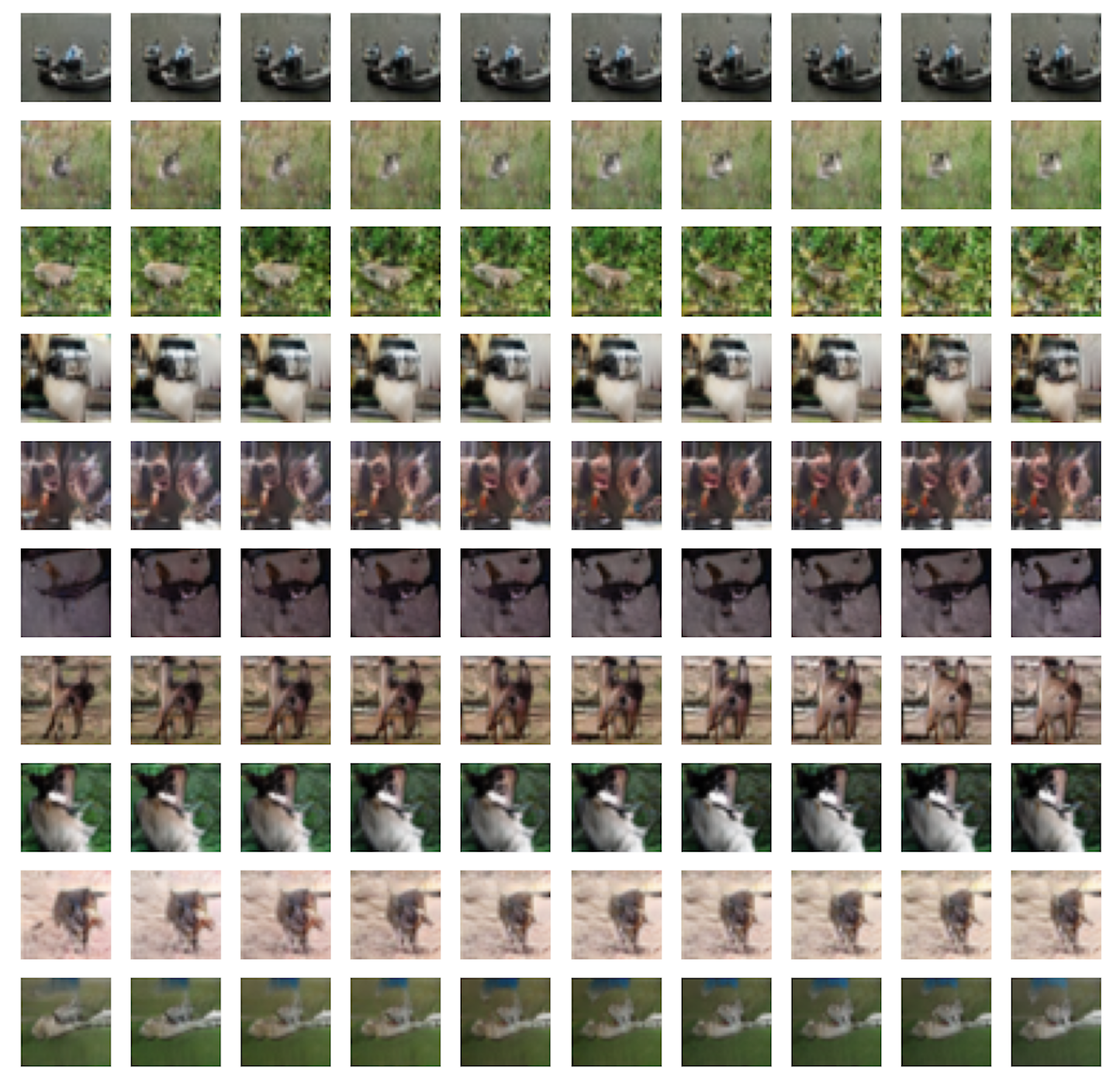}
    \caption{MALA samples}
    \label{fig:samples_dcgan_cifar_mala}
\end{subfigure}
\caption{\isir\ and MALA samples, DC-GAN.}
\label{fig:mala_isir_dcgan}
\end{figure}

\begin{figure}[h]
\vspace{-2\baselineskip}
\centering
\begin{subfigure}{0.9\textwidth}
    \includegraphics[width=1\textwidth]{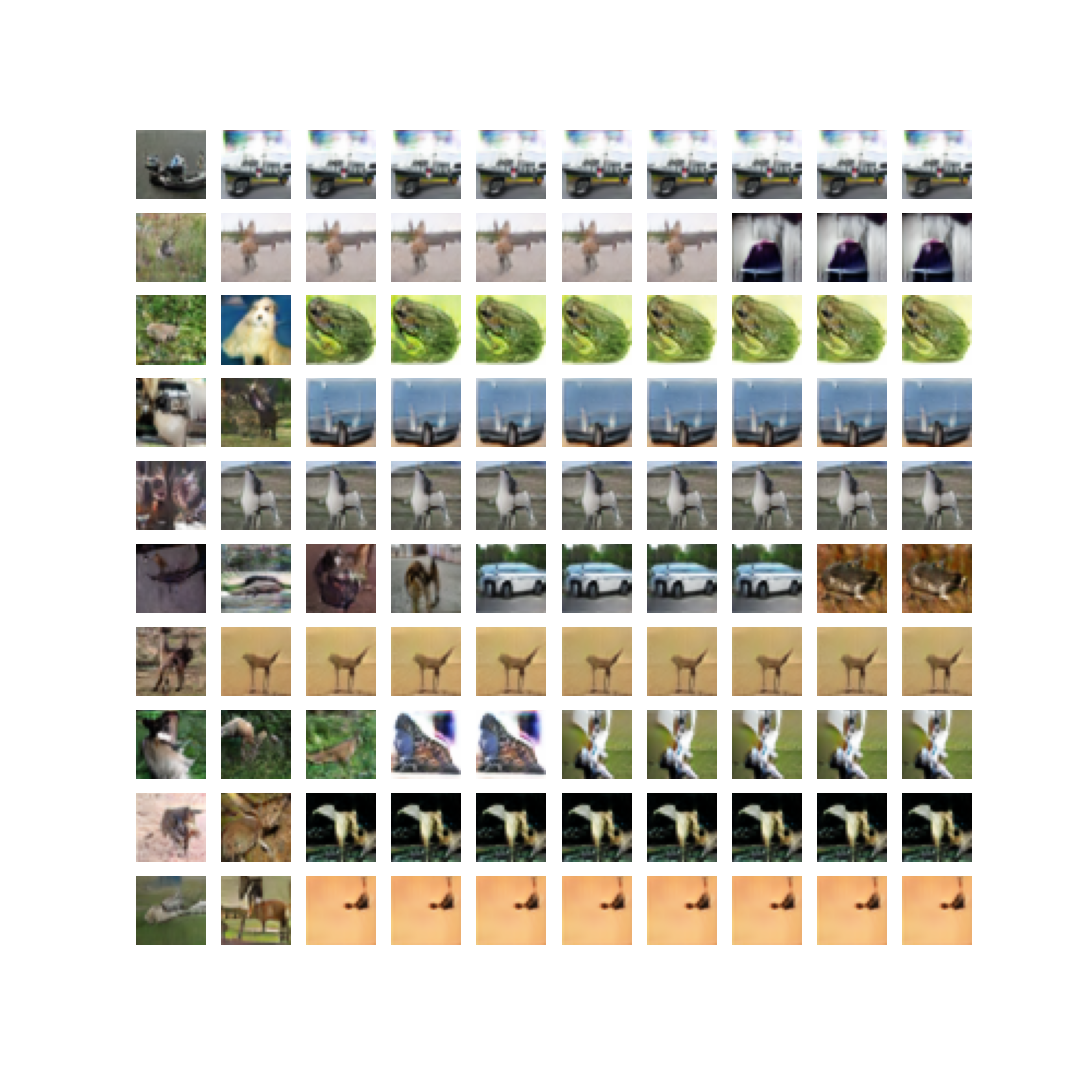}
    \vspace{-5\baselineskip}
    \caption{\XTryM\ samples}
    \label{fig:samples_dcgan_cifar_ex2}
\end{subfigure}
\centering
\begin{subfigure}{0.9\textwidth}
    \includegraphics[width=1\textwidth]{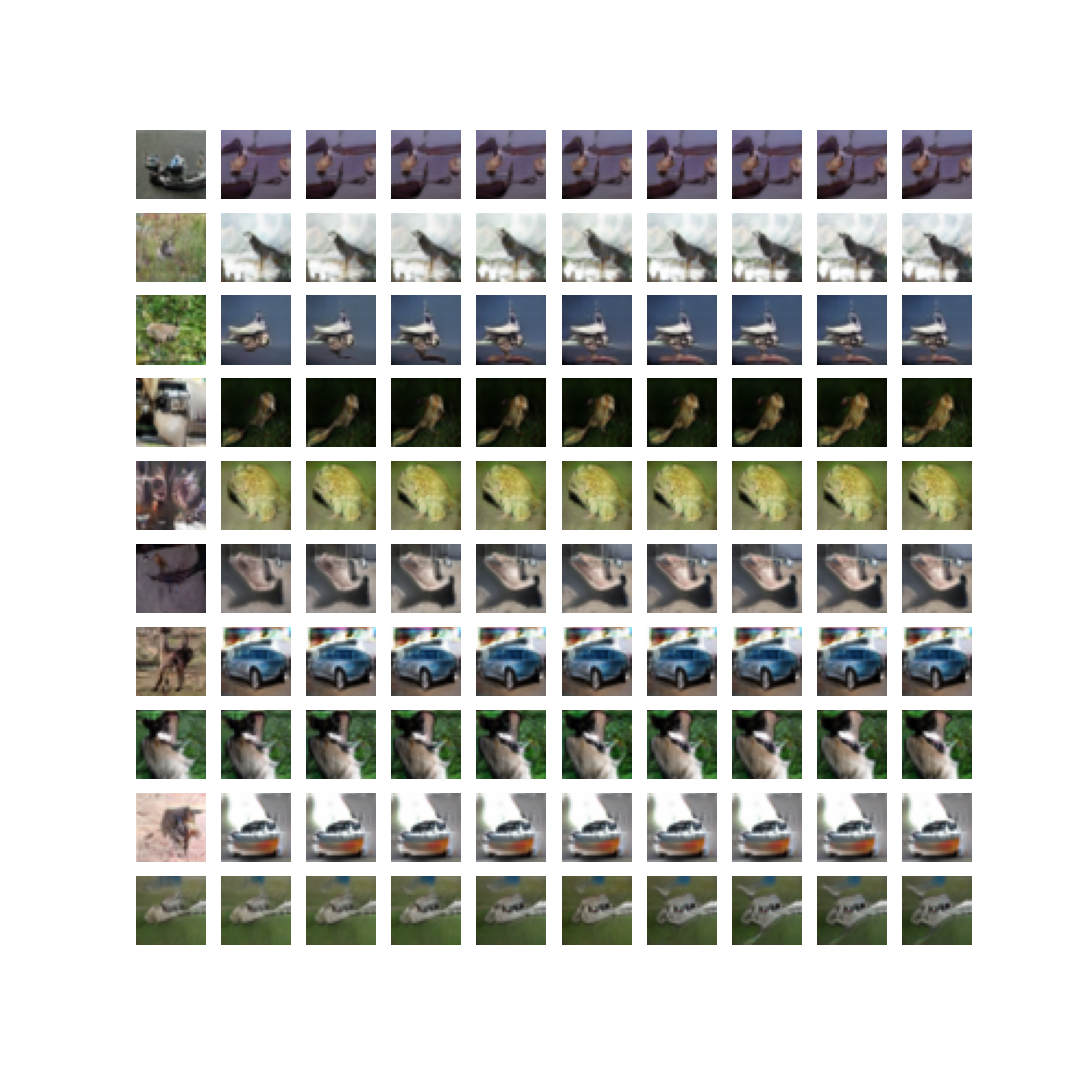}
    \vspace{-5\baselineskip}
    \caption{\FlXTryM\ samples}
    \label{fig:samples_dcgan_cifar_flex2}
\end{subfigure}
\caption{\XTryM\ and \FlXTryM\ samples, DC-GAN.}
\label{fig:ex2_flex2_dcgan}
\end{figure}

%% file: isir-proof.tex
\subsection{\isir\ and Multiple-try Metropolis (MTM) algorithm}
\label{subsec:isir-MTM}
In the MTM algorithm, $N$ \iid sample proposals $\{X^{i}_{k+1}\}_{i = 1}^N$ are drawn from a kernel $\kernelMTM(y, \cdot)$ in each iteration. In a second step, a sample $Y^*_{k+1}$ is selected with probability proportional to the weights (the exact expression of the weighting weights differs from ours, but this does not change the complexity of the algorithm). In a third step, $N-1$ \iid\ proposals are drawn from the kernel $\kernelMTM(Y^*_{k+1}, \cdot)$ and it is assumed that the move is $Y_{k+1}= Y^*_{k+1}$ with an \emph{generalised M-H} ratio, see~\citep[eq.~3]{liu2000multiple}. This step is bypassed in \isir, reducing the computational complexity by a factor of $2$.

\subsection{\isir\ as a systematic scan two-stage Gibbs sampler}
\label{subsec:two-stages-gibbs}
We analyze a slightly modified version of the {\isir} algorithm, with an extra randomization of the state position.
The $k$-th iteration is defined as follows. Given a state $Y_k \in \Xset$,
\begin{enumerate}[label=(\roman*),nosep,leftmargin=*]
\item draw $I_{k+1} \in \{1,\dots,N\}$ uniformly at random and set $X^{I_{k+1}}_{k+1}=Y_k$;
\item draw $\chunkum{X_{k+1}}{1}{N}{I_{k+1}}$ independently from the proposal distribution $\proposal$;
\item compute, for $i \in \{1, \dots, N\}$, the normalized importance weights
\[ \omega^i_{N,k+1} = \weightfunc(X^i_{k+1})/\sum_{\ell=1}^N \weightfunc(X^\ell_{k
+1});
\]
\item select $Y_{k+1}$ from the set $\chunku{X_{k+1}}{1}{N}$ by choosing $X_{k+1}^i$ with probability $\omega^i_{N,k+1}$.
\end{enumerate}
Thus, compared to the simplified \isir\ algorithm given in the introduction, the state is inserted uniformly at random into the list of candidates instead of being inserted at the first position. Of course, this change has no impact as long as we are interested in integrating functions that are permutation invariant with respect to candidates, which is the case throughout our work. Still, this randomization makes the analysis much more transparent.

In what follows, we show that {\isir} can be interpreted as a systematic-scan two-stage Gibbs sampler sampling, which alternately samples from the full conditionals of the extended target $\eTarget$, which is carefully defined below in terms of the state and candidate pool. Here we essentially follow the work of \cite{tjelmeland2004using,andrieu2010particle,andrieu2018uniform}. This is formalized by a dual representation of $\eTarget$, presented below in \Cref{thm:Gibbs:duality}, which provides the two complete conditionals in question.
We introduce the Markov kernel
\begin{equation}
\label{eq:definition-psiN}
\partopopN(y, \rmd \chunku{x}{1}{N}) = \frac{1}{N} \sum_{i=1}^N \delta_y(\rmd x^i) \prod_{j \ne i} \proposal(\rmd x^j)
\end{equation}
on $\Xset \times \Xsigma^{\varotimes N}$, which probabilistically describes the candidate selection operation in {\isir}.
Note that by construction, for each $y \in \Xset$, $\ell \in \{1, \dots, N\}$ and nonnegative measurable function $h: \Xset \to \rset^+$,
$$
\partopopN h(y) = \int \partopopN(y, \rmd \chunku{x}{1}{N}) h(x^\ell) = \left(1 - \frac{1}{N} \right) \proposal(h) + \frac{1}{N} h(y).
$$
Using the kernel $\partopopN$ we may now define properly the extended target $\eTarget$ as the probability law
\begin{equation}
\label{eq:extended-target}
\eTarget(\rmd (y, \chunku{x}{1}{N})) = \target(\rmd y) \partopopN(y, \rmd \chunku{x}{1}{N}) = \frac{1}{N} \sum_{i=1}^N \target(\rmd y)
\delta_y(\rmd x^i) \prod_{j \ne i} \proposal(\rmd x^j)
\end{equation}
on $(\Xset^{N + 1}, \Xsigma^{\varotimes (N + 1)})$. Note that since for every $A \in \Xsigma$, $\eTarget(\indi{A \times \Xset}) = \target(A)$, the target $\target$ coincides with the marginal of $\eTarget$ with respect to the state. Moreover, it is easily seen that $\partopopN$ provides the conditional distribution, under $\eTarget$, of the candidate pool given the state.

On the other hand, using that $\target(\rmd y) \delta_y(\rmd x^i) = \weightfunc(x^i) \proposal(\rmd x^i) \delta_{x^i}(\rmd y) / \proposal(\weightfunc)$, the marginal distribution $\eTargetmarginx$ of $\eTarget$ with respect to $\chunku{x}{1}{N}$ is given by
\begin{equation} \label{eq:marginal-joint}
\eTargetmarginx(\rmd \chunku{x}{1}{N}) = \frac{1}{\proposal(\weightfunc)} \utargetkern \1_\Xset(\chunku{x}{1}{N}) \prod_{j=1}^N \proposal(\rmd x^j) \eqsp,
\end{equation}
where we have set
\begin{equation}
\label{eq:def-utargetkern}
\utargetkern (\chunku{x}{1}{N}, \rmd y) = \sum_{i=1}^N \weightfunc(x^i) \delta_{x^i}(\rmd y) / N, \quad 
\targetkern(\chunku{x}{1}{N}, \rmd y) = {\utargetkern(\chunku{x}{1}{N}, \rmd y)}/{\utargetkern \1_{\Xset} (\chunku{x}{1}{N})}
\end{equation}
It is interesting to note that the marginal $\eTargetmarginx$ has a probability density function, proportional to $\utargetkern \1_\Xset(\chunku{x}{1}{N}) = \sum_{i=1}^N \weightfunc(x^i) / N$, with respect to the product measure $\proposal^{\varotimes N}$. Using \eqref{eq:marginal-joint}, we immediately obtain the following result.

\begin{theorem}[duality of extended target] \label{thm:Gibbs:duality}
For every $N \in \nsets$,
\[
\eTarget(\rmd (y, \chunku{x}{1}{N})) = \target(\rmd y) \partopopN(y,\rmd \chunku{x}{1}{N}) = \eTargetmarginx(\rmd \chunku{x}{1}{N}) \targetkern(\chunku{x}{1}{N}, \rmd y).
\]
\end{theorem}
Using this dual representation of $\eTarget$, \isir\ can be interpreted as a two-stage Gibbs sampler. Given the state $Y_k$, $N$ candidates $\chunku{X_{k+1}}{1}{N}$ are sampled from $\partopopN(Y_k, \cdot)$. In a second step, the next state $Y_{k+1}$ is sampled given the current candidates from $\targetkern(\chunku{X_{k+1}}{1}{N}, \cdot)$. The two-stages Gibbs sampler generates a Markov chain $\dsequence{Y}{\chunku{X}{1}{N}}[k][\nset]$ with Markov kernel
\begin{equation}
\label{eq:joint-kernel}
\MKisirjoint((y, \chunku{x}{1}{N}), C) = \int \partopopN(y,\rmd \chunku{\tilde{x}}{1}{N}) \targetkern(\chunku{\tilde{x}}{1}{N},\rmd \tilde{y}) \indi{C}(\rmd (y,\chunku{\tilde{x}}{1}{N}))  \eqsp, \quad C \in \Xsigma^{\otimes (N+1)}\eqsp.
\end{equation}
Note that the Markov kernel $\MKisirjoint(y, \chunku{x}{1}{N}, \cdot)$ does not depend on $\chunku{x}{1}{N}$, which means that only the state $Y_k$ needs to be stored from one iteration to another. Given a distribution $\xijoint$ on $(\Xset^{n+1},\Xsigma^{\otimes (n+1)})$, we denote by $\PP _\xijoint$ the distribution of the canonical Markov chain $\dsequence{Y}{\chunku{X}{1}{N}}[k][\nset]$ with kernel $\MKisirjoint$. With these notations, for any nonnegative measurable function $f: \Xset^{n+1} \to \rset$, we get, for $k \in \nsets$,
\begin{align}
\label{eq:definition-Mkisir-joint}
\CPE[\xijoint]{f(Y_{k},\chunku{X_{k}}{1}{N})}{\mcf_{k-1}}=
\int \MKisirjoint((Y_{k-1}, \chunku{X_{k-1}}{1}{N}), \rmd (y,\chunku{x}{1}{N})) f(\chunku{x}{1}{N})= \MKisirjoint f(Y_{k-1}, \chunku{X_{k-1}}{1}{N}) \eqsp.
\end{align}
The systematic scan two-stages Gibbs sampler is one of the MCMC algorithm structures that has given rise to many works. We summarize in the theorem below the important properties of this sampler; see \cite{liu1994covariance}, \cite[Chapter~9]{robert:casella:2013},  \cite{andrieu2016random} and the references therein.
\begin{theorem}
\label{theo:main-properties-deterministic-scan}
Assume that for any $y \in \Xset$, $\weightfunc(y) >0$. Then,
\begin{itemize}[leftmargin=*,nosep]
\item The Markov kernel $\MKisirjoint$ is Harris recurrent and ergodic with unique invariant distribution $\eTarget$.
\item The Markov kernel $\MKisir$ is reversible \wrt\ $\target$, Harris recurrent and ergodic.
\end{itemize}
\end{theorem}
The proof follows from \cite[Theorem~9.6, Lemma~9.11]{robert:casella:2013}.
The following theorem establishes the unbiasedness of the estimator $\targetkern f(\chunku{X}{1}{N})$ under $\eTarget$.
\begin{theorem} \label{thm:unbiasedness}
For every $N \in \nsets$ and $\target$-integrable function $f$,
$$
\pi(f):= \int \targetkern f(\chunku{x}{1}{N}) \eTargetmarginx(\rmd \chunku{x}{1}{N}) = \int \targetkern f(\chunku{x}{1}{N}) \pi(\rmd x^1) \prod_{j=2}^N \proposal(\rmd x^j) 
\eqsp.
$$
\end{theorem}
\begin{proof}
Using \eqref{eq:marginal-joint} we get
\begin{align}
\int \eTargetmarginx(\rmd \chunku{x}{1}{N}) \targetkern f(\chunku{x}{1}{N})
&=  \int \frac{1}{N \proposal(\weightfunc)} \sum_{\ell=1}^N \weightfunc(x^\ell) \targetkern f(\chunku{x}{1}{N}) \prod_{j=1}^N \proposal(\rmd x^j) \\
&= \frac{1}{N \proposal(\weightfunc)} \int \sum_{i=1}^N \weightfunc(x^i) f(x^i) \prod_{j=1}^N \proposal(\rmd x^j)= \target(f),
\end{align}
and the first identity follows. The second identity stems from the fact that the function $\targetkern f(\chunku{x}{1}{N})$ is invariant under permutation.
\end{proof}

%% file: proof_adapt.tex
\def\Vset{\mathbb{U}}
\def\Vsigma{\mathcal{U}}
\def\Eset{\mathbb{E}}
\def\Esigma{\mathcal{E}}
\def\MCState{U}
\def\MCstate{u}
\def\MCNoise{E}
\def\MCnoise{e}
\def\borel{\mathcal{B}}
\def\Proj{\operatorname{Proj}}
\def\proj{{\operatorname{proj}}}
\def\targetadapt{\mu}
\def\driftfunc{V}
\def\propnoise{\rho}
The proof relies on results of stochastic approximation with Markovian dynamics; see  \cite{andrieu2005stability,andrieu2015stability}. For reader's convenience, before going into the details, we give an outline of the proof. The motivation of such algorithms is to find the roots of the function $h: \Theta \to \rset^q$, $\Theta \subset \rset^q$
\[
h(\theta) =  \int_{\Vset \times \Eset} H(\theta, \MCstate, \MCnoise) \targetadapt(\rmd \MCnoise) \propnoise_\theta( \rmd \MCnoise) \eqsp,
\]
for  families of functions $\{H(\theta, \MCstate, \MCnoise) : \Theta \times \Vset \times \Eset \rightarrow \Theta\}$, a family of probability distributions $\{\propnoise_\theta, \theta \in \Theta$ of $(\Eset,\Esigma)$ and a probability distribution $\targetadapt$ on a space $(\Vset,\Vsigma)$. These roots are not available analytically and a way of finding them numerically consists of considering the controlled Markov chain on $\left\{(\Theta \times \Vset)^{\nset},(\borel(\Theta) \otimes \Vsigma)^{\otimes \nset} \right\}$ initialized at some $(\theta_0,\MCState_0)= (\vartheta,\MCstate) \in \Theta \times \Vset$ and defined recursively for a sequence of stepsize $\sequence{\gamma}[i][\nset]$ by
\begin{equation}
\begin{split}
&\MCState_{i+1}  \sim P_{\theta_{i}}(\MCState_{i}, \cdot) \eqsp, \quad \MCNoise_{i+1} \sim \propnoise_{\theta_i} \\
&\theta_{i+1}=\theta_{i}+\gamma_{i+1} H(\theta_{i}, \MCState_{i+1},\MCNoise_{i+1}) \eqsp,
\end{split}
\end{equation}
where $\left\{P_{\theta}, \theta \in \Theta\right\}$ is a family of Markov kernels such that for each $\theta \in \Theta$, $\targetadapt P_\theta= \targetadapt$. The rationale for this recursion goes as follows. Let us first rewrite the Robbins-Monro recursion
\[
\theta_{i+1}= \theta_{i} + \gamma_{i+1} \{ h(\theta_{i}) + \xi_{i+1} \},
\]
where  $\xi_{i+1}= H(\theta_i,\MCState_{i+1},\MCNoise_{i+1})$  is referred to as the "noise". Therefore, $\{\theta_i\}$ is a noisy version of the sequence $\{\bar{\theta}_i\}$ defined as $\bar{\theta}_{i+1}= \bar{\theta}_i + \gamma_{i+1} h(\bar{\theta}_i)$. The convergence of such sequences has been studied by many authors, starting with \cite{metivier1987theoremes} under various conditions. A crucial step of such convergence analysis consists of assuming that the sequence $\{\theta_i\}$ remains bounded with probability 1 in a compact set of $\Theta$. This problem has traditionally can be circumvented by
means of modifications of the recursion. Indeed, one of the major difficulties
specific to the Markovian dynamic scenario is that $\{\theta_i\}$ governs the ergodicity of the controlled Markov chain $\{\MCState_i\}$  and that stability properties of $\{\theta_i\}$  require "good" ergodicity properties which might vanish whenever $\{\theta_i\}$ approaches $\partial \Theta$ often away from the roots of $h(\theta)$, resulting in instability. Most existing results rely on modifications of the updates designed to ensure a form of ergodicity
of $\{\xi_i\}$ which in turn ensures that $\{\theta_i\}$ inherits the stability properties of $\{\bar{\theta}_i\}$; see \eg\ \cite{andrieu2005stability,andrieu2014markovian} and the discussion in \cite[Section~3]{andrieu2015stability}. We follow here \cite{andrieu2014markovian}. Let $\{\mathcal{R}_i\}$ be a sequence of compact subsets of $\Theta$ and consider the recursion:
\begin{equation}
\begin{split}
&\MCState_{i+1} \sim P_{\theta_i}(\MCState_i,\cdot) \quad \MCNoise_{i+1} \sim \propnoise_{\theta_i} \\
&\theta_{i+1}^* = \theta_i + \gamma_{i+1} H(\theta_i,\MCState_{i+1},\MCNoise_{i+1})  \\
&\theta_{i+1}= \theta_{i+1}^* \indi{\mathcal{R}_{i+1}}(\theta^*_{i+1}) + \theta_{i+1}^{\operatorname{proj}} \indi{\mathcal{R}_{i+1}^c}(\theta_{i+1}^*)
\end{split}
\end{equation}
where, denoting $\mathcal{F}_i= \sigma( \MCState_0,\theta_j, j \leq i)$, $\theta_{i+1}^\proj$ is a random variable measurable w.r.t $\mathcal{F}_i \vee \sigma(\theta_{i+1}^*)$. Most common practical projection mechanisms include $\theta_{i+1}^{\operatorname{proj}}= \theta_i$, 'rejecting' an update outside the current feasible set, and $\theta_{i+1}^{\operatorname{proj}}= \Proj_{\mathcal{R}_{i+1}}(\theta^*_{i+1})$, where $\Proj$ is a measurable mapping $\Theta  \setminus \mathcal{R}_{i+1} \to \mathcal{R}_{i+1}$. In words, the expanding projections approach only ensures that $\theta_i$ is in a feasible set $\mathcal{R}_i$ but does not involve potentially harmful ‘restarts’ as is the case with the adaptive reprojection strategy of \cite{andrieu2005stability}.
We use the results in \cite{andrieu2014markovian} to show that the SA $\{\theta_i\}$ 'stays away' from $\partial \Theta$ with probability one for any initialization $(\theta_0,\MCstate) \in \mathcal{R}_0 \times \Vset$ under appropriate conditions on $\{H(\theta,\MCstate,\MCnoise), (\theta,\MCstate,\MCnoise) \in \Theta \times \Vset \times \Eset\}$, $\{ P_\theta, \theta \in \Theta\}$ and $\{\mathcal{R}_i\}$.
We denote throughout  the probability distribution associated to the process $\left(\theta_{i}, \MCState_{i}\right)_{i \geq 0}$ defined in Algorithm $1.1$ and starting at $\left(\theta_{0}, \MCState_{0}\right) \equiv(\theta, \MCstate) \in \Theta \times \Vset$ as $\PP_{\theta, \MCstate}(\cdot)$ and the associated expectation as $\PE_{\theta, \MCstate}[\cdot]$.
The approach developed in \cite{andrieu2014markovian} relies on the existence of a Lyapunov function $w: \Theta \rightarrow[0, \infty)$ for the recursion on $\theta$ and the subsequent proof that $\left\{w\left(\theta_{i}\right)\right\}$ is $\PP_{\theta, \MCstate}$-a.s. under some adequate level. For any $M>0$, we define the level sets $\mathcal{W}_{M}:=\{\theta \in \Theta: w(\theta) \leq M\}$.
Consider the following assumptions:
\begin{assumptionSA}
\label{assum:SA:stability-1}
There exists a continuously differentiable function $w: \Theta \rightarrow \coint{0, \infty}$ such that
\begin{enumerate}[label=(\roman*),leftmargin=*,nosep]
\item For all $\theta,\theta' \in \Theta$,
\[
\Vert \nabla w(\theta) - \nabla w(\theta') \Vert \leq C_{w} \Vert \theta - \theta' \Vert \eqsp.
\]
\item the projection sets are increasing subsets of $\Theta$, that is, $\mathcal{R}_{i} \subset \mathcal{R}_{i+1}$ for all $i \geq 0$, and
$$
\hat{\Theta}:=\bigcup_{i=0}^{\infty} \mathcal{R}_{i} \subset \Theta \eqsp,
$$
\item there exists a constant $M_{0}>0$ such that for any $\theta \in \mathcal{W}_{M_{0}}^{c} \cap \hat{\Theta}$
$$
\ps{\nabla w(\theta)}{h(\theta)} \leq 0
$$
\item the family of random variables $\left\{\theta_{i}^{\mathrm{proj}}\right\}_{i \geq 1}$ satisfies for all $i \geq 1$ whenever $\theta_{i}^{*} \notin \mathcal{R}_{i}$
$$
\theta_{i}^{\mathrm{proj}} \in \mathcal{R}_{i} \quad \text { and } \quad w\left(\theta_{i}^{\mathrm{proj}}\right) \leq w\left(\theta_{i}^{*}\right) \quad \PP_{\theta, \MCstate}-\as.
$$
\item there exists constants $ c \in \coint{0, \infty}$ and a non-decreasing sequence of constants $\zeta_{i} \in \coint{1, \infty}$ satisfying $\sup _{\theta \in \mathcal{R}_{i}}|\nabla w(\theta)| \leq c \zeta_{i}$ for all $i \geq 0$.
\end{enumerate}
\end{assumptionSA}
Hereafter, we denote the 'centred' version of  $\bar{H}(\theta,\MCstate,\MCnoise):=H(\theta, \MCstate,\MCnoise)-h(\theta)$. For the stability results, we shall introduce the following general condition on the noise sequence. In general terms, it is related to the rate at which $\left\{\theta_{i}\right\}$ may approach $\partial \hat{\Theta}$ in relation to the growth of $\Vert H(\theta, \MCstate, \MCnoise) \Vert$ and the loss of ergodicity of $P_{\theta}$.
\begin{assumptionSA}
\label{assum:SA:stability-2}
For any $(\theta, \MCstate) \in \mathcal{R}_{0} \times \Vset$ it holds that
\begin{enumerate}[label=(\roman*),leftmargin=*,nosep]
\item $\PP_{\theta, \MCstate}\left(\lim _{i \rightarrow \infty} \gamma_{i+1} \Vert\nabla w(\theta_{i})\Vert  \cdot \Vert H(\theta_{i}, \MCState_{i+1},\MCNoise_{i+1}) \Vert =0 \right)=1$,
\item $\PE_{\theta, \MCstate}\left[\sum_{i=0}^{\infty} \gamma_{i+1}^{2}\Vert H(\theta_{i}, \MCState_{i+1},\MCNoise_{i+1}) \Vert^2 \right]<\infty$,
\item $\PE_{\theta, \MCstate}\left[\sup _{k \geq 0}\left|\sum_{i=0}^{k} \gamma_{i+1}\ps{\nabla w(\theta_{i})}{ \bar{H}(\theta_{i}, \MCState_{i+1},\MCNoise_{i+1})} \right|\right]<\infty$.
\item $\lim _{\theta \rightarrow \partial \hat{\Theta}} w(\theta)=\infty$
\end{enumerate}
\end{assumptionSA}
\begin{theorem}
Assume \Cref{assum:SA:stability-1}-\Cref{assum:SA:stability-2}. Then, for any $(\theta, \MCstate) \in \mathcal{R}_{0} \times \mathrm{\MCState}$
$$
\PP_{\theta, \MCstate}\left(\limsup _{i \rightarrow \infty} w\left(\theta_{i}\right)<\infty\right)=1 .
$$
\end{theorem}
\begin{proof}
The proof is a simple adaptation of \cite[Theorem~2.5]{andrieu2014markovian}.
\end{proof}

The condition $\lim _{\theta \rightarrow \partial \hat{\Theta}} w(\theta)=\infty$ is weakened in \cite[Section~2.2]{andrieu2014markovian}. Verifiable conditions implying \Cref{assum:SA:stability-2} are given in
\cite[Section~3, Condition~3.1]{andrieu2014markovian}. They are summarized in the next assumption.
In the assumptions below, it is implicitly assumed that \Cref{assum:SA:stability-1} holds with constants $\left(\zeta_{i}\right)_{i \geq 0}$.

We denote $\tilde{H}(\theta,\MCstate)= \int \bar{H}(\theta,\MCstate,\MCnoise) \propnoise(\rmd \MCnoise)$ and we consider the following assumptions:
\begin{assumptionSA}
\label{assum:SA:stability-3}
 For all $\theta \in \hat{\Theta}$, the solution $g_{\theta}: \Vset \rightarrow \Theta$ to the Poisson equation $g_{\theta}(\MCstate)-P_{\theta} g_{\theta}(\MCstate) \equiv \tilde{H}(\theta, \MCstate)$ exists and for all $i \geq 0$ the step size $\Gamma_{i+1}$ is independent of $\mathcal{F}_{i}$ and $\MCState_{i+1}$. Moreover, there exist a measurable function $\driftfunc: \Vset \rightarrow[1, \infty)$ and constants $c<\infty, \beta_{H}, \beta_{g} \in[0,1 / 2]$ and $\alpha_{g}, \alpha_{H}, \alpha_{\driftfunc} \in[0, \infty)$ such that for all $(\theta, \MCstate) \in \mathcal{R}_{0} \times \Vset$
\begin{enumerate}[label=(\roman*),leftmargin=*,nosep]
\item $\sup _{\theta \in \mathcal{R}_{i}}|\tilde{H}(\theta, \MCstate)| \leq c \zeta_{i}^{\alpha_H} \driftfunc^{\beta_H}(\MCstate)$,
\item $\PE_{\theta, \MCstate}\left[\driftfunc\left(\MCState_{i}\right)\right] \leq c \zeta_{i}^{\alpha_\driftfunc} \driftfunc(\MCstate)$,
\item $\sup _{\theta \in \mathcal{R}_{i}}\left[\left|g_{\theta}(\MCstate)\right|+\left|P_{\theta} g_{\theta}(\MCstate)\right|\right] \leq c \zeta_{i}^{\alpha_{g}} \driftfunc^{\beta_{g}}(\MCstate)$,
\item $\sum_{i=1}^{\infty} \gamma_{i+1}\ \zeta_{i} \PE_{\theta, \MCstate}\left[\left|P_{\theta_{i}} g_{\theta_{i}}\left(\MCState_{i}\right)-P_{\theta_{i-1}} g_{\theta_{i-1}}\left(\MCState_{i}\right)\right|\right]<\infty$,
\item \label{item:stability-det-1} $\sum_{i=1}^{\infty} \gamma^2_{i} \zeta_{i}^{2 +2\left(\left(\alpha_{H}+\beta_{H} \alpha_{\driftfunc}\right) \vee\left(\alpha_{g}+\beta_{g} \alpha_{\driftfunc}\right)\right)}<\infty$,
\item \label{item:stability-det-2} $\sum_{i=1}^{\infty} \gamma_{i+1} \gamma_{i} \zeta_{i}^{\alpha_{H}+\alpha_{g}+\left(\beta_{H}+\beta_{g}\right) \alpha_\driftfunc}<\infty$,
\item \label{item:stability-det-3} $\sum^{\infty}\left|\gamma_{i+1}-\gamma_{i} \right| \zeta_{i}^{1+\alpha_{g}+\beta_{g} \alpha_{\driftfunc}}<\infty$.
\end{enumerate}
\end{assumptionSA}
For geometrically ergodic Markov chain, these conditions may be shown to boil down to "uniform-in-$\theta$" geometric ergodicity conditions and "smoothness" of the mapping $\theta \mapsto P_\theta$.
\begin{assumptionMC}
\label{assum:MC-1}
For any $r \in\ocint{0,1}$ and any $\theta \in \hat{\Theta}$, there exist constants $M_{\theta, r} \in \coint{0, \infty}$ and $\rho_{\theta, r} \in$ $(0,1)$, such that for any function $\|f\|_{\driftfunc^{r}}<\infty$
$$
\left|P_{\theta}^{k}f(\MCstate)-\targetadapt_\theta(f)\right| \leq \driftfunc^{r}(\MCstate)\|f\|_{\driftfunc^{r}} M_{\theta, r} \rho_{\theta, r}^{k}
$$
for all $k \geq 0$ and all $\MCstate \in \Vset$. Moreover, it holds that
$\sup _{\theta \in \mathcal{R}_{i}} M_{\theta, r} \leq c_{r} \zeta_{i}^{\alpha_{M}} \quad$ and $\quad \sup _{\theta \in \mathcal{R}_{i}}\left(1-\rho_{\theta, r}\right)^{-1} \leq c_{r} \zeta_{i}^{\alpha_{\rho}}$.
\end{assumptionMC}

\begin{assumptionMC}
\label{assum:MC-2}
For any $\theta, \theta^{\prime} \in \hat{\Theta}$, there exist a constant $D_{\theta, \theta^{\prime}, r} \in[0, \infty)$ and a constant $\beta_{D} \in(0, \infty)$ independent of $\theta, \theta^{\prime}$ and $r$ such that for any function $\|f\|_{\driftfunc^{r}}<\infty$
$$
\left\|P_{\theta} f-P_{\theta^{\prime}} f\right\|_{\driftfunc^{r}} \leq\|f\|_{\driftfunc^{r}} D_{\theta, \theta^{\prime}, r}\left|\theta-\theta^{\prime}\right|^{\beta_{D}} .
$$
Moreover, $\sup _{\left(\theta, \theta^{\prime}\right) \in \mathcal{R}_{i}^{2}} D_{\theta, \theta^{\prime}, r} \leq c_{r}^{D} \zeta_{i}^{\alpha_{D}}$ for some constant $c_{r}^{D} \in[0, \infty)$ depending only on $r \in(0,1]$
\end{assumptionMC}

\begin{assumptionMC}
\label{assum:MC-3}\Cref{assum:SA:stability-3}-(i) and (ii) hold with constants $\alpha_{H}, \beta_{H}$ and $\alpha_{\driftfunc}$, and there exist constants $c<\infty, \alpha_{\Delta} \in[0, \infty)$ and $\beta_{\Delta}>0$ such that
$$
\sup _{\left(\theta, \theta^{\prime}\right) \in \mathcal{R}_{i}^{2}}\left\|\tilde{H}(\theta, \cdot)-\tilde{H}\left(\theta^{\prime}, \cdot\right)\right\|_{\driftfunc^{\beta_{H}}} \leq c \zeta_{i}^{\alpha_{\Delta}}\left|\theta-\theta^{\prime}\right|^{\beta_{\Delta}} .
$$
\end{assumptionMC}
Up to this point, we have only considered the stability of the stochastic approximation process
with expanding projections. Indeed, after showing the stability we know that the projections can
occur only finitely often (almost surely), and the noise sequence can typically be controlled. Given this, the stochastic approximation literature provides several alternatives to show the convergence; see \cite{kushner2003stochastic,borkar2009stochastic}. We formulate below a convergence result following from \cite{andrieu2005stability}.
\begin{assumptionSA}
\label{assum:cvgce}
The set $\Theta \subset \mathbb{R}^{d}$ is open, the mean field $h: \Theta \rightarrow \mathbb{R}^{d}$ is continuous, and there exists a continuously differentiable function $\hat{w}: \Theta \rightarrow\coint{0, \infty}$ such that
\begin{enumerate}[label=(\roman*),leftmargin=*,nosep]
\item there exists a constant $M_{0}>0$ such that
$$
\mathcal{L}:=\{\theta \in \Theta:\ps{\nabla \hat{w}(\theta)}{h(\theta)}=0\} \subset\left\{\theta \in \Theta: \hat{w}(\theta)<M_{0}\right\}
$$
\item there exists $M_{1} \in\left(M_{0}, \infty\right]$  such that $\left\{\theta \in \Theta: \hat{w}(\theta) \leq M_{1}\right\}$  is compact.
\item for all $\theta \in \Theta \setminus \mathcal{L}$, the inner product $\ps{\nabla \hat{w}(\theta)}{ h(\theta)}<0$ and the closure of $\hat{w}(\mathcal{L})$ has an empty interior.
\end{enumerate}
\end{assumptionSA}

\begin{theorem}
Assume \Cref{assum:cvgce} holds, and let $\mathcal{K} \subset \Theta$ be a compact set intersecting $\mathcal{L}$,
that is, $\mathcal{K} \cap \mathcal{L} \neq \varnothing$. Suppose that $\left(\gamma_{i}\right)_{i \geq 1}$ is a sequence of non-negative real numbers satisfying $\lim _{i \rightarrow \infty} \gamma_{i}=0$ and $\sum_{i=1}^{\infty} \gamma_{i}=\infty$. Consider the sequence $\left(\theta_{i}\right)_{i \geq 0}$ taking values in $\Theta$ and defined through the recursion $\theta_{i}=\theta_{i}-1+\gamma_{i} h\left(\theta_{i-1}\right)+\gamma_{i} \varepsilon_{i}$ for all $i \geq 1$, where $\left(\varepsilon_{i}\right)_{i \geq 1}$ take values in $\mathbb{R}^{d}$.
If there exists an integer $i_{0}$ such that $\left\{\theta_{i}\right\}_{i \geq i_{0}} \subset \mathcal{K}$ and $\lim _{m \rightarrow \infty} \sup _{n \geq m}\left|\sum_{i=m}^{n} \gamma_{i} \varepsilon_{i}\right|=0$, then $\lim _{n \rightarrow \infty} \inf _{x \in \mathcal{L} \cap \mathcal{K}}\Vert\theta_{n}-x\Vert=0 .$
\end{theorem}
We have now all the necessary elements to prove \Cref{thm:KL-simplified}. For simplicity, we set
$\alpha_k = \alpha_\infty$ for any $k \in \nset$ and $\gamma_k= 1 / (1 + k)^{\iota}$ where $\iota \in \ocint{1/2,1}$.
In this case, the state space is $\Vset= \Xset^M$ and $\Eset= \Zset^{(N-1) \cdot M}$, $\MCState_k= (Y_k[j])_{j=1}^M$, $\MCNoise_k= (\chunku{Z_k}{2}{N}[j])_{j=2}^N$. 
With $\MCstate= (y[j])_{j=1}^M$  and $\MCnoise= (\chunku{z}{2}{N}[j])_{j=1}^M$, 
$H(\theta,\MCstate,\MCnoise)$ is given by
\[
H(\theta,\MCstate)= M^{-1} \sum_{mj=1}^N \{\alpha_\infty H^f(\theta,y[j],\chunku{z}{2}{N}[j]) + (1-\alpha_\infty) H^b(\theta,\chunku{z}{2}{N}[j]) \} \eqsp.
\]
where $H^f$ and $H^b$ are defined respectively in \eqref{eq:fwdKL} and \eqref{eq:bwdKL}.
In this case, the Markov kernel $P_\theta$ is given for any nonnegative function $f$,
\[
P_{\theta,N} f(y[1],\dots,y[M])=  \int \prod_{j=1}^N \XtryK[\theta,N](y[j],\rmd \tilde{y}[j])  f(\tilde{y}[1],\dots,\tilde{y}[M]) \eqsp,
\]
and $\XtryK[\theta,N]$ is defined in \eqref{sec:coupling-with-local} with $\proposal \leftarrow \proposal_\theta$ and $\weightfunc \leftarrow \weightfunc_\theta$. 
By construction, for any $\theta \in \Theta$, $P_\theta$ has a unique stationary distribution which is given by
$\targetadapt = \target^{\otimes M}$.
Using \Cref{thm:unbiasedness}, and, for all $\theta \in \Theta$,
\[
H^f(\theta,\chunku{x}{1}{N})= \targetkern[\theta,N] [\nabla_\theta \log \proposal_\theta](\chunku{x}{1}{N})
\]
we get that
\[
h(\theta)= - \alpha_\infty \nabla_\theta \KL{\target}{\proposal_\theta} - (1-\alpha_{\infty}) \nabla_\theta \KL{\proposal_\theta}{\target} \eqsp.
\]
Recall that $\Theta= \rset^q$.  To check \Cref{assum:SA:stability-1}, we set
\begin{equation}
w(\theta)= \alpha_\infty  \KL{\target}{\proposal_\theta} - (1-\alpha_{\infty})  \KL{\proposal_\theta}{\target} \eqsp,  \text{for $\theta \in \Theta$}.
\end{equation}
and for $i \in \nset$, $\zeta_i= \log(i+1)$. The subset $\mathcal{R}_i$ is a ball centered at $0$ and of radius $r_i$ where $r_i$ is chosen so that $\sup_{\| \theta \|\leq r_i} \nabla w(\theta) \| \leq c \zeta_i$ (such $r_i$ exists using \Cref{assum:condition-gradient}).  
It is easily checked that  \Cref{assum:SA:stability-1} is satisfied thanks to \Cref{assum:condition-gradient} (note in particular that $\nabla w$ is globally Lipshitz under the stated conditions). Conditions \Cref{assum:SA:stability-3}-\ref{item:stability-det-1}-\ref{item:stability-det-2}-\ref{item:stability-det-3} are automatically satisfied.

We consider the drift function for the Markov kernel $P_{\theta,N}$
\begin{equation}
\label{eq:drift-function}
\driftfunc(y[1],\dots,y[M])= \sum_{i=1}^M V(y[i]) 
\end{equation}
where $V$ is the drift function in \Cref{assum:rejuvenation-kernel}. \Cref{assum:MC-1} follows from \Cref{theo:main-geometric-ergodicity}  under \Cref{assum:independent-proposal-strengthen}. It is important to note that it is essential to have explicit controls on the drift and reduction conditions here. Conditions \Cref{assum:MC-2} and \Cref{assum:MC-2} follow from \Cref{assum:condition-gradient}. The precise tuning of constants is done along the same lines as \cite[Section~5.3]{andrieu2014markovian}. 